\documentclass{article}
\usepackage[preprint]{tmlr}   

\usepackage{amsmath,amssymb,amsthm}
\usepackage{graphicx}
\usepackage{booktabs}
\makeatletter\let\AND\@undefined\makeatother
\usepackage{algorithm}
\usepackage{algorithmic}
\usepackage{bbm}
\usepackage{enumitem}
\usepackage{fontspec}
\usepackage[english,provide=*]{babel}
\usepackage{hyperref}

\babelprovide[import,onchar=ids fonts]{english}
\babelprovide[onchar=ids fonts]{tifinagh}
\babelfont[tifinagh]{rm}[Path=fonts/, Extension=.ttf, UprightFont=*]{ebrima}
\newfontfamily\tifinaghfont[Path=fonts/, Extension=.ttf, UprightFont=*]{ebrima}

\newtheorem{theorem}{Theorem}[section]
\newtheorem{proposition}[theorem]{Proposition}
\newtheorem{corollary}[theorem]{Corollary}
\newtheorem{lemma}[theorem]{Lemma}
\newtheorem{definition}[theorem]{Definition}
\newtheorem{remark}[theorem]{Remark}
\newtheorem{example}[theorem]{Example}

\DeclareMathOperator{\tr}{tr}
\DeclareMathOperator{\diag}{diag}
\DeclareMathOperator{\Var}{Var}
\DeclareMathOperator{\Exp}{Exp}

\DeclareMathOperator{\vecc}{vec}

\DeclareRobustCommand{\yat}{\text{\normalfont\tifinaghfont ⵟ}}

\title{Bernstein--Schur Kernels: Random Features by Sketched Modulation and Radial Randomization}
\author{\name Taha Bouhsine \email taha@azetta.ai \\
        \addr Azetta AI}

\begin{document}
\maketitle

\begin{abstract}
\emph{Bernstein--Schur kernels} are products of a finite-feature kernel and a completely monotone shift-invariant kernel: nonstationary kernels falling between the shift-invariant and dot-product templates random features exploit, so neither Bochner sampling nor polynomial sketching applies to the full kernel directly. We give one random-feature construction for the whole class that \emph{randomizes both factors}: it sketches the finite modulation and samples the radial factor's one-dimensional Bernstein--Widder scale before applying Gaussian random Fourier features, giving feature dimension $Dm$, free of the $O(d^2)$ size of the exact modulation feature. With the modulation kept exact (the $m\to\infty$ limit), we prove unbiasedness, an exact variance, and a matrix-Bernstein operator-norm bound controlled by the top kernel and modulation eigenvalues and an intrinsic dimension rather than the crude $N\max_{ij}$ route. Whitening this argument at the ridge makes the effective dimension $d_{\mathrm{eff}}(\lambda)$ the \emph{exact} intrinsic dimension of the matrix variance, so $O((1+\|P\|_{\mathrm{op}}/\lambda)\log(d_{\mathrm{eff}}/\delta))$ radial draws preserve the kernel-ridge solution; tilting the draw by a closed-form whitened leverage improves this to the effective-dimension count $O((1+d_{\mathrm{eff}})\log(d_{\mathrm{eff}}/\delta))$. Conditioning on the sketch carries every guarantee to the deployed doubly-randomized estimator up to one additive sketch term, and all hold for the whole class with the modulation Gram in place of the polynomial one. The flagship instance is the biased $\yat$-kernel $k_{\yat,b}(w,x)=(w^\top x+b)^2/(\|w-x\|^2+\varepsilon)$, whose family span contains the inverse-multiquadric kernel by finite differences in $b$. Experiments validate the construction off-sphere, where the kernel is genuinely non-dot-product and Nystr\"om degrades with $d$ at matched landmark count, and show on a controlled test that the $\yat$-kernel is preferred on the coupled alignment--proximity target and not on single-factor controls.
\end{abstract}

\section{Introduction}
\label{sec:intro}

Kernel methods (support vector machines \citep{cortes1995support}, kernel ridge regression, Gaussian processes \citep{rasmussen2006gaussian}, and the attention layers of modern transformers, which compute a kernel smoother over tokens \citep{tsai2019transformer}) rest on the $N\times N$ Gram matrix $K_{ij}=k(x_i,x_j)$, whose $O(N^2)$ storage and $O(N^3)$ factorization become infeasible once $N$ exceeds $10^5$. The kernel may be fixed or learned implicitly \citep{li2019implicit}, but in every case scaling these methods means never forming $K$ explicitly, and the dominant tool is the random feature map: a low-dimensional $z$ with $\mathbb{E}[z(x)^\top z(w)]=k(x,w)$, so that $K\approx ZZ^\top$ is handled in $O(NM)$ time and memory. For attention in particular, this is the random-feature linear-attention route \citep{katharopoulos2020transformers,choromanski2021rethinking}.

Which kernels admit such a map is dictated by their structure. \emph{Random Fourier features} (RFF) \citep{rahimi2007random} cover shift-invariant kernels $k(x,w)=\phi(x-w)$ through Bochner's theorem, sampling the nonnegative spectral measure whose Fourier transform is $\phi$; quasi-Monte Carlo \citep{avron2016quasi} and orthogonal frequencies \citep{yu2016orthogonal} reduce their variance, and the construction underpins large-scale kernel ridge regression \citep{avron2017random}. \emph{Polynomial sketches} and dot-product random features (TensorSketch, Fastfood, and the dot-product feature maps of \citealp{pham2013fast,le2013fastfood,kar2012random}, with oblivious high-degree sketches by \citealp{ahle2020oblivious} and complex-valued, variance-optimized improvements by \citealp{wacker2024improved}) do the same for dot-product kernels $k(x,w)=\kappa(x^\top w)$. \emph{Nystr\"om methods} \citep{williams2000using,musco2017recursive} take a different route, building a low-rank approximation from $m\ll N$ landmark columns at the mercy of the kernel's spectral decay. Each family is bound to a structure: RFF to shift-invariance, sketches to dot-product form, Nystr\"om to fast eigenvalue decay. A unifying thread is that any kernel written as a \emph{mixture} can be linearized by sampling the mixture, which extends RFF beyond Gaussian spectra to broad stationary families \citep{wilson2013gpkernels}; our radial factor is handled in exactly this way. What is specific here is that the mixture is applied to only the radial part of a kernel that is \emph{not} stationary overall, and is composed with an exact finite feature for the polynomial part, so the scheme sits at the intersection of mixture-based RFF and exact feature maps rather than inside either.

The kernel we wish to scale fits none of these molds. The biased $\yat$-kernel
\begin{equation}\label{eq:yat_biased}
k_{\yat,b}(w,x) = \frac{(w^\top x + b)^2}{\|w-x\|^2 + \varepsilon}, \qquad b\ge0,\ \varepsilon>0,
\end{equation}
couples alignment and proximity: a squared inner product in the numerator, an inverse squared-distance in the denominator. The finite-difference identity in Proposition~\ref{prop:imq_findiff} shows that the family span $\operatorname{span}\{k_{\yat,b}:b\ge0\}$ contains the inverse-multiquadric (IMQ) kernel. We do not rely on any single-kernel universality statement in this paper; fixed-$b>0$ universality is discussed separately by \citet{bouhsine2026action}. In the construction and experiments we treat a fixed $b>0$ as a practical inductive-bias parameter, and the unbiased $k_{\yat}$ ($b=0$), whose RKHS functions all vanish at the origin, is the special case. But $k_{\yat,b}$ is \emph{not} shift-invariant and \emph{not} a dot-product kernel: the radial denominator $\|x-w\|^2$ alone is shift-invariant, yet multiplying it by the alignment numerator $(x^\top w+b)^2$ destroys both forms at once, so neither Bochner sampling nor polynomial sketching applies to the full kernel, and exact methods revert to the $O(N^2)$ Gram matrix. (On the unit sphere $\|x-w\|^2=2-2x^\top w$, so $k_{\yat,b}$ coincides with a rational dot-product kernel and a dimension-efficient dot-product route becomes available; deriving that route is left to future work, and we target the general $\mathbb{R}^d$ construction here, where no dot-product reduction of the \emph{full} kernel exists. One weaker factorization does survive off-sphere: at each fixed scale the Gaussian factors through positive per-point envelopes, $e^{-t\|x-w\|^2}=e^{-2t\|x\|^2}e^{-2t\|w\|^2}\,\mathbb{E}[\cdots]$, the FAVOR$^+$ route of Proposition~\ref{prop:positive}; but composed with the natural mixing law it has \emph{infinite variance} on every pair with $x^\top w\ge\varepsilon/8$ (Proposition~\ref{prop:pos_dichotomy}), so the trigonometric construction below is not one choice among equals.) The key empirical test is therefore off-sphere (Figure~\ref{fig:offsphere}): on a bounded ball with varying norms (Section~\ref{sec:exp_offsphere}) the kernel is genuinely non-dot-product, and RAY (Random Approximation of the $\yat$-kernel) retains its $O(1/\sqrt D)$ behavior while landmark methods degrade with $d$.

The obstacle dissolves once $k_{\yat,b}$ is read as a Schur product (Eq.~\ref{eq:schur}) of two individually tractable pieces: the biased degree-2 polynomial kernel $(w^\top x+b)^2$, which has an exact finite feature map, and the IMQ-type radial kernel $(\|w-x\|^2+\varepsilon)^{-1}$, which, being completely monotone in $\|x-w\|^2$, is a Bernstein--Widder mixture of Gaussians over a single scale parameter (the classical Schoenberg radial kernels \citep{schoenberg1938metric}, for which broad scale-mixture samplers exist beyond this IMQ case \citep{langrene2024mixture}). We turn this into a feature scheme that randomizes \emph{both} factors: sample the radial factor's one-dimensional scale, apply random Fourier features to the resulting Gaussian, and sketch the polynomial factor. Keeping the polynomial exact is the analyzable limit ($m\to\infty$), where the variance and concentration are sharpest; but that feature carries the $O(d^2)$ polynomial size, whereas sketching makes the dimension $Dm$, free of $d^2$, at the cost of one controllable sketch term (Theorem~\ref{thm:ts_opnorm}). Beyond the radial random Fourier features, the only added randomness is the one-dimensional Bernstein scale, so for a fixed dataset the radial draw count carries no explicit $d$-dependence (as for standard RFF), the genuine contrast with uniform Nystr\"om, which we confirm degrades with $d$.

The point is not that products of kernels can be tensored, a standard closure property, but that factoring this particular nonstationary, non-dot-product kernel turns each piece into a textbook object (an exact finite polynomial feature and a Gaussian mixture), so a kernel fitting neither random-feature template is linearized by combining the standard tools for each. This places RAY among random-feature methods that reach beyond Bochner sampling: for dot-product \citep{kar2012random,pham2013fast}, compositional \citep{daniely2017random}, and non-stationary spectral or harmonizable \citep{remes2017nonstationary,ton2018spatial,shen2019harmonizable} kernels, from which it differs by factoring the kernel into a finite modulation and a completely monotone radial factor and randomizing each by its native tool, a polynomial sketch and a Bernstein--Widder radial sampler (Section~\ref{sec:discussion}).

The contribution is the analysis of the resulting estimator and the class it defines. For exact modulation we prove unbiasedness, an exact variance for the flat estimator, an expected matrix-Bernstein operator-norm bound (with a matching tail) controlled by the top eigenvalues of the kernel and modulation Gram matrices together with an intrinsic dimension, and relative-spectral kernel-ridge stability; conditioning on the sketch carries the operator-norm guarantee to the deployed doubly-randomized estimator. Unbiasedness, the variance, and the uniform entrywise bound lift verbatim to the \emph{Bernstein--Schur} class (a finite-feature kernel times a completely monotone radial kernel, of which $k_{\yat,b}$ is the flagship) by a single theorem (Theorem~\ref{thm:bernstein_schur}); the matrix-Bernstein and kernel-ridge statements, proved for $k_{\yat,b}$, extend to every member with the modulation Gram in place of the polynomial one (Theorem~\ref{thm:class_bernstein}). The construction is self-contained: positive-definiteness from the Schur factorization (Eq.~\ref{eq:schur}), the IMQ span from the finite-difference identity (Proposition~\ref{prop:imq_findiff}), and the radial mixture from Bernstein--Widder (Section~\ref{sec:step-bernstein}); we rely on \citet{bouhsine2026action} only for the historical infinite-feature interpretation and fixed-$b$ universality not needed here.

\paragraph{Contributions.}
\begin{enumerate}[label=(\roman*),leftmargin=2em]
\item \textbf{A random-feature construction for a kernel class, not one kernel.} We identify \emph{Bernstein--Schur kernels}, a finite-feature kernel times a completely monotone shift-invariant kernel, and give one estimator for the whole class: keep the finite feature exact, sample the radial Bernstein--Widder scale, and apply random Fourier features (Theorem~\ref{thm:bernstein_schur}, unbiased with variance and uniform bounds). The biased $\yat$-kernel is the flagship instance (Section~\ref{sec:ryf}, Proposition~\ref{prop:nonstationary} shows it is genuinely neither stationary nor dot-product); the deployed estimator randomizes both factors, with exact modulation as the analyzable limit (Section~\ref{sec:step-sketch}).
\item \textbf{Sharp variance and optimal allocation} (Section~\ref{sec:guarantees}): the \emph{exact} variance of the recommended flat estimator (Theorem~\ref{thm:exact_variance}), a two-level identity proving one frequency per scale is variance-optimal within the hierarchical estimator at a fixed radial-feature budget (Proposition~\ref{prop:budget}), and a normalized variant whose variance is bounded free of the data radius and bias (Proposition~\ref{prop:normalized}).
\item \textbf{Gram concentration and high-probability KRR guarantees}: a uniform Gram bound (Theorem~\ref{thm:uniform}) and an expected matrix-Bernstein operator-norm bound (Theorem~\ref{thm:bernstein}) with a matching high-probability tail (Corollary~\ref{cor:bernstein_tail}), controlled by the top eigenvalues of the kernel and modulation Gram matrices and an intrinsic dimension rather than the crude $N\max_{ij}$ entrywise route (while still allowing worst-case $N$-scaling through those spectra). Whitening the same argument at the ridge $A=K+\lambda I$ makes the ridge effective dimension $d_{\mathrm{eff}}(\lambda)=\operatorname{tr}(K(K+\lambda I)^{-1})$ the \emph{exact} intrinsic dimension of the matrix variance (Theorem~\ref{thm:krr_whitened}): $D=O\bigl((1+\|P\|_{\mathrm{op}}/\lambda)\log(d_{\mathrm{eff}}/\delta)\bigr)$ radial draws give the relative-spectral condition of Theorem~\ref{thm:krr_spectral} with high probability, hence coefficient stability and objective-value preservation to constants (Corollary~\ref{cor:krr_highprob}), carried to the deployed sketched estimator (Corollary~\ref{cor:krr_deployed}). Tilting the radial draw by a closed-form \emph{whitened leverage} (a quadratic form whose mean over the base law is exactly $d_{\mathrm{eff}}(\lambda)$) replaces the $\|P\|_{\mathrm{op}}/\lambda$ factor by $d_{\mathrm{eff}}(\lambda)$ itself, the leverage-sampled count of the random-feature KRR literature, with the same proof (Theorem~\ref{thm:krr_leverage}). All matrix-level guarantees hold for the whole Bernstein--Schur class with the modulation Gram in place of $P$ (Theorem~\ref{thm:class_bernstein}); a full excess-risk theorem additionally needs control of the approximate kernel's effective dimension, which we leave open (Remark~\ref{rmk:risk}).
\item \textbf{RAY, the doubly-randomized estimator, and the signal-modulation decomposition.} The exact polynomial feature costs $O(d^2)$ per draw, so RAY sketches it: the feature dimension becomes $Dm$, free of $d^2$. Conditioning on the sketch, this doubly-randomized estimator carries the \emph{same} intrinsic-dimension operator-norm guarantee as the exact-modulation limit plus a single additive sketch term, tunable by $m$ independently of $D$ (Theorem~\ref{thm:ts_opnorm}); exact modulation is the $m\to\infty$ limit. A modulation--radial error decomposition (Proposition~\ref{prop:gate}) explains why the product still helps: radial approximation noise is modulated by the true finite feature, modulation error is localized by radial proximity, and only an interaction term mixes them, so the alignment$\times$proximity product suppresses similarity false positives even when the modulation is sketched (Appendix~\ref{sec:exp_gate}).
\item \textbf{Empirical validation} (Section~\ref{sec:experiments}): an off-sphere bounded-ball test confirming the general-$\mathbb{R}^d$ niche where no direct dot-product reduction is available and Nystr\"om degrades with $d$ (Section~\ref{sec:exp_offsphere}), a TensorSketch deployment study showing the radial-plus-sketch error split and removal of the $O(d^2)$ modulation floor (Section~\ref{sec:exp_ts}), a controlled coupled-target preference test where the $\yat$-kernel is the best bias exactly on the alignment--proximity target and not on single-factor controls (Section~\ref{sec:exp_necessity}), and a large-scale streaming primal run where the Gram is impossible (Section~\ref{sec:exp_higgs}). Cost-matched caveats, variance/bias checks, downstream KRR, attention, and sphere-normalized sanity checks are deferred to Appendix~\ref{app:further}.
\end{enumerate}

Section~\ref{sec:prelim} fixes notation and the random-feature background; Section~\ref{sec:ryf} constructs the estimator step by step; Section~\ref{sec:guarantees} establishes its guarantees and variance reduction; Section~\ref{sec:experiments} reports the experiments; Sections~\ref{sec:discussion}--\ref{sec:conclusion} discuss and conclude. Proofs and further validations are in the appendix.

\begin{figure}[t]
\centering
\includegraphics[width=\textwidth]{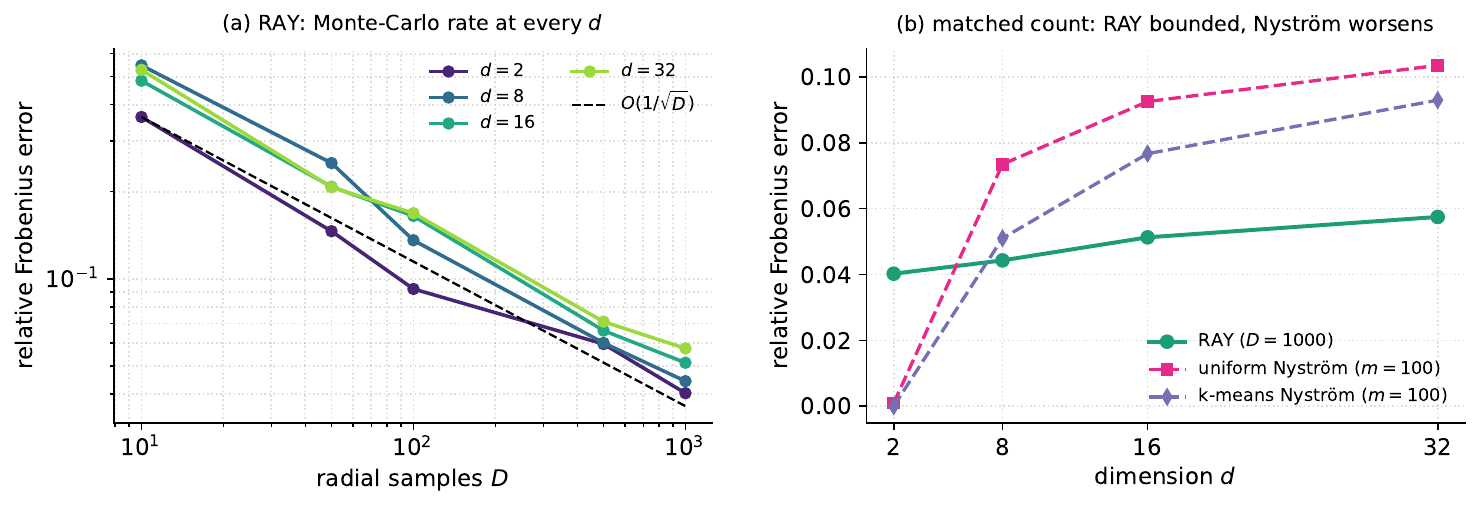}
\caption{The key regime: an \emph{off-sphere} bounded ball (varying norms), where $k_{\yat,b}$ is genuinely non-dot-product and no direct dot-product reduction is available. \textbf{(a)} RAY's relative Frobenius Gram error follows the $O(1/\sqrt D)$ Monte-Carlo rate at every dimension. \textbf{(b)} At $D=1000$ RAY stays bounded as $d$ grows, while uniform and k-means Nystr\"om (fixed $m=100$ landmarks) degrade (matched in radial/landmark count; the cost-matched comparison is Appendix~\ref{sec:exp_faircost}).}
\label{fig:offsphere}
\end{figure}

\section{Preliminaries}
\label{sec:prelim}

\subsection{Notation}
\label{sec:notation}
Vectors are columns in $\mathbb{R}^d$; $x^\top w$ is the inner product, $\|x\|$ the Euclidean norm, and $X=\{x_1,\dots,x_N\}\subset\mathbb{R}^d$ a compact dataset with $\|x_i\|\le R$. For a kernel $k$, $K=[k(x_i,x_j)]\in\mathbb{R}^{N\times N}$ is its Gram matrix; $\circ$ is the entrywise (Hadamard/Schur) product, so $(A\circ B)_{ij}=A_{ij}B_{ij}$, and $A\preceq B$ the Loewner order. We write $x\otimes x$ for the Kronecker square and $\vecc(\cdot)$ for its vectorization, so $\vecc(x\otimes x)^\top\vecc(w\otimes w)=(x^\top w)^2$. The target is the biased $\yat$-kernel $k_{\yat,b}$ of~\eqref{eq:yat_biased}, with bias $b\ge0$ and bandwidth $\varepsilon>0$; it factors (Section~\ref{sec:step-schur}) into the degree-2 modulation $p_b(x,w)=(x^\top w+b)^2$ with feature $p_b(x)$ of dimension $d_b=d(d+1)/2+d+1$, and the radial factor $h_\varepsilon(x,w)=(\|x-w\|^2+\varepsilon)^{-1}$. A random feature map $z:\mathbb{R}^d\to\mathbb{R}^M$ satisfies $\mathbb{E}[z(x)^\top z(w)]=k(x,w)$; the deployed estimator uses $D$ radial draws $t_j\sim\Exp(\varepsilon)$ with frequencies $\omega_j\sim\mathcal{N}(0,2t_jI_d)$ and a degree-2 sketch $\mathrm{TS}_m$ of width $m$, giving feature dimension $M=Dm$. We write $\|\cdot\|_{\mathrm{op}}$ for the operator norm and $\Exp(\varepsilon)$ for the exponential law of rate $\varepsilon$.

\subsection{Random features and positive-definite kernels}
\label{sec:background}
A continuous symmetric $k$ is positive definite if every Gram $K$ is PSD; by Mercer's theorem it then has a nonnegative spectral decomposition on compact $X$. Two closure properties drive the construction. \emph{Schur product:} if $k_1,k_2$ are PD then so is $k_1k_2$, with $K_1\circ K_2\succeq0$ \citep{schur1911bemerkungen}. \emph{Mixtures:} a nonnegative mixture $\int k_t\,d\mu(t)$ of PD kernels is PD, and sampling $t\sim\mu$ yields an unbiased random feature — the principle behind random Fourier features for shift-invariant kernels \citep{rahimi2007random} and their scale-mixture extensions \citep{wilson2013gpkernels}. A function is \emph{completely monotone} in $\|x-w\|^2$ if it is a Laplace transform of a nonnegative measure; by Bernstein--Widder it is then a Gaussian scale mixture (Section~\ref{sec:step-bernstein}). The quality of a random-feature Gram $\widehat K=ZZ^\top$ is measured here by the relative Frobenius error $\|\widehat K-K\|_F/\|K\|_F$ and, for downstream stability, by the operator-norm error $\|\widehat K-K\|_{\mathrm{op}}$, whose effective scale is set not by $N$ but by an intrinsic dimension of the kernel spectrum (Section~\ref{sec:guarantees}).

\section{Constructing the Random \texorpdfstring{$\yat$}{Yat}-Feature}
\label{sec:ryf}

We seek a feature map $z:\mathbb{R}^d\to\mathbb{R}^{M_b}$ with $\mathbb{E}[z(x)^\top z(w)]=k_{\yat,b}(x,w)$, so that the Gram matrix of $X=\{x_1,\dots,x_N\}$ is approximated by $ZZ^\top$ in $O(NM_b)$ rather than $O(N^2)$. We build $z$ in five steps: (i) factor $k_{\yat,b}$ into a polynomial and a radial kernel (Section~\ref{sec:step-schur}); (ii) write the radial kernel as a nonnegative mixture of Gaussians over a one-dimensional scale (Section~\ref{sec:step-bernstein}); (iii) discretize the mixture by sampling that scale (Section~\ref{sec:step-sample}); (iv) approximate the Gaussian factor by random Fourier features and tensor with the modulation feature (Section~\ref{sec:step-rff}); and (v) sketch the modulation, so that \emph{both} factors are randomized and the feature dimension is free of the $O(d^2)$ polynomial size (Section~\ref{sec:step-sketch}). Steps~(i)--(iv) with an exact modulation give the analyzable limit; step~(v) is the deployed estimator.

\subsection{Step 1: Schur factorization}
\label{sec:step-schur}

The kernel splits as a Schur product of a polynomial and a radial kernel,
\begin{equation}\label{eq:schur}
k_{\yat,b}(w,x) = \underbrace{(w^\top x + b)^2}_{p_b(w,x)} \cdot \underbrace{\frac{1}{\|w-x\|^2 + \varepsilon}}_{h_\varepsilon(w,x)},
\end{equation}
both factors positive definite ($p_b$ from its explicit feature map $p_b(x)^\top p_b(w)$ (Proposition~\ref{prop:biased_feature}), and $h_\varepsilon$ as a nonnegative Bernstein--Widder mixture of Gaussian kernels (Section~\ref{sec:step-bernstein})), so $k_{\yat,b}$ is PSD by the Schur product theorem \citep{schur1911bemerkungen}. Being continuous and positive definite, $k_{\yat,b}$ is a Mercer kernel in the usual integral-operator sense on every compact $X\subset\mathbb{R}^d$, the setting throughout (Section~\ref{sec:guarantees} fixes $\|x\|\le R$); we make no claim on noncompact domains. (This is self-contained; \citet{bouhsine2026action} is cited only for historical context and fixed-$b$ universality results not used here.) The polynomial factor is the tractable one: it has an exact, finite feature map. That the full kernel nonetheless sits outside both random-feature templates is not a matter of degree but a structural fact.

The radial factor is itself recovered from the family by an elementary \emph{forward} second difference in $b$: $(2h^2)^{-1}\bigl[k_{\yat,b+2h}-2k_{\yat,b+h}+k_{\yat,b}\bigr]=h_\varepsilon$ exactly, for any $b\ge0$, $h>0$, with all three biases inside the domain $b\ge0$ (Proposition~\ref{prop:imq_findiff}, Appendix~\ref{app:proofs}). This is the only ``spans IMQ'' statement we use, and it keeps the claim self-contained.

\begin{proposition}[Off-sphere the kernel is neither stationary nor dot-product]\label{prop:nonstationary}
For $b\ge0$, $\varepsilon>0$, the kernel $k_{\yat,b}$ is not shift-invariant on any domain containing two points of different norm, and not a dot-product kernel on any domain containing two pairs of equal inner product $s=x^\top w\ne-b$ but different distance (at $s=-b$ the numerator vanishes for both pairs, so this value is excluded). (On the unit sphere, where distance and inner product are in bijection, it does reduce to a dot-product kernel.)
\end{proposition}
\noindent The witnesses are immediate: $k_{\yat,b}(x,x)=(\|x\|^2+b)^2/\varepsilon$ varies with $\|x\|$, ruling out shift-invariance; and $x{=}w{=}e_1$ versus $x'{=}\sqrt2\,e_1,\,w'{=}e_1/\sqrt2$ have $x^\top w=x'^\top w'=1$ but $\|x-w\|^2=0\ne\tfrac12=\|x'-w'\|^2$, so $k_{\yat,b}$ is not a function of $x^\top w$ alone. This is why neither Bochner sampling nor dot-product sketching applies to the full kernel, and why the off-sphere experiment (Section~\ref{sec:exp_offsphere}) is the key test.

\begin{proposition}[Biased polynomial feature]\label{prop:biased_feature}
For $b\ge0$ define
\begin{equation}\label{eq:biased_poly}
p_b(x) = \bigl(\vecc(x\otimes x),\; \sqrt{2b}\,x,\; b\bigr)^\top \in \mathbb{R}^{d^2+d+1}.
\end{equation}
Then $p_b(x)^\top p_b(w) = (x^\top w)^2 + 2b(x^\top w) + b^2 = (x^\top w+b)^2$, and $\|p_b(x)\|^2 = (\|x\|^2+b)^2$.
\end{proposition}
\begin{proof}
$p_b(x)^\top p_b(w) = \vecc(x\otimes x)^\top\vecc(w\otimes w) + 2b\,x^\top w + b\cdot b = (x^\top w)^2 + 2b(x^\top w) + b^2$, where the constant coordinate $b$ contributes $b\cdot b=b^2$. Setting $w=x$ gives $\|p_b(x)\|^2 = \|x\|^4 + 2b\|x\|^2 + b^2 = (\|x\|^2+b)^2$.
\end{proof}

This finite feature is the object the deployed estimator will \emph{sketch}. We develop the construction in two passes. First, holding the modulation exact, the estimator's randomness is entirely radial, so the unbiasedness and variance identities are transparent (Section~\ref{sec:guarantees}); exact modulation is the $m\to\infty$ analyzable limit. The deployed estimator then sketches this feature (Step~5, Section~\ref{sec:step-sketch}), making the feature dimension free of the $O(d^2)$ polynomial size at the price of a single, controllable modulation-sketch term on top of the radial one (Theorem~\ref{thm:ts_opnorm}). Richer modulations (attention-style features or learned anchors) slot into the same two-error decomposition. The feature is real exactly on the kernel's domain $b\ge0$ (the entry $\sqrt{2b}$ requires it), consistent with the definition~\eqref{eq:yat_biased}. By symmetry the dimension reduces from $d^2$ to $d(d+1)/2$ using the upper triangle of $x\otimes x$ with $\sqrt2$ scaling on off-diagonals; for large $d$, TensorSketch \citep{pham2013fast} reduces it further to a controllable $D_{\mathrm{poly}}$ (Appendix~\ref{app:further}). Setting $b=0$ recovers the unbiased feature $p_0(x)=\vecc(x\otimes x)$. From here on, unless stated otherwise, $p_b$ denotes the symmetrized degree-2 feature of dimension $d_b=d(d+1)/2+d+1$, exactly equivalent in inner product to the redundant $d^2+d+1$ tensor feature of~\eqref{eq:biased_poly}.

\subsection{Step 2: Bernstein--Widder representation of the radial factor}
\label{sec:step-bernstein}

The radial factor $h_\varepsilon$ is completely monotone in $\|x-w\|^2$, so by the Bernstein--Widder theorem \citep{widder1941laplace} it is a Laplace mixture of Gaussian kernels,
\begin{equation}\label{eq:bernstein}
h_\varepsilon(x,w) = \frac{1}{\|x-w\|^2+\varepsilon} = \int_0^\infty e^{-t(\|x-w\|^2+\varepsilon)}\,dt = \int_0^\infty e^{-t\varepsilon}\,g_t(x,w)\,dt,
\end{equation}
where $g_t(x,w)=e^{-t\|x-w\|^2}$ is the Gaussian kernel at scale $t$. Multiplying by the polynomial factor and tensoring its feature gives the exact feature map of $k_{\yat,b}$ (the infinite-feature interpretation, derived here and discussed historically by \citealp{bouhsine2026action}): with $\Phi_b(x)(t)=e^{-t\varepsilon/2}\varphi_t(x)\otimes p_b(x)$ ($\varphi_t$ the canonical Gaussian feature map),
\begin{equation}\label{eq:feature_map}
\langle\Phi_b(x),\Phi_b(w)\rangle = \int_0^\infty e^{-t\varepsilon}\,g_t(x,w)\,(x^\top w+b)^2\,dt = k_{\yat,b}(x,w).
\end{equation}
The mixing measure $\varepsilon e^{-\varepsilon t}\,dt$ is, up to normalization, the density of $T\sim\Exp(\varepsilon)$. Since $\mathbb{E}_{T\sim\Exp(\varepsilon)}[e^{-T\|x-w\|^2}]=\varepsilon/(\|x-w\|^2+\varepsilon)$, the kernel is an expectation over that single scale,
\begin{equation}\label{eq:expectation}
k_{\yat,b}(x,w)=\frac{1}{\varepsilon}\,\mathbb{E}_{T\sim\Exp(\varepsilon)}\!\bigl[g_T(x,w)\,(x^\top w+b)^2\bigr],
\end{equation}
and it is this expectation, rather than the integral~\eqref{eq:feature_map}, that the next step discretizes.

\subsection{Step 3: Sampling the radial scale}
\label{sec:step-sample}

Equation~\eqref{eq:expectation} replaces the integral over $t$ by an expectation over a \emph{one-dimensional} random scale. Drawing $t_1,\dots,t_D\overset{\mathrm{iid}}{\sim}\Exp(\varepsilon)$ and averaging gives the Monte-Carlo estimate
\begin{equation}\label{eq:mc}
k_{\yat,b}(x,w)\approx \frac{(x^\top w+b)^2}{\varepsilon}\cdot\frac1D\sum_{j=1}^D g_{t_j}(x,w).
\end{equation}
At this point only the radial scale is sampled and the modulation is still exact, the analyzable limit, whose single (radial) error source makes the unbiasedness and variance identities of Section~\ref{sec:guarantees} clean. Step~5 then sketches the modulation, randomizing \emph{both} factors and removing the $O(d^2)$ floor at the cost of one additional, controllable error source (Theorem~\ref{thm:ts_opnorm}). This shapes the scheme's dimension behavior: the outer Bernstein scale $t_j$ is one-dimensional, and while the inner Gaussian frequency $\omega\in\mathbb{R}^d$ of the radial RFF is still $d$-dimensional, the fixed-dataset Hoeffding/union-bound count for the radial estimate carries no explicit $d$-dependence (Corollary~\ref{cor:sample_complexity}), as it does for standard RFF; the genuine $d$-contrast is in applicability and in the polynomial feature dimension, discussed in Section~\ref{sec:guarantees}. It remains to make $g_{t_j}$ computable by an inner product.

\subsection{Step 4: Random Fourier features and the estimator}
\label{sec:step-rff}

Each Gaussian $g_{t}(x,w)=e^{-t\|x-w\|^2}$ is shift-invariant, so Bochner's theorem applies to it (even though it does not apply to $k_{\yat,b}$): drawing $\omega\sim\mathcal{N}(0,2tI_d)$ and $\beta\sim\mathrm{Unif}([0,2\pi])$, $\mathbb{E}_{\omega,\beta}[2\cos(\omega^\top x+\beta)\cos(\omega^\top w+\beta)]=g_t(x,w)$ \citep{rahimi2007random}. Tensoring $D'$ such features with the exact polynomial feature yields the estimator; we write $d_b:=d(d{+}1)/2+d+1$ for the symmetric polynomial feature (the default; the redundant $d^2$ tensor is reduced losslessly) and $M_b=D\,D'\,d_b$.

\begin{definition}[Exact-modulation $\yat$-feature]\label{def:ryf}
Given $D,D'\in\mathbb{N}$, $\varepsilon>0$, $b\ge0$, the \emph{exact-modulation} $\yat$-feature map $z:\mathbb{R}^d\to\mathbb{R}^{M_b}$ (the analyzable limit, which the deployed RAY map of Step~5 randomizes further) is:
\begin{enumerate}[label=\arabic*.]
\item Draw $t_1,\dots,t_D\overset{\mathrm{iid}}{\sim}\Exp(\varepsilon)$.
\item For each $j$, draw $\omega_{j,1},\dots,\omega_{j,D'}\overset{\mathrm{iid}}{\sim}\mathcal{N}(0,2t_j I_d)$ and $\beta_{j,\ell}\overset{\mathrm{iid}}{\sim}\mathrm{Unif}([0,2\pi])$ (the factor $2$ is the standard RFF convention so that $\mathbb{E}_{\omega,\beta}[2\cos(\omega^\top x+\beta)\cos(\omega^\top w+\beta)]=e^{-t_j\|x-w\|^2}$).
\item Gaussian RFF at scale $t_j$: $\psi_{t_j}(x)=\sqrt{2/D'}\,(\cos(\omega_{j,\ell}^\top x+\beta_{j,\ell}))_{\ell=1}^{D'}\in\mathbb{R}^{D'}$.
\item Exact biased polynomial feature: $p(x)=\varepsilon^{-1/2}p_b(x)\in\mathbb{R}^{d_b}$.
\item Block: $z_j(x)=D^{-1/2}\,\psi_{t_j}(x)\otimes p(x)$. Concatenate: $z(x)=(z_1(x)^\top,\dots,z_D(x)^\top)^\top$.
\end{enumerate}
\end{definition}

The two levels $D$ and $D'$ are less fundamental than they look, and naming what the inner draw actually samples makes this clear. The hierarchical draw $t\sim\Exp(\varepsilon)$, $\omega\sim\mathcal{N}(0,2tI_d)$ produces $\omega$ with marginal density $\int_0^\infty\mathcal{N}(\omega;0,2tI_d)\,\varepsilon e^{-\varepsilon t}\,dt$, which by the same Gaussian scale mixture~\eqref{eq:bernstein} is the normalized Bochner spectral distribution of the rescaled radial factor $\varepsilon h_\varepsilon$; the missing total mass $1/\varepsilon$ is carried by the feature scaling. Step~4 is therefore standard random Fourier features for $h_\varepsilon$ tensored with the exact polynomial feature. The IMQ spectral density can be written using modified Bessel functions; the radial-scale mixture is simply a tractable sampler for that distribution, not a device that competes with it. Seen this way the inner level is redundant: at equal cost $DD'$ the $D'$ frequencies of a block share one scale $t_j$ and are correlated, whereas $D'$ independent spectral draws are not, so variance is minimized at $D'=1$ except in the degenerate zero-outer-variance case (Appendix~\ref{sec:exp_budget} confirms this directly). We keep $D'$ general for the analysis but recommend $D'=1$, drawing $D$ independent $(t_j,\omega_j)$ pairs, which collapses the estimator to
\begin{equation}\label{eq:flat}
z(x)=D^{-1/2}\Bigl(\sqrt2\,\cos(\omega_j^\top x+\beta_j)\,p(x)\Bigr)_{j=1}^D,\qquad \omega_j\sim\mathcal{N}(0,2t_jI_d),\ t_j\sim\Exp(\varepsilon).
\end{equation}
The contribution of this section is thus the Schur factorization that isolates an exactly-representable, sketchable polynomial modulation (Section~\ref{sec:step-schur}) paired with this sampler for the IMQ factor, not an improvement over IMQ random features; the dimension comparison this invites is taken up in Section~\ref{sec:comparison}.

\begin{theorem}[Unbiasedness]\label{thm:unbiased}
$\mathbb{E}[z(x)^\top z(w)]=k_{\yat,b}(x,w)$.
\end{theorem}
\begin{proof}
$z(x)^\top z(w)=\frac1D\sum_{j=1}^D(\psi_{t_j}(x)^\top\psi_{t_j}(w))(p(x)^\top p(w))$. The polynomial inner product is exact: $p(x)^\top p(w)=(x^\top w+b)^2/\varepsilon$. By the standard RFF identity \citep{rahimi2007random}, $\mathbb{E}_\omega[\psi_{t_j}(x)^\top\psi_{t_j}(w)\mid t_j]=e^{-t_j\|x-w\|^2}$. Taking expectation over $t_j\sim\Exp(\varepsilon)$ and using~\eqref{eq:expectation},
\[
\mathbb{E}[z(x)^\top z(w)] = \mathbb{E}_{T\sim\Exp(\varepsilon)}\!\bigl[e^{-T\|x-w\|^2}\bigr]\cdot\frac{(x^\top w+b)^2}{\varepsilon}=\frac{\varepsilon}{\|x-w\|^2+\varepsilon}\cdot\frac{(x^\top w+b)^2}{\varepsilon}=k_{\yat,b}(x,w).\qedhere
\]
\end{proof}

The construction used nothing specific to $k_{\yat,b}$ beyond its structure as a finite-feature kernel times a completely monotone shift-invariant kernel. The estimator, and its guarantees, hold for the whole class.

\begin{remark}[Differentiability and learnable $(\varepsilon,b)$]\label{rmk:learnable}
The estimator is differentiable in its own kernel parameters with the base randomness held fixed. Reparameterize the radial draw as $t_j=-\varepsilon^{-1}\log u_j$, $u_j\sim\mathrm{Unif}[0,1]$, and $\omega_j=\sqrt{2t_j}\,g_j$, $g_j\sim\mathcal N(0,I_d)$, so $\varepsilon$ enters $z(x)$ smoothly through $t_j$ and $\omega_j$, while $b$ enters through the exact polynomial feature $p_b$. At every fixed $(\varepsilon,b)$ the map is unbiased (Theorem~\ref{thm:unbiased}), so a gradient of any downstream objective in $(\varepsilon,b)$ flows through an unbiased kernel estimate and per-head learnable $(\varepsilon,b)$ in $\yat$-attention is a reparameterization away, with $\varepsilon$ free to adapt to the local attention sharpness (the hard regime of Appendix~\ref{sec:exp_attention}).
\end{remark}

\begin{theorem}[Bernstein--Schur random features]\label{thm:bernstein_schur}
Let $k(x,w)=p(x,w)\,f(\|x-w\|^2)$ with $p(x,w)=u(x)^\top u(w)$ for a finite feature $u:\mathbb{R}^d\to\mathbb{R}^{d_p}$, and $f$ completely monotone with Bernstein mixture $f(r)=\int_0^\infty e^{-tr}\,d\nu(t)$ of \emph{finite} mass $m_f:=\nu(\mathbb{R}_{\ge0})=f(0)<\infty$ (distinct from the sketch dimension $m$ of Section~\ref{sec:step-sketch}). Draw $T_j\sim\nu/m_f$, $\omega_j\mid T_j\sim\mathcal{N}(0,2T_jI_d)$, $\beta_j\sim\mathrm{Unif}[0,2\pi]$, and set $z(x)=\bigl(\sqrt{m_f/D}\,\sqrt2\cos(\omega_j^\top x+\beta_j)\,u(x)\bigr)_{j=1}^D$. Then $\mathbb{E}[z(x)^\top z(w)]=k(x,w)$, and if $\|u(x)\|\le B$ on $X$, then $\Var[z(x)^\top z(w)]\le 3m_f^2B^4/(2D)$ and, with probability $\ge1-\delta$, $\sup_{i,j}|z(x_i)^\top z(x_j)-k(x_i,x_j)|\le m_fB^2\sqrt{8\log(2N^2/\delta)/D}$.
\end{theorem}
\noindent We call these \emph{Bernstein--Schur kernels}. The biased $\yat$-kernel is the case $u=p_b$, $f(r)=(r+\varepsilon)^{-1}$ with $d\nu(t)=e^{-\varepsilon t}\,dt$, so the mass is $m_f=\nu(\mathbb{R}_{\ge0})=1/\varepsilon$ and the law $\nu/m_f$ is $\Exp(\varepsilon)$; with $B=\max_x\|p_b(x)\|=R^2+b$ the variance scale $m_f^2B^4=(R^2+b)^4/\varepsilon^2$ recovers the bounds of Section~\ref{sec:guarantees}. Theorem~\ref{thm:bernstein_schur} establishes unbiasedness, the variance, and the uniform entrywise bound for the whole class; the matrix-Bernstein (Theorem~\ref{thm:bernstein}) and kernel-ridge (Theorems~\ref{thm:krr_spectral} and~\ref{thm:krr_whitened}) statements we prove for $k_{\yat,b}$ and then extend to every member in Theorem~\ref{thm:class_bernstein}, with $P$ replaced by $m_f$ times the bounded modulation Gram. The class is an elementary closure construction, but it is useful because it packages a family of kernels that are often neither stationary nor pure dot-product, and it permits unified unbiased features, exact variance formulas, and matrix-level approximation guarantees. Table~\ref{tab:bernstein_schur} lists nontrivial members; only $u$ and $\nu$ change between them.

\begin{table}[h]
\centering
\caption{Bernstein--Schur kernels $k=p\cdot f$. Each row satisfies the finite-mass hypothesis $m_f=f(0)<\infty$ of Theorem~\ref{thm:bernstein_schur} and is linearized by the same estimator: keep the finite feature of the modulation $p$ exact, draw the radial scale from the (normalized) Bernstein measure of $f$, and apply random Fourier features to the Gaussian. The ``mixing'' column gives the normalized law $\nu/m_f$ (the IMQ has $d\nu=e^{-\varepsilon t}dt$, mass $m_f=1/\varepsilon$; the generalized IMQ $(r+\varepsilon)^{-\alpha}$ has $d\nu\propto t^{\alpha-1}e^{-\varepsilon t}dt$, mass $m_f=\varepsilon^{-\alpha}$; the Mat\'ern-$\tfrac12$ radial $e^{-\sqrt r/\sigma}$, completely monotone in $r=\|x-w\|^2$, has the L\'evy/inverse-Gaussian mixture with mass $m_f=1$ and the exact two-line sampler $T=1/(2\sigma^2Z^2)$, $Z\sim\mathcal{N}(0,1)$). Here $\Gamma(\alpha,\varepsilon)$ denotes shape $\alpha$ and rate $\varepsilon$. Only the two columns right of $f$ change between instances.}
\label{tab:bernstein_schur}
\setlength{\tabcolsep}{4pt}\small
\resizebox{\textwidth}{!}{%
\begin{tabular}{lllll}
\toprule
Modulation $p(x,w)$ & Radial $f(r)$ & Mixing $\nu$ & feat.\ dim $d_p$ & Inductive bias \\
\midrule
$(x^\top w+b)^2,\ b\ge0$ & IMQ $(r+\varepsilon)^{-1}$ & $\Exp(\varepsilon)$ & $\tfrac{d(d+1)}2{+}d{+}1$ & alignment $\times$ proximity (biased $\yat$) \\
$(x^\top w+b)^q,\,q\in\mathbb{N},b\ge0$ & $(r+\varepsilon)^{-\alpha},\,\alpha>0$ (gen.\ IMQ) & $\Gamma(\alpha,\varepsilon)$ & $\binom{d+q}{q}$ & $q$-way interactions $\times$ locality \\
$(x^\top w+b)^q,\,q\in\mathbb{N},b\ge0$ & Mat\'ern-$\tfrac12$ $e^{-\sqrt r/\sigma}$ & $T=\frac{1}{2\sigma^2Z^2}$, $Z\sim\mathcal{N}(0,1)$ (L\'evy) & $\binom{d+q}{q}$ & $q$-way interactions $\times$ exponential locality \\
$x^\top w+b,\ b\ge0$ & finite-mass compl.\ monotone $f$ & Bernstein meas.\ of $f$ & $d{+}1$ & signed local-linear trends \\
$u(x)^\top u(w)$ & IMQ / Gaussian mix & corresponding & $\dim u$ & data-driven modulation $\times$ proximity \\
\bottomrule
\end{tabular}}
\end{table}

\noindent These other members are linearized by the same estimator in practice, not only in principle: on a genuinely non-$\yat$ instance (a degree-3 modulation $(x^\top w+b)^3$ times a generalized-IMQ radial $(\|x-w\|^2+\varepsilon)^{-2}$ with $\Gamma(2,\varepsilon)$ mixing), the same scheme (keep $u$ exact, draw $T_j\sim\Gamma(2,\varepsilon)$, apply RFF) is unbiased and converges at the Monte-Carlo rate, exactly as for the $\yat$-kernel (Appendix~\ref{app:further}). The class-level theorem is thus more than a formal generalization: swapping $u$ and $\nu$ in one estimator linearizes a different nonstationary kernel. The class also has standing customers outside this paper. Compositional kernel search (the Automatic Statistician, \citealp{duvenaud2013structure}) emits GP kernels as products such as $\mathrm{LIN}^2\times\mathrm{RQ}$, and the rational-quadratic radial $\mathrm{RQ}_\alpha(r)=(1+r/(2\alpha\sigma^2))^{-\alpha}$ \emph{is} the generalized IMQ of Table~\ref{tab:bernstein_schur} after rescaling ($\varepsilon'=2\alpha\sigma^2$, mass $m_f=\mathrm{RQ}_\alpha(0)=1$, mixing law $\Gamma(\alpha,\varepsilon')$), so every such grammar product is a Bernstein--Schur kernel and inherits the estimator and the full matrix-level guarantee set with no new analysis, where previously it had no scaling story past the Gram. Run on California housing (off-sphere, $d=8$, $\alpha=2$), the class estimator of $(x^\top w+b)^2\,\mathrm{RQ}_\alpha$ converges to the exact composite Gram at the Monte-Carlo rate (fitted slope $-0.47$), matches exact composite-kernel ridge regression already at $D=200$ ($0.540$ vs.\ $0.543$ test RMSE), and fits the full $N=19{,}640$ training set in $1.2$ seconds as an explicit-feature primal where the $N\times N$ Gram has $4\times10^8$ entries (Appendix~\ref{app:further}).

\subsection{Step 5: Sketching the modulation, the deployed estimator}
\label{sec:step-sketch}

Definition~\ref{def:ryf} keeps the modulation feature $p_b$ exact, so its dimension $d_b=O(d^2)$ fixes the feature size $D\,D'\,d_b$. The deployed estimator removes this floor by randomizing the modulation as well. Let $\mathrm{TS}_m$ be a degree-2 TensorSketch of the augmented input $(x,\sqrt b)$, so $\mathbb{E}[\mathrm{TS}_m(x)^\top\mathrm{TS}_m(w)]=(x^\top w+b)^2$ over the sketch randomness \citep{pham2013fast,avron2014subspace}, and set $\widehat p_m(x)=\varepsilon^{-1/2}\mathrm{TS}_m(x)\in\mathbb{R}^m$. The deployed \emph{RAY} map (randomizing both factors) replaces step~4 of Definition~\ref{def:ryf} by $\widehat p_m$ (using the flat $D'=1$ form):
\begin{equation}\label{eq:doubly}
\widehat z(x)=D^{-1/2}\bigl(\sqrt2\,\cos(\omega_j^\top x+\beta_j)\,\widehat p_m(x)\bigr)_{j=1}^D\in\mathbb{R}^{Dm},\qquad \omega_j\sim\mathcal{N}(0,2t_jI_d),\ t_j\sim\Exp(\varepsilon).
\end{equation}
Both factors are now random (the radial scale by the Bernstein mixture, the modulation by the sketch) and the feature dimension $Dm$ is free of $d^2$. The experiments (Section~\ref{sec:exp_ts}) and Proposition~\ref{prop:ts_variance} instead use a lower-variance \emph{quadratic-only} variant that sketches only the degree-2 term $(x^\top w)^2$ and keeps the linear and constant terms $2b\,x^\top w+b^2$ exact, of dimension $D(m{+}d{+}1)$: still free of the $O(d^2)$ floor, larger only by an additive $O(d)$ per draw. Both forms are covered by Theorem~\ref{thm:ts_opnorm} (the analysis touches the modulation only through $\widehat P_m$). The exact-modulation estimator~\eqref{eq:flat} (Definition~\ref{def:ryf}) is the limit $m\to\infty$ ($\widehat p_m\to p_b$), where the modulation randomness vanishes; that limit is the object the variance and concentration identities of Section~\ref{sec:guarantees} are stated for, and Theorem~\ref{thm:ts_opnorm} carries the operator-norm guarantee to finite $m$ at the cost of one additional, controllable sketch term. The modulation randomizer is modular: any unbiased or low-rank PSD modulation feature plugs in (TensorSketch, random-Maclaurin products, or anchor features), the choice being a geometry-dependent design decision rather than part of the analysis (Section~\ref{sec:discussion}). We state the guarantees for exact modulation and then transfer them, so the next two sections read as: clean identities at the limit, then the doubly-randomized bound.

\begin{example}[One RAY coordinate, end to end]\label{ex:ray}
Fix $b=1$, $\varepsilon=1$, $d=3$. To build a single radial draw of the deployed map~\eqref{eq:doubly}: (1) draw a scale $t\sim\Exp(1)$, say $t=0.4$; (2) draw a frequency $\omega\sim\mathcal{N}(0,2tI_3)$ and phase $\beta\sim\mathrm{Unif}[0,2\pi)$, giving the radial Fourier feature $\sqrt2\cos(\omega^\top x+\beta)$, unbiased for the Gaussian $e^{-t\|x-w\|^2}$ at that scale; (3) form the modulation sketch $\widehat p_m(x)=\mathrm{TS}_m(x,\sqrt b)$, unbiased for $(x^\top w+1)^2$; (4) tensor them, $z_1(x)=\sqrt2\cos(\omega^\top x+\beta)\,\widehat p_m(x)\in\mathbb{R}^m$. Then $\mathbb{E}[z_1(x)^\top z_1(w)]=e^{-t\|x-w\|^2}(x^\top w+1)^2$, and averaging the scale over $t\sim\Exp(1)$ recovers $h_\varepsilon\cdot p_b=k_{\yat,1}$. Stacking $D$ such coordinates and dividing by $\sqrt D$ gives the $\mathbb{R}^{Dm}$ map. Setting $b=0$ and replacing $\widehat p_m$ by the exact $\vecc(x\otimes x)$ recovers the unbiased exact-modulation feature of Definition~\ref{def:ryf}.
\end{example}

\begin{remark}[Complex modulation sketches]\label{rmk:complex}
The modulation randomizer is the only part of RAY that need not be real. \citet{wacker2024improved} show that complex-valued polynomial sketches have strictly smaller variance than real Rademacher sketches for dot-product kernels, because the unit-modulus fourth moment is smaller; substituting unit-phase signs $s\in\{1,i,-1,-i\}$ in the quadratic-only $\mathrm{TS}_m$ and reading off $\operatorname{Re}\langle\mathrm{TS}_m(x),\overline{\mathrm{TS}_m(w)}\rangle$ leaves the estimator unbiased while lowering the additive sketch term in Theorem~\ref{thm:ts_opnorm}. The radial factor is unaffected. Measured on the protocol of Figure~\ref{fig:ts_opnorm}, the drop-in is real at every sketch size: the spectral error $\eta$ falls by $1.5$--$1.7\times$ (at $m{=}128$, $0.042\to0.027$), the sketch term $\|E_P\circ R\|_{\mathrm{op}}$ by the same factor ($3.84\to2.51$), and the per-entry sketch variance is halved (ratios $0.43$--$0.52$ across $m\in\{64,\dots,512\}$; Appendix~\ref{app:exp_details}, \texttt{complex\_sketch}), consistent with the fourth-moment computation of \citet{wacker2024improved}.
\end{remark}

\section{Approximation Guarantees}
\label{sec:guarantees}

\subsection{Variance and uniform error}

Throughout, $X\subset\mathbb{R}^d$ is compact with $\|x\|\le R$ for all $x\in X$. Proofs are in Appendix~\ref{app:proofs}.

The variance of the recommended flat estimator ($D'=1$) is known in closed form. A coarser envelope valid for all $D,D'$ ($V_D\le\frac{(R^2+b)^4}{D\varepsilon^2}(1+\frac{3}{2D'})$, the $3/(2D')$ term being the inner random-Fourier-feature variance) and the two-level identity~\eqref{eq:budget_var} whose budget argument singles out $D'=1$ as optimal are deferred to Appendix~\ref{app:proofs} (Theorem~\ref{thm:variance}, Proposition~\ref{prop:budget}); the experiments confirm $D'=1$ directly (Appendix~\ref{sec:exp_budget}).

\begin{theorem}[Exact variance of the flat estimator]\label{thm:exact_variance}
Write $r=\|x-w\|^2$ and $a=(x^\top w+b)^2$. The flat estimator, \emph{one} inner frequency per scale, $D'=1$ (not the $D'\to\infty$ inner-averaged estimator), $\widehat{k}_D=\tfrac{a}{\varepsilon}\tfrac1D\sum_{j=1}^D 2\cos(\omega_j^\top x+\beta_j)\cos(\omega_j^\top w+\beta_j)$, with $T_j\sim\Exp(\varepsilon)$ and $\omega_j\mid T_j\sim\mathcal{N}(0,2T_jI_d)$, has
\[
\Var[\widehat{k}_D(x,w)]=\frac{a^2}{D\varepsilon^2}\left[\,1+\frac12\frac{\varepsilon}{\varepsilon+4r}-\Bigl(\frac{\varepsilon}{\varepsilon+r}\Bigr)^2\,\right].
\]
At $r=0$ the bracket is $\tfrac12$ (the one-frequency variance does not vanish); the $D'\to\infty$ inner-averaged estimator instead has outer-only variance $\tfrac{a^2}{D\varepsilon^2}[\tfrac{\varepsilon}{\varepsilon+2r}-(\tfrac{\varepsilon}{\varepsilon+r})^2]$, the $V_{\mathrm{out}}$ term of Proposition~\ref{prop:budget}.
\end{theorem}
\noindent The identity locates the variance precisely: it is dominated by the fourth power of the bias-shifted alignment $a^2=(x^\top w+b)^4$ and scales as $\varepsilon^{-2}$, while the radial bracket equals $\tfrac12$ at $r=0$, tends to $1$ as $r\to\infty$, and is uniformly bounded by $\tfrac32$. So the alignment numerator, not the distance, sets the scale; this recovers the $O((R^2+b)^4/(D\varepsilon^2))$ order of Theorem~\ref{thm:variance} with exact constants.

A Hoeffding bound per entry plus a union bound gives a uniform entrywise Gram error $\sup_{ij}|z(x_i)^\top z(x_j)-k_{\yat,b}|\le(R^2+b)^2\varepsilon^{-1}\sqrt{8\log(2N^2/\delta)/D}$ with probability $1-\delta$ (Theorem~\ref{thm:uniform}), hence a sample complexity $D=O((R^2+b)^4\varepsilon^{-2}\tau^{-2}\log(N/\eta))$ free of explicit $d$ (Corollary~\ref{cor:sample_complexity}); both are in Appendix~\ref{app:proofs}. Converting this entrywise bound to operator norm through $\|A\|_{\mathrm{op}}\le N\max_{ij}|A_{ij}|$ costs a factor $N$ (Corollary~\ref{thm:gram_concentration}). That factor is wasteful: it ignores that the per-draw error matrices are structured. Each radial draw contributes $K^{(j)}=(\Psi_j\Psi_j^\top)\circ P$ with $P=[(x_i^\top x_j+b)^2/\varepsilon]$, a Schur product of two positive semidefinite matrices and hence itself PSD; matrix Bernstein exploits this.

\begin{theorem}[Expected matrix-Bernstein operator-norm bound]\label{thm:bernstein}
Let $K=[k_{\yat,b}(x_i,x_j)]$ and $P=[(x_i^\top x_j+b)^2/\varepsilon]$ be the kernel and polynomial Gram matrices (both PSD), and let $K_D$ be the exact-modulation estimate from $D$ i.i.d.\ radial draws. With $V=\sum_{j}\mathbb{E}[(D^{-1}(K^{(j)}-K))^2]$ the matrix variance and $d_{\mathrm{int}}=\operatorname{tr}(V)/\|V\|_{\mathrm{op}}\le N$ its intrinsic dimension (if $V=0$ the estimate is exact and the bound holds trivially with the convention $d_{\mathrm{int}}=1$),
\[
\mathbb{E}\,\|K_D-K\|_{\mathrm{op}} \;\le\; 3\sqrt{\frac{2\|P\|_{\mathrm{op}}\,\|K\|_{\mathrm{op}}\,\log(8d_{\mathrm{int}})}{D}}\;+\;\frac{6\|P\|_{\mathrm{op}}\,\log(8d_{\mathrm{int}})}{D}.
\]
\end{theorem}
\noindent The leading term scales with the top eigenvalues $\|P\|_{\mathrm{op}},\|K\|_{\mathrm{op}}$ and the effective rank $d_{\mathrm{int}}$, all $\ll N$ for spectrally concentrated data; since $\|P\|_{\mathrm{op}}\le\operatorname{tr}(P)\le N(R^2+b)^2/\varepsilon$ and $d_{\mathrm{int}}\le N$, the bound never exceeds the order of Corollary~\ref{thm:gram_concentration} and is data-adaptively tighter. The theorem is stated for exact modulation; Theorem~\ref{thm:class_bernstein} extends it, with the rest of the matrix-level guarantee set, to every Bernstein--Schur kernel after replacing $P$ by the bounded modulation Gram multiplied by the finite mass $m_f$. The key step is that $K^{(j)}\succeq0$, so $(K^{(j)})^2\preceq\|K^{(j)}\|_{\mathrm{op}}K^{(j)}\preceq2\|P\|_{\mathrm{op}}K^{(j)}$ (the Schur-multiplier bound, Lemma~\ref{lem:schur}, gives $\|K^{(j)}\|_{\mathrm{op}}\le2\|P\|_{\mathrm{op}}$). Hence $\|V\|_{\mathrm{op}}\le2\|P\|_{\mathrm{op}}\|K\|_{\mathrm{op}}/D$, and the intrinsic matrix-Bernstein inequality \citep{tropp2012user,tropp2015matrix} yields the bound. The full proof is in Appendix~\ref{app:proofs}; an empirical check across datasets of varying spectral spread confirms the data-adaptive constant (Appendix~\ref{sec:exp_gram}).

\begin{corollary}[High-probability operator-norm bound]\label{cor:bernstein_tail}
With the summands $Y_j=D^{-1}(K^{(j)}-K)$, the almost-sure bound $\|Y_j\|_{\mathrm{op}}\le L:=3\|P\|_{\mathrm{op}}/D$, and the matrix variance $\bigl\|\sum_j\mathbb{E}[Y_j^2]\bigr\|_{\mathrm{op}}\le v:=2\|P\|_{\mathrm{op}}\|K\|_{\mathrm{op}}/D$, the tail form of the intrinsic matrix-Bernstein inequality \citep[Thm.~7.3.1]{tropp2015matrix} gives, with probability at least $1-\delta$,
\begin{equation}\label{eq:bernstein_tail}
\|K_D-K\|_{\mathrm{op}}\le 2\sqrt{\frac{\|P\|_{\mathrm{op}}\|K\|_{\mathrm{op}}\log(8d_{\mathrm{int}}/\delta)}{D}}+\frac{2\|P\|_{\mathrm{op}}\log(8d_{\mathrm{int}}/\delta)}{D}.
\end{equation}
\end{corollary}
\noindent The proof (Appendix~\ref{app:proofs}) establishes $L$ and $v$ from $K^{(j)}\succeq0$ and the Schur-multiplier bound (Lemma~\ref{lem:schur}), then inverts the intrinsic Bernstein tail.

The same conditioning argument carries this bound to the \emph{doubly}-randomized estimator of Step~5, where the modulation is also sketched, the map one actually deploys.

\begin{theorem}[Operator-norm error of RAY]\label{thm:ts_opnorm}
Let $\widehat K_{D,m}=\widehat P_m\circ\widehat R_D$ be the doubly-randomized estimator~\eqref{eq:doubly}: a single degree-2 TensorSketch of dimension $m$ (drawn once, shared across draws) gives the modulation Gram $\widehat P_m=[\mathrm{TS}_m(x_i)^\top\mathrm{TS}_m(x_j)]\succeq0$ with $\mathbb{E}\,\widehat P_m=P=[(x_i^\top x_j+b)^2/\varepsilon]$, and $D$ radial draws give the radial Gram $\widehat R_D$ with $\mathbb{E}\,\widehat R_D=R$, the unit-diagonal Gram $R_{ij}=\varepsilon/(\varepsilon+\|x_i-x_j\|^2)$ (so $K=R\circ P$), independent of the sketch. Suppose the sketch satisfies the spectral event $\|\widehat P_m-P\|_{\mathrm{op}}\le\eta\|P\|_{\mathrm{op}}$ (Remark~\ref{rmk:ose} discusses when the degree-2 TensorSketch achieves it). Then on that event, with probability at least $1-\delta$ over the radial draws,
\begin{equation}\label{eq:ts_opnorm}
\|\widehat K_{D,m}-K\|_{\mathrm{op}}\le \underbrace{2\sqrt{\frac{(1+\eta)\|P\|_{\mathrm{op}}\|K_S\|_{\mathrm{op}}\log(8d_{\mathrm{int},S}/\delta)}{D}}+\frac{2(1+\eta)\|P\|_{\mathrm{op}}\log(8d_{\mathrm{int},S}/\delta)}{D}}_{\text{radial},\ O(D^{-1/2})}+\underbrace{\eta\,\|P\|_{\mathrm{op}}}_{\text{sketch}},
\end{equation}
where $K_S=\widehat P_m\circ R$, $\|K_S\|_{\mathrm{op}}\le(1+\eta)\|P\|_{\mathrm{op}}$, and $d_{\mathrm{int},S}=\operatorname{tr}(V_S)/\|V_S\|_{\mathrm{op}}\le N$ is the intrinsic dimension of the \emph{sketch-conditioned} matrix variance $V_S$ (the variance of Theorem~\ref{thm:bernstein} for the pair $(\widehat P_m,K_S)$ rather than $(P,K)$, reducing to $d_{\mathrm{int}}$ as $m\to\infty$). At $m\to\infty$ ($\eta\to0$) the sketch term vanishes and~\eqref{eq:ts_opnorm} is exactly Corollary~\ref{cor:bernstein_tail}: randomizing the modulation costs the single additive term $\eta\|P\|_{\mathrm{op}}$, set by $m$ independently of $D$.
\end{theorem}
\noindent \emph{Conditioning is the whole argument.} Given the sketch, $\widehat P_m$ is a fixed PSD modulation Gram and $\widehat K_{D,m}=\widehat P_m\circ\widehat R_D$ is an \emph{exact-modulation} estimator of $K_S=\widehat P_m\circ R$; Theorem~\ref{thm:bernstein} and Corollary~\ref{cor:bernstein_tail} apply verbatim with $P\mapsto\widehat P_m$ (so $\|\widehat P_m\|_{\mathrm{op}}\le(1+\eta)\|P\|_{\mathrm{op}}$ and $\|K_S\|_{\mathrm{op}}\le\|\widehat P_m\|_{\mathrm{op}}$ by Lemma~\ref{lem:schur} with $R$ unit-diagonal), giving the radial term. The remaining bias is $K_S-K=(\widehat P_m-P)\circ R=R\circ E_P$ with $E_P:=\widehat P_m-P$ symmetric but \emph{not} PSD; here the symmetric Schur-multiplier bound Lemma~\ref{lem:schur}(b) applies (with $A=R\succeq0$, $\max_i R_{ii}=1$), giving $\|R\circ E_P\|_{\mathrm{op}}\le\|E_P\|_{\mathrm{op}}\le\eta\|P\|_{\mathrm{op}}$. The triangle inequality combines the two; the split is validated in Section~\ref{sec:exp_ts}. The conservative entrywise Frobenius decomposition of Proposition~\ref{prop:ts_variance} needs no subspace-embedding hypothesis and remains available as a fallback (Appendix~\ref{app:proofs}).

\begin{remark}[Achieving the spectral event]\label{rmk:ose}
The degree-2 TensorSketch is an oblivious subspace embedding \citep{pham2013fast,avron2014subspace}: with sketch dimension $m$ polynomial in the statistical dimension $s_\lambda=\operatorname{tr}\!\bigl(P(P+\lambda I)^{-1}\bigr)$ and in $\eta^{-1}$ (the exact dependence is sketch- and degree-specific: up to degree-dependent factors, and the original degree-2 TensorSketch bound of \citet{avron2014subspace} is superlinear in $s_\lambda$), it delivers, with probability $1-\delta_s$, a $(1\pm\eta)$ \emph{ridge} subspace embedding $(1-\eta)(P+\lambda I)\preceq\widehat P_m+\lambda I\preceq(1+\eta)(P+\lambda I)$. This is a \emph{regularized} guarantee: the absolute event $\|\widehat P_m-P\|_{\mathrm{op}}\le\eta\|P\|_{\mathrm{op}}$ assumed in Theorem~\ref{thm:ts_opnorm} is its $\lambda\to0$ idealization (there $s_\lambda\to\operatorname{rank}(P)$ and the count becomes $\Omega(\eta^{-2}\operatorname{rank}(P))$). The ridge-embedding event is the more useful one, because it transfers through the Schur product without loss: Proposition~\ref{prop:ridge_sketch} carries it from the modulation Gram $P$ to the kernel Gram $K$ at the \emph{same} ridge scale, so the deployed estimator inherits a ridge-relative sketch guarantee with sketch size polynomial in $s_\lambda(P)$ rather than $\operatorname{rank}(P)$. The entrywise decomposition of Proposition~\ref{prop:ts_variance} needs no embedding hypothesis at all.
\end{remark}

\begin{proposition}[Ridge-relative sketch transfer]\label{prop:ridge_sketch}
Suppose the modulation sketch is a $(1\pm\eta)$ ridge subspace embedding of $P$ at scale $\lambda$, $(1-\eta)(P+\lambda I)\preceq\widehat P_m+\lambda I\preceq(1+\eta)(P+\lambda I)$ (Remark~\ref{rmk:ose}). Then the exact-radial deployed Gram $K_S=\widehat P_m\circ R$ satisfies the same sandwich at the same scale,
\begin{equation}\label{eq:ridge_transfer}
(1-\eta)(K+\lambda I)\preceq K_S+\lambda I\preceq(1+\eta)(K+\lambda I),
\qquad\text{equivalently}\qquad \bigl\|A^{-1/2}(K_S-K)A^{-1/2}\bigr\|_{\mathrm{op}}\le\eta,
\end{equation}
with $A=K+\lambda I$. Consequently $\|K_S-K\|_{\mathrm{op}}\le\eta(\|K\|_{\mathrm{op}}+\lambda)\le\eta(\|P\|_{\mathrm{op}}+\lambda)$, sharpening the additive sketch term $\eta\|P\|_{\mathrm{op}}$ of Theorem~\ref{thm:ts_opnorm} whenever $\lambda\ll\|P\|_{\mathrm{op}}$.
\end{proposition}
\begin{proof}
The radial Gram $R$ is PSD with unit diagonal, so $R\circ I=I$ and hence $R\circ(M+\lambda I)=R\circ M+\lambda I$ for every symmetric $M$; in particular $R\circ(P+\lambda I)=K+\lambda I$ and $R\circ(\widehat P_m+\lambda I)=K_S+\lambda I$. Both gaps in the hypothesis sandwich, $(1+\eta)(P+\lambda I)-(\widehat P_m+\lambda I)$ and $(\widehat P_m+\lambda I)-(1-\eta)(P+\lambda I)$, are PSD; Schur-multiplying each by $R\succeq0$ keeps it PSD (Lemma~\ref{lem:schur}(a)), which is~\eqref{eq:ridge_transfer} after substituting the two identities. Conjugating by $A^{-1/2}$ gives the whitened form. For the last claim, \eqref{eq:ridge_transfer} gives $-\eta(K+\lambda I)\preceq K_S-K\preceq\eta(K+\lambda I)$, so $\|K_S-K\|_{\mathrm{op}}\le\eta(\|K\|_{\mathrm{op}}+\lambda)$, and $\|K\|_{\mathrm{op}}\le\|P\|_{\mathrm{op}}$ by Lemma~\ref{lem:schur} ($R$ unit-diagonal).
\end{proof}

This propagates to downstream kernel ridge regression: the approximation perturbs the learned predictor only through the operator-norm Gram error.

\begin{theorem}[KRR stability under relative spectral approximation]\label{thm:krr_spectral}
For a target $y\in\mathbb{R}^N$ and ridge $\lambda>0$, let $\hat\alpha=(K+\lambda I)^{-1}y$, $\tilde\alpha=(K_D+\lambda I)^{-1}y$, $A=K+\lambda I$, and $E=K_D-K$. If the whitened error obeys $\|A^{-1/2}EA^{-1/2}\|_{\mathrm{op}}\le\rho<1$, then
\[
(1-\rho)A\preceq K_D+\lambda I\preceq(1+\rho)A,\qquad \|\tilde\alpha-\hat\alpha\|_A\le\frac{\rho}{1-\rho}\,\|\hat\alpha\|_A,
\]
where $\|v\|_A^2=v^\top A v$.
\end{theorem}
\noindent This is the spectral-approximation form standard in random-feature KRR analysis \citep{avron2017random}: control of the \emph{relative} error $\rho$, not the raw $\|K_D-K\|_{\mathrm{op}}$, governs the predictor. The theorem is deterministic; Theorem~\ref{thm:krr_whitened} supplies the high-probability condition by running the matrix-Bernstein argument of Theorem~\ref{thm:bernstein} on the \emph{whitened} summands $A^{-1/2}(K^{(j)}-K)A^{-1/2}$, where the ridge effective dimension $d_{\mathrm{eff}}(\lambda)=\operatorname{tr}(K(K+\lambda I)^{-1})$ appears as the intrinsic dimension of the whitened matrix variance. A cruder coefficient bound $\|\hat\alpha-\tilde\alpha\|_2\le\lambda^{-2}\|K_D-K\|_{\mathrm{op}}\|y\|_2$ also holds, since $K_D=ZZ^\top\succeq0$ makes $K_D+\lambda I$ invertible with inverse norm $\le1/\lambda$. It matches the $1/\sqrt D$ convergence observed downstream (Appendix~\ref{sec:exp_krr}).

\begin{theorem}[Whitened matrix Bernstein: the high-probability KRR condition]\label{thm:krr_whitened}
Let $A=K+\lambda I$ with $\lambda>0$, let $K_D$ be the exact-modulation estimate from $D$ i.i.d.\ radial draws, and set $\rho_D=\|A^{-1/2}(K_D-K)A^{-1/2}\|_{\mathrm{op}}$. With the ridge effective dimension $d_{\mathrm{eff}}(\lambda)=\operatorname{tr}(KA^{-1})$, $\kappa_\lambda=\|K\|_{\mathrm{op}}/(\|K\|_{\mathrm{op}}+\lambda)<1$, and $\tilde d_\lambda=d_{\mathrm{eff}}(\lambda)/\kappa_\lambda$, with probability at least $1-\delta$,
\begin{equation}\label{eq:whitened_tail}
\rho_D\;\le\;2\sqrt{\frac{\kappa_\lambda\,\|P\|_{\mathrm{op}}\,\log(8\tilde d_\lambda/\delta)}{\lambda\,D}}\;+\;\frac{2}{3}\Bigl(1+\frac{2\|P\|_{\mathrm{op}}}{\lambda}\Bigr)\frac{\log(8\tilde d_\lambda/\delta)}{D}.
\end{equation}
In particular, for any target $\rho_0\in(0,1]$,
\begin{equation}\label{eq:whitened_count}
D\;\ge\;\frac{16}{\rho_0^{2}}\Bigl(1+\frac{\|P\|_{\mathrm{op}}}{\lambda}\Bigr)\log\frac{8\tilde d_\lambda}{\delta}
\qquad\Longrightarrow\qquad \mathbb{P}\{\rho_D\le\rho_0\}\ge1-\delta.
\end{equation}
\end{theorem}
\noindent The proof (Appendix~\ref{app:proofs}) whitens the three ingredients of Theorem~\ref{thm:bernstein}: the per-draw positivity $K^{(j)}\succeq0$ survives conjugation, the a.s.\ bound becomes $(1+2\|P\|_{\mathrm{op}}/\lambda)/D$, and the variance majorant becomes $\frac{2\|P\|_{\mathrm{op}}}{\lambda D}A^{-1/2}KA^{-1/2}$, whose intrinsic dimension is \emph{exactly} $\tilde d_\lambda$: the effective dimension is not an assumption imported from the KRR literature but the intrinsic dimension of the whitened variance. This closes the condition Theorem~\ref{thm:krr_spectral} left open. Numerically the identity is exact: across $\lambda\in\{10,1,0.1,0.01\}$ the intrinsic dimension of $A^{-1/2}KA^{-1/2}$ matches $\tilde d_\lambda$ to machine precision (relative error $\le4\times10^{-15}$), the whitened $\rho_D$ decays at the $O(D^{-1/2})$ rate, the objective value stays in the Corollary~\ref{cor:krr_highprob} sandwich on every seed where $\rho_D<1$, and the class instance of Theorem~\ref{thm:class_bernstein}, the polynomially modulated Mat\'ern-$\tfrac12$ kernel with its L\'evy sampler, reproduces all three (Appendix~\ref{app:exp_details}).

\begin{corollary}[High-probability KRR guarantee]\label{cor:krr_highprob}
On the event $\rho_D\le\rho_0<1$ (so, under the draw count~\eqref{eq:whitened_count}, with probability at least $1-\delta$), Theorem~\ref{thm:krr_spectral} applies: $(1-\rho_0)A\preceq K_D+\lambda I\preceq(1+\rho_0)A$ and $\|\tilde\alpha-\hat\alpha\|_A\le\tfrac{\rho_0}{1-\rho_0}\|\hat\alpha\|_A$. Moreover the optimal value of the kernel-ridge objective, $\min_f\sum_i(f(x_i)-y_i)^2+\lambda\|f\|_{\mathcal H}^2=\lambda\,y^\top(K+\lambda I)^{-1}y$, is preserved to relative constants:
\[
\lambda\,y^\top(K_D+\lambda I)^{-1}y\;\in\;\Bigl[\tfrac{1}{1+\rho_0},\ \tfrac{1}{1-\rho_0}\Bigr]\cdot\lambda\,y^\top(K+\lambda I)^{-1}y.
\]
At $\rho_0=\tfrac12$ the coefficient error is at most $\|\hat\alpha\|_A$ and the objective value is preserved within $[\tfrac23,2]$.
\end{corollary}

The count~\eqref{eq:whitened_count} carries the factor $\|P\|_{\mathrm{op}}/\lambda$ because plain $\Exp(\varepsilon)$ sampling pays the worst-case per-draw norm. The radial randomness is cheap to reweight, and the right tilt is computable in closed form: it is the trace of the whitened per-draw Gram, and its mean over the base law is \emph{exactly} the effective dimension. Reweighting by it replaces $\|P\|_{\mathrm{op}}/\lambda$ with $d_{\mathrm{eff}}(\lambda)$ itself, the leverage-sampled count of the random-feature KRR literature \citep{avron2017random,li2021unified}, here obtained on the Gram side with the same matrix-Bernstein proof.

\begin{theorem}[Leverage-weighted radial sampling achieves the effective-dimension count]\label{thm:krr_leverage}
Let $\theta=(t,\omega,\beta)$ denote one radial draw with base law $\pi$ ($t\sim\Exp(\varepsilon)$, $\omega\mid t\sim\mathcal{N}(0,2tI_d)$, $\beta\sim\mathrm{Unif}[0,2\pi]$), per-draw Gram $K^{(\theta)}=(\psi_\theta\psi_\theta^\top)\circ P$ with $\psi_\theta[i]=\sqrt2\cos(\omega^\top x_i+\beta)$, and fix $\lambda>0$, $A=K+\lambda I$. Define the \emph{whitened draw leverage}
\[
\bar d_\lambda(\theta) := \operatorname{tr}\!\bigl(A^{-1}K^{(\theta)}\bigr)=\psi_\theta^\top\bigl(A^{-1}\circ P\bigr)\psi_\theta\;\ge\;0,
\qquad \mathbb{E}_\pi[\bar d_\lambda]=d_{\mathrm{eff}}(\lambda).
\]
Draw $\theta_1,\dots,\theta_D$ i.i.d.\ from the tilted law $d\pi^*_\lambda=(\bar d_\lambda/d_{\mathrm{eff}})\,d\pi$ and set $K^*_D=\frac1D\sum_j \frac{d_{\mathrm{eff}}}{\bar d_\lambda(\theta_j)}K^{(\theta_j)}$. Then $K^*_D$ is unbiased, and with probability at least $1-\delta$, with $\ell=\log(8\tilde d_\lambda/\delta)$,
\begin{equation}\label{eq:leverage_tail}
\rho_D=\bigl\|A^{-1/2}(K^*_D-K)A^{-1/2}\bigr\|_{\mathrm{op}}\;\le\;\sqrt{\frac{2\,d_{\mathrm{eff}}(\lambda)\,\kappa_\lambda\,\ell}{D}}\;+\;\frac{2}{3}\,\frac{(1+d_{\mathrm{eff}}(\lambda))\,\ell}{D}.
\end{equation}
In particular, for any $\rho_0\in(0,1]$,
\begin{equation}\label{eq:leverage_count}
D\;\ge\;\frac{8}{\rho_0^{2}}\bigl(1+d_{\mathrm{eff}}(\lambda)\bigr)\log\frac{8\tilde d_\lambda}{\delta}
\qquad\Longrightarrow\qquad \mathbb{P}\{\rho_D\le\rho_0\}\ge1-\delta,
\end{equation}
replacing the factor $1+\|P\|_{\mathrm{op}}/\lambda$ of Theorem~\ref{thm:krr_whitened} by $1+d_{\mathrm{eff}}(\lambda)$. Since $d_{\mathrm{eff}}(\lambda)\le N$ for every $\lambda$, the leverage count never exceeds $O(N\log(\tilde d_\lambda/\delta))$, whereas the uniform count grows without bound as $\lambda\to0$.
\end{theorem}
\noindent The proof (Appendix~\ref{app:proofs}) whitens the same three ingredients with the weights in place: on the support of $\pi^*_\lambda$ the weighted whitened draw has operator norm at most its trace, which the tilt normalizes to exactly $d_{\mathrm{eff}}$; the variance majorant becomes $\frac{d_{\mathrm{eff}}}{D}A^{-1/2}KA^{-1/2}$, the \emph{same} core matrix as in Theorem~\ref{thm:krr_whitened}, so its intrinsic dimension is again exactly $\tilde d_\lambda$. Two readings. \emph{(i) What is being tilted.} The leverage $\bar d_\lambda(\theta)$ is a quadratic form in $\psi_\theta$ with the fixed PSD matrix $A^{-1}\circ P$: draws whose cosine pattern injects energy where $A^{-1}$ is large (the small-eigenvalue directions of $K$ that the ridge does not protect) are over-sampled and down-weighted. \emph{(ii) Deployment.} Exact tilting needs $A^{-1}\circ P$, the usual leverage chicken-and-egg; the standard answers (a pilot uniform estimate, Nystr\"om or sketched approximations of $A^{-1}$, iterative reweighting) apply verbatim since the proof only uses the importance identity and the trace normalization. Numerically the prediction lands as stated: on the protocol of Theorem~\ref{thm:krr_whitened} ($N{=}300$, $\|P\|_{\mathrm{op}}{=}306$), the uniform draw count $D^*(\rho_D\le\tfrac12)$ climbs $50\to800\to{>}3200$ as $\lambda$ drops $10\to1\to0.1$, while the leverage-tilted sampler reaches it at $D^*=50/200/400/1600$ across $\lambda\in\{10,1,0.1,0.01\}$, tracking $d_{\mathrm{eff}}=12.3/38.6/90.8/163.3$; the identity $\mathbb{E}_\pi[\bar d_\lambda]=d_{\mathrm{eff}}$ holds on a $5\times10^4$-draw pool to within $2\%$, and both samplers keep the $O(D^{-1/2})$ whitened rate (fitted slopes $-0.47$ to $-0.63$; Appendix~\ref{app:exp_details}, \texttt{leverage\_radial\_sampling}).

\begin{remark}[From spectral approximation to risk, and what leverage buys]\label{rmk:risk}
Three reads of Theorems~\ref{thm:krr_whitened} and~\ref{thm:krr_leverage}. \emph{(i) The counts.} $D=O\bigl((1+\|P\|_{\mathrm{op}}/\lambda)\log d_{\mathrm{eff}}\bigr)$ is the plain-i.i.d.-sampling rate, the Gram-side analogue of the $\Omega(1/\lambda)$ feature counts for \emph{uniformly} sampled random Fourier features in KRR \citep{avron2017random,rudi2017generalization,bach2017equivalence}; Theorem~\ref{thm:krr_leverage} improves the leading factor to $d_{\mathrm{eff}}(\lambda)$ itself, the ridge-leverage count of \citet{avron2017random,li2021unified}, by tilting the \emph{joint} radial draw $(t,\omega,\beta)$ rather than the scale alone. What remains open on this axis is purely computational: a deployment-grade approximation of the tilt $\bar d_\lambda$ (pilot estimates, Nystr\"om approximations of $A^{-1}$) with the approximation error folded into the bound. \emph{(ii) Rates.} The condition $\rho\le\tfrac12$ at the regularization schedule $\lambda=\lambda_N$ of the source/capacity conditions is the interface assumption of the random-feature KRR literature, and both counts stay polynomial in $N$ along any such schedule, the leverage count uniformly so ($d_{\mathrm{eff}}\le N$). \emph{(iii) An honest caveat.} A spectral sandwich alone does not bound the fixed-design risk: $K=0$ with $K'=\rho\lambda I$ satisfies the $\rho$-sandwich yet inflates the in-sample variance term from $0$ to $\sigma^2\rho^2/(1+\rho)^2$. A risk theorem therefore needs, in addition, control of the \emph{approximate} kernel's effective dimension $d_{\mathrm{eff}}(K^*_D)$, which a plain sandwich does not supply; with the count side now closed, this is the one missing ingredient between Theorem~\ref{thm:krr_leverage} and a minimax-optimal RAY-KRR statement. The coefficient and objective-value guarantees of Corollary~\ref{cor:krr_highprob} hold for $K^*_D$ as they stand.
\end{remark}

The same conditioning device that carried Corollary~\ref{cor:bernstein_tail} to the deployed estimator (Theorem~\ref{thm:ts_opnorm}) carries Theorem~\ref{thm:krr_whitened} as well.

\begin{corollary}[Ridge-relative KRR condition for the deployed estimator]\label{cor:krr_deployed}
Let $\widehat K_{D,m}$ be the RAY estimate~\eqref{eq:doubly} and fix a target $\rho_0\in(0,1]$. Suppose the sketch is a $(1\pm\eta)$ ridge embedding of $P$ at scale $\lambda$ (Remark~\ref{rmk:ose}) with $\eta\le\rho_0/4$. Then, on the sketch event, with probability at least $1-\delta$ over the radial draws,
\[
D\;\ge\;\frac{64}{\rho_0^{2}}\Bigl(1+\frac{(1+\eta)\|P\|_{\mathrm{op}}}{\lambda}\Bigr)\log\frac{8\tilde d_{\lambda,S}}{\delta}
\qquad\Longrightarrow\qquad
\bigl\|A^{-1/2}(\widehat K_{D,m}-K)A^{-1/2}\bigr\|_{\mathrm{op}}\le\rho_0,
\]
and Corollary~\ref{cor:krr_highprob} holds for the deployed Gram $\widehat K_{D,m}$ verbatim. Here $\tilde d_{\lambda,S}$ is the intrinsic dimension of the sketch-conditioned whitened variance.
\end{corollary}
\begin{proof}
Proposition~\ref{prop:ridge_sketch} gives, on the sketch event, the deterministic bound $\rho_{\mathrm{sk}}:=\|A^{-1/2}(K_S-K)A^{-1/2}\|_{\mathrm{op}}\le\eta$ for the exact-radial Gram $K_S=\widehat P_m\circ R$. Conditioned on the sketch, $\widehat K_{D,m}$ is an exact-modulation estimate of $K_S$, so Theorem~\ref{thm:krr_whitened} applied to the pair $(\widehat P_m,K_S)$ with $\|\widehat P_m\|_{\mathrm{op}}\le(1+\eta)\|P\|_{\mathrm{op}}$ and target $\rho_0/2$, whitened by $A_S=K_S+\lambda I$, gives $\|A_S^{-1/2}(\widehat K_{D,m}-K_S)A_S^{-1/2}\|_{\mathrm{op}}\le\rho_0/2$ under the stated count. Since $A_S\preceq(1+\eta)A$, the $A$-whitened norm is at most $(1+\eta)$ times the $A_S$-whitened one, and the triangle inequality gives $\rho_0\ge(1+\eta)\tfrac{\rho_0}{2}+\eta$ for $\eta\le\rho_0/4$.
\end{proof}
\noindent The sketch requirement is now scale-free, $\eta\le\rho_0/4$, a constant rather than the $\eta\lesssim\lambda/\|P\|_{\mathrm{op}}$ of the absolute-event route, and the size is polynomial in the statistical dimension $s_\lambda(P)$ rather than $\operatorname{rank}(P)$ (Remark~\ref{rmk:ose}). The radial count carries $\log\tilde d_{\lambda,S}$ rather than the worst-case $\log N$: the ridge embedding gives $\tilde d_{\lambda,S}\le(\tilde d_\lambda+\eta N\kappa_\lambda^{-1})/(1-\eta)$, so the count is the genuine effective-dimension count $\log\tilde d_\lambda$ once $\eta\lesssim\tilde d_\lambda/N$, degrading to $\log N$ only as the sketch coarsens. This is the honest content of the transfer: the \emph{sketch accuracy} requirement drops to a constant unconditionally, while the \emph{logarithmic dimension} is intrinsic only when the sketch is fine enough.

None of the matrix-level analysis is special to the $\yat$-kernel. The proofs consume exactly three structural facts (a PSD per-draw Gram that is a Schur product with a bounded-diagonal rank-one factor, a PSD modulation Gram, and a unit-diagonal PSD radial Gram), and every Bernstein--Schur kernel supplies all three.

\begin{theorem}[Class-level matrix concentration and KRR condition]\label{thm:class_bernstein}
Let $k=p\cdot f$ be a Bernstein--Schur kernel in the setting of Theorem~\ref{thm:bernstein_schur} (modulation feature $u$, mixing measure $\nu$ of finite mass $m_f$), let $G_u=[u(x_i)^\top u(x_j)]$ be the modulation Gram, and set $P_u:=m_fG_u$. Then Theorem~\ref{thm:bernstein}, Corollary~\ref{cor:bernstein_tail}, Theorem~\ref{thm:krr_whitened}, Theorem~\ref{thm:krr_leverage}, and Corollary~\ref{cor:krr_highprob} hold verbatim for the class estimator of Theorem~\ref{thm:bernstein_schur} with $P$ replaced by $P_u$. For $k_{\yat,b}$, $P_u=P$.
\end{theorem}
\noindent The proof (Appendix~\ref{app:spectral}) verifies the three facts and nothing else. The statement upgrades the class from ``unbiased with entrywise control'' (Theorem~\ref{thm:bernstein_schur}) to the full operator-norm and kernel-ridge guarantee set; in particular the polynomially modulated Mat\'ern-$\tfrac12$ kernel $(x^\top w+b)^q\,e^{-\|x-w\|/\sigma}$, whose radial scale has the exact two-line sampler of Table~\ref{tab:bernstein_schur}, now carries them all.

Read across $b$, these bounds expose what the bias costs. Setting $b=0$ recovers the unbiased bounds with $R^4$ and $R^8$; for $b>0$ they inflate by powers of $(1+b/R^2)$, the bias enlarging the effective data radius from $R^2$ to $R^2+b$. This is the price side of the trade whose benefit, the finite-difference IMQ reach and the practical bias shift, motivated $k_{\yat,b}$ over $k_{\yat}$ in the first place (Section~\ref{sec:intro}); the $(R^2+b)^4$ variance law is not merely an artifact of the proof but is borne out empirically in Appendix~\ref{sec:exp_bias}.

The $(R^2+b)^4$ blow-up is removable by normalizing the modulation feature to $q_b(x)=p_b(x)/(\|x\|^2+b)$ (unit norm): the resulting \emph{normalized estimator} has variance bounded by $3/(2D\varepsilon^2)$, free of $R$ and $b$ (Proposition~\ref{prop:normalized}, Appendix~\ref{app:proofs}). It is not an approximation to $k_{\yat,b}$ but an unbiased estimator of a \emph{different}, cosine-rescaled kernel, so the gain is stability, not a free fix.

Finally, the Gram error has a structure that explains why the product is worth keeping even when the modulation is itself compressed. Writing $K=P\circ H$ with $P_{ij}=p_b(x_i,x_j)$, $H_{ij}=h_\varepsilon(x_i,x_j)$:

\begin{proposition}[Modulation--radial error decomposition]\label{prop:gate}
With $\widehat P_m$ any modulation approximation and $\widehat H_D$ the radial estimate, $\widehat K_{D,m}=\widehat P_m\circ\widehat H_D$ satisfies
\[
\widehat K_{D,m}-K=\underbrace{P\circ(\widehat H_D-H)}_{\text{radial error, modulated by finite feature}}+\underbrace{(\widehat P_m-P)\circ H}_{\text{modulation error, localized by proximity}}+\underbrace{(\widehat P_m-P)\circ(\widehat H_D-H)}_{\text{interaction}}.
\]
For exact modulation ($\widehat P_m=P$) this collapses to $\widehat K_D-K=P\circ(\widehat H_D-H)$, so entrywise $|\widehat K_{ij}-K_{ij}|=p_b(x_i,x_j)\,|\widehat H_{ij}-H_{ij}|$.
\end{proposition}
\noindent (Proof: expand $(P+E_P)\circ(H+E_H)$ and subtract $K$.) The radial Monte-Carlo noise is Schur-scaled by the modulation; under bounded norms, and literally for the normalized variant where $G_{ij}=(x_i^\top x_j+b)^2/((\|x_i\|^2+b)(\|x_j\|^2+b))\in[0,1]$, this acts as an alignment gate, since $p_b$ otherwise also carries norm and bias scale. Appendix~\ref{sec:exp_gate} shows this modulation suppresses false positives even when the polynomial factor is sketched.

\subsection{Where the dimension enters, and what to compare against}
\label{sec:comparison}

Table~\ref{tab:rf_comparison} states the sample complexity for a $\tau$-approximation of the $N$-point Gram matrix. We are careful about the setting, because it is easy to overclaim. At this \emph{dataset level} (a union bound over the $\le N^2$ pairs, via Hoeffding), the RAY radial count is free of $d$ (Corollary~\ref{cor:sample_complexity}). But so is standard RFF: a union bound over a finite point set removes the explicit $d$ for the Gaussian and IMQ kernels as well. Dimension-freeness of the sample count is therefore a property of the dataset-level analysis, not a special feature of our construction.

\begin{table}[h]
\centering
\caption{Sample complexity for a dataset-level $\tau$-approximation of the $N$-point Gram matrix (Hoeffding + union bound). All counts are free of the input dimension $d$ at this level; the constant is the squared per-draw range. Standard RFF uses an absolutely bounded cosine feature, so its count is free of $R$ and the bandwidth $\sigma$; the IMQ--Laplace estimator carries the radial mass $m_f=1/\varepsilon$, giving $\varepsilon^{-2}$. RAY's count is exactly that radial $\varepsilon^{-2}$ times the \emph{unbounded} polynomial modulation range $(R^2+b)^4$, the honest price of the alignment numerator, removed only by normalizing the modulation (Proposition~\ref{prop:normalized}). The methods differ \emph{qualitatively} in applicability, since RFF needs shift-invariance and sketching a dot-product form, and $k_{\yat,b}$ has neither. $R$ is the data radius.}
\label{tab:rf_comparison}
\begin{tabular}{lll}
\toprule
Kernel (structure) & Method & $D$ for $\tau$-approx \\
\midrule
Gaussian (shift-invariant) & RFF & $O\!\bigl(\tau^{-2}\log N\bigr)$ \\[2pt]
IMQ $h_\varepsilon$ (shift-invariant) & RFF--Laplace & $O\!\bigl(\varepsilon^{-2}\,\tau^{-2}\log N\bigr)$ \\[2pt]
$k_{\yat,b}$ (neither shift-inv.\ nor dot-product) & RAY (ours) & $O\!\bigl(\tfrac{(R^2+b)^4}{\varepsilon^2}\,\tau^{-2}\log N\bigr)$ \\
\bottomrule
\end{tabular}
\end{table}

\begin{table}[h]
\centering
\caption{Approximation taxonomy. Random-feature methods draw feature parameters from kernel-defined distributions before seeing the dataset and are unbiased for arbitrary pairs. Nystr\"om methods are data-dependent deployment baselines: they choose columns, centers, or leverage-sampled landmarks from the observed training set and approximate the empirical Gram matrix or downstream predictor.}
\label{tab:method_taxonomy}
{\small
\begin{tabular}{llll}
\toprule
Method & Data-dependent? & Unbiased for $k(x,w)$? & Streaming primal? \\
\midrule
Exact modulation (analyzable limit) & No & Yes & Yes \\
Gaussian / IMQ RFF & No & Yes & Yes \\
RAY (sketched modulation) & No (sketch-random) & Yes, over all draws & Yes \\
Uniform Nystr\"om & Yes (columns from $X$) & No & No \\
k-means Nystr\"om & Yes (clustered $X$) & No & No \\
Ridge-leverage Nystr\"om & Yes (scores from $X$) & No & No \\
\bottomrule
\end{tabular}
}
\end{table}

What RAY genuinely buys is therefore not a smaller $d$-dependence than RFF, but \emph{applicability}: it brings a kernel that is neither shift-invariant nor a dot-product, on which RFF and polynomial sketching cannot be run, into the same dataset-level, $d$-free regime that RFF enjoys for the kernels it does cover, at the cost of a larger constant: $(R^2+b)^4/\varepsilon^2$ in place of the absolute constant of a bounded-cosine RFF, the unbounded polynomial modulation being the one factor RFF does not carry. The explicit $d$ familiar from RFF reappears only in the \emph{domain-level} (uniform-over-a-$d$-ball) analysis through the covering number; there RAY's inner Gaussian features incur it as well, so RAY is not $d$-free in that stronger sense either. Nystr\"om occupies a different category (Table~\ref{tab:method_taxonomy}): it is an adaptive landmark approximation to the observed sample, not a data-independent random-feature map. Its error can degrade with dimension through the spectral-decay rate $O(m^{-2s/d})$, which is the contrast our matched-landmark experiments target (Section~\ref{sec:exp_dimfree}, Table~\ref{tab:gram_error}), but its adaptivity can also make it a stronger deployment baseline at fixed representation size (Appendix~\ref{sec:exp_faircost}).

\subsection{Variance reduction}
\label{sec:variance_reduction}

Three optional refinements of the sampling lower the variance further; all are evaluated in Appendix~\ref{app:further}.

\paragraph{Quasi-Monte Carlo over $t$.} Replace $t_j\sim\Exp(\varepsilon)$ by quasi-random samples through the inverse CDF $t_j=-\varepsilon^{-1}\log u_j$, $u_j$ from a Sobol or Halton sequence. The radial integrand $t\mapsto e^{-t(\varepsilon+\|x-w\|^2)}$ is $C^\infty$ and of bounded variation, so by Koksma--Hlawka \citep{kuo2016application} the \emph{$t$-integral's} QMC error is $O((\log D)^s/D)$. We emphasize, however, that this is the rate of the outer integral only: the estimator also carries the inner Gaussian RFF Monte-Carlo noise, which is not quasi-randomized, and empirically (Appendix~\ref{sec:exp_var}) it is this inner noise that sets the overall convergence, so QMC delivers a constant-factor variance reduction rather than a faster rate. Deterministic Sobol/Halton points also trade the exact i.i.d.\ unbiasedness of Theorem~\ref{thm:unbiased} for a bounded deterministic integration error; \emph{randomized} QMC (a scrambled or shifted sequence) restores unbiasedness in expectation over the randomization while keeping the improved error, and is the variant we recommend when unbiasedness must hold exactly.

\paragraph{Orthogonal features for the inner approximation.} For each scale $t_j$, drawing the inner frequencies as scaled rows of a Haar-random orthogonal matrix \citep{yu2016orthogonal} reduces the inner-loop variance without affecting unbiasedness. The combined QMC-outer / orthogonal-inner estimator achieves the best of both. At the recommended $D'=1$ there is no inner block to decorrelate, and the right target is instead the \emph{marginal} spectral measure of the radial factor (Section~\ref{sec:step-rff}): drawing the $D$ frequencies $\omega_j$ jointly orthogonal or quasi-random against that single isotropic distribution, with radii coupled to the inverse radial CDF, decorrelates the draws the current i.i.d.\ estimator leaves independent. This is free variance reduction and the natural form of the refinement when $D'=1$.

\paragraph{Importance sampling over $t$.} Sampling $t\sim\Exp(\varepsilon+\eta)$ with weight $w_t=\frac{\varepsilon}{\varepsilon+\eta}e^{\eta t}$ keeps the estimator unbiased and reduces variance for nearby pairs ($\|x-w\|^2\le\eta$) by up to $(\varepsilon/(\varepsilon+\eta))^2$, \emph{provided $\eta$ is not too large}: the importance estimate of the radial factor at squared distance $r$ has a finite second moment only for $\eta<\varepsilon+2r$ (Appendix~\ref{app:further}), so a single proposal safe for every pair, including $r=0$, requires $\eta<\varepsilon$. The proposal is data-independent and so practical within that range, if suboptimal for the full pairwise-distance spread.

\paragraph{Deterministic radial quadrature.} The outer scale integral can be discretized deterministically rather than by Monte Carlo, replacing the across-scale variance of Proposition~\ref{prop:budget} by a uniform, controllable bias. This answers the kernel-quadrature question of Section~\ref{sec:discussion} for the radial factor in the affirmative.

\begin{proposition}[Positive-weight radial quadrature]\label{prop:quadrature}
The radial factor is the Laplace transform $h_\varepsilon(r)=\int_0^\infty e^{-\varepsilon t}e^{-tr}\,dt=1/(\varepsilon+r)$. For every tolerance $\tau>0$ there exist, by exponential-sum approximation of $1/x$ on a bounded interval \citep{beylkin2010approximation}, $D=O\!\bigl(\log(1+4R^2/\varepsilon)\,\log(1/\tau)\bigr)$ positive nodes $t_1,\dots,t_D>0$ and positive weights $w_1,\dots,w_D>0$ with
\begin{equation}\label{eq:quad_bias}
\sup_{r\in[0,4R^2]}\Bigl|\textstyle\sum_{j=1}^D w_j\,e^{-t_j r}-h_\varepsilon(r)\Bigr|\le\tau.
\end{equation}
Replacing the random scales $t_j\sim\Exp(\varepsilon)$ by these nodes and the average by $\sum_j w_j(\cdot)$ leaves a surrogate radial Gram that is a nonnegative combination of Gaussian kernels, so the surrogate $\widehat k=\bigl(\sum_j w_j e^{-t_j\|\cdot-\cdot\|^2}\bigr)\cdot p$ is positive definite, and the entrywise error is the deterministic bias $\sup_{i,j}|\widehat k_{ij}-k_{ij}|\le\tau(R^2+b)^2$, free of the $O(D^{-1/2})$ radial fluctuation of Theorem~\ref{thm:uniform}.
\end{proposition}
\begin{proof}
Equation~\eqref{eq:quad_bias} is the cited exponential-sum bound applied to $1/x$ on $[\varepsilon,\varepsilon+4R^2]$ after the shift $x=\varepsilon+r$; positive weights keep each $e^{-t_j\|x-w\|^2}$ a Gaussian kernel, so the nonnegative combination and its Schur product with the PSD modulation $p$ are PSD. Multiplying~\eqref{eq:quad_bias} by the modulation, bounded by $(R^2+b)^2$, gives the entrywise bias.
\end{proof}
\noindent Each fixed-scale Gaussian still needs an inner feature, so the deployed dimension is unchanged; what the nodes remove is the outer scale sampling. The empirical scoping is sharp (Appendix~\ref{app:exp_details}, \texttt{radial\_quadrature}): with the Gaussian factor computed \emph{exactly} (the Gram model), a $D$-node Gauss--Laguerre rule reaches machine precision by $D=32$ where Monte-Carlo is still at $5.4\times10^{-2}$ RMSE; but in the deployed estimator the inner Fourier noise dominates (Appendix~\ref{sec:exp_var}), and a naive equal-allocation node rule is in fact \emph{worse} than i.i.d.\ sampling at matched draws ($0.38$ vs.\ $0.17$ at $D=32$), because the nodes spend budget on scales the random law visits in proportion to their weight. The nodes are therefore the right tool for exact-Gaussian surrogates and small-Gram settings, and as the bounded scale sets that positive features require (Proposition~\ref{prop:pos_dichotomy}, through the largest node), not a drop-in replacement for the trigonometric sampler. The count grows only \emph{logarithmically} with the radial dynamic range $4R^2/\varepsilon$; the price is the loss of exact unbiasedness, traded for the uniform bias~\eqref{eq:quad_bias}, and the unbiased sampler remains the variant for the streaming and analysis claims.

\section{Experiments}
\label{sec:experiments}

The experiments answer four questions in order. First, does RAY approximate the off-sphere kernel in the regime where neither shift-invariant RFF nor dot-product sketches apply? Second, does TensorSketch remove the $O(d^2)$ modulation floor while preserving the predicted radial-plus-sketch error split? Third, is the alignment$\times$proximity kernel useful exactly when the target couples those two factors? Finally, does the resulting feature map run in the streaming, Gram-impossible regime it is designed for? Cost-matched caveats, gating diagnostics, attention fidelity, sphere-normalized sanity checks, and other ablations are deferred to Appendix~\ref{app:further}. This keeps the main section focused on the empirical spine; the appendix records the negative and scoped results needed to interpret it honestly.

\subsection{Off-sphere validation on a bounded ball}
\label{sec:exp_offsphere}

The key test is off-sphere (Figure~\ref{fig:offsphere}): a bounded ball where $k_{\yat,b}$ is not a function of $x^\top w$ alone, so no dot-product reduction exists. We sample $N=1000$ points with directions uniform on $\mathbb{S}^{d-1}$ and radii uniform on $[0.25,1]$ (a dimension-independent spread of norms; uniform-in-ball sampling would collapse toward the sphere as $d$ grows). Approximating the exact Gram with RAY (flat $D'=1$, $b=1$, $\varepsilon$ the median squared distance, $5$ seeds) reproduces the $O(1/\sqrt D)$ rate, while uniform- and k-means-landmark Nystr\"om degrade with $d$ (Table~\ref{tab:offsphere}): k-means helps by a constant but not the trend, and RAY at $D=1000$ is below both by $d=8$. This matches radial draws against landmarks, not feature dimension (RAY's is $D\,d_b$); Appendix~\ref{sec:exp_faircost} gives the cost-matched comparison. Bounded norm, not sphere normalization, is what the estimator requires.

\begin{table}[h]
\centering
\caption{Relative Frobenius Gram error on an \emph{off-sphere} bounded ball ($\|x\|\in[0.25,1]$, varying norms; $N=1000$, $b=1$, mean over $5$ seeds). RAY follows the $O(1/\sqrt D)$ rate at every dimension; both uniform and k-means Nystr\"om ($m=100$) worsen as $d$ grows, k-means by a smaller constant. Standard deviation over the $5$ seeds is $\le0.015$ for the RAY columns and $\le0.003$ for the Nystr\"om columns, so the $d\ge8$ ordering (RAY below both Nystr\"om variants) is outside seed noise.}
\label{tab:offsphere}
\begin{tabular}{lccccc}
\toprule
& \multicolumn{3}{c}{RAY, radial samples $D$} & uniform Nystr\"om & k-means Nystr\"om \\
\cmidrule(lr){2-4}
$d$ & $10$ & $100$ & $1000$ & $m=100$ & $m=100$ \\
\midrule
$2$  & $0.36$ & $0.09$ & $0.040$ & $0.001$ & $\mathbf{0.000}$ \\
$8$  & $0.55$ & $0.14$ & $0.044$ & $0.073$ & $0.051$ \\
$16$ & $0.48$ & $0.17$ & $0.051$ & $0.092$ & $0.077$ \\
$32$ & $0.53$ & $0.17$ & $0.057$ & $0.103$ & $0.093$ \\
\bottomrule
\end{tabular}
\end{table}

\noindent The adaptive \emph{ridge-leverage-score} Nystr\"om \citep{musco2017recursive}, the strongest standard baseline, does not change the picture: it tracks uniform Nystr\"om ($0.072$/$0.094$/$0.110$ at $d=8$/$16$/$32$), k-means stays the best landmark variant ($0.051$/$0.079$/$0.097$), and all three worsen with $d$ while RAY at $D=1000$ remains below them by $d=16$. The behavior persists into genuinely high dimension: on the same off-sphere ball at $d\in\{128,256,512\}$, RAY follows the $O(1/\sqrt D)$ rate (fitted slopes $-0.41$, $-0.51$, $-0.52$) with relative error $0.05$--$0.06$ at $D=1000$, unchanged in order across a $4\times$ range of $d$.

The fidelity carries to real data kept off-sphere. Standardizing and scaling each dataset by its maximum row norm (so $\|x\|\le1$ with \emph{varying} norms, not on the sphere), RAY at $D=128$ tracks the exact $\yat$-kernel's downstream accuracy on \texttt{digits} ($0.983$ vs.\ exact $0.986$) and on a larger \texttt{covtype} subsample ($N{=}3000$, $d{=}54$: $0.695$ vs.\ exact $0.682$), in both cases matching or beating Gaussian RFF and k-means Nystr\"om at the same budget. So the sphere-normalized ridge-regression results (Appendix~\ref{app:further}) are not an artifact of normalization: RAY reproduces the exact kernel off-sphere as well.

\subsection{Compressing the polynomial factor with TensorSketch}
\label{sec:exp_ts}

RAY sketches the degree-2 modulation to escape the $O(d^2)$ feature floor (Section~\ref{sec:step-sketch}); here we validate the error decomposition and exercise the accuracy--cost trade. The variance splits into a radial Monte-Carlo term ($\sim D^{-1/2}$), a polynomial-sketch term ($\sim m^{-1/2}$), and their product (Proposition~\ref{prop:ts_variance}). For a fixed pair ($4000$ repetitions) the split is clean: varying $D$ at $m=256$ drives only the radial term ($0.09\to8.5\times10^{-4}$), varying $m$ at $D=1000$ only the sketch term ($4.2\times10^{-3}\to1.4\times10^{-4}$), and the empirical variance matches the three-term formula throughout (ratio $0.94$--$1.06$). The operator-norm error of Theorem~\ref{thm:ts_opnorm} shows the same structure (Figure~\ref{fig:ts_opnorm}): at $m{=}128$ the radial term falls as $D^{-1/2}$ ($40.5\to3.6$ over $D{=}10\to1000$) while the sketch term is a $D$-independent floor that decays with $m$ ($19.5\to7.4$ as $m{=}64\to256$), vanishing as $m\to\infty$.

\begin{figure}[t]
\centering
\includegraphics[width=\textwidth]{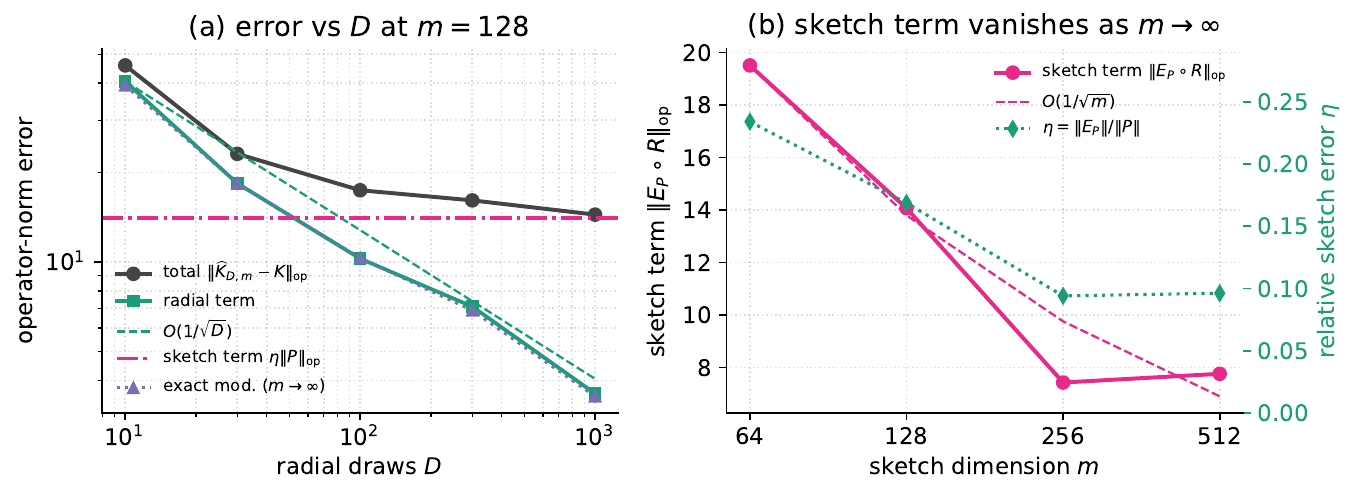}
\caption{Operator-norm error of the deployed (doubly-randomized) RAY estimator, validating Theorem~\ref{thm:ts_opnorm} (off-sphere, $d=16$, $N=300$, $\|P\|_{\mathrm{op}}=186$). \textbf{(a)} At fixed sketch size $m{=}128$, the radial term falls as $O(1/\sqrt D)$ while the sketch term $\eta\|P\|_{\mathrm{op}}$ is a $D$-independent floor; the total decays to that floor, and the $m\to\infty$ (exact-modulation) curve is the zero-floor limit. \textbf{(b)} The sketch term $\|E_P\circ R\|_{\mathrm{op}}$ and the relative sketch error $\eta=\|E_P\|_{\mathrm{op}}/\|P\|_{\mathrm{op}}$ both shrink with $m$, so increasing $m$ recovers exact modulation (Corollary~\ref{cor:bernstein_tail}). The single additive sketch term is the entire price of randomizing the modulation.}
\label{fig:ts_opnorm}
\end{figure}

At finite $m$ the compression recovers most of the efficiency the exact feature loses to its $O(d^2)$ size: on sphere-normalized digits ($d=64$) at $M=d_b$, RAY ($m=128$) reaches $0.977$ where exact modulation, starved to a single radial draw, scores $0.928$, near the optimal rank-$M$ ceiling $0.986$ (sanity check, Table~\ref{tab:ts}, Appendix~\ref{app:further}). In low dimension there is nothing to compress (california, $d_b=45$): the sketch only adds noise and one takes $m\to\infty$. RAY is the default scalable map; exact modulation is its low-$d$ endpoint, not a separate method.

At a fixed feature dimension $M=D(m{+}d{+}1)$ the sketch size $m$ trades against the radial-draw count $D$, and the optimum is interior (Table~\ref{tab:dm}), at the closed-form $m^\star=\sqrt{BM/A}$ derived from the two-term variance (Proposition~\ref{prop:optimal_m}). The best allocation for Gram fidelity is an intermediate sketch size; downstream KRR error instead favors smaller $m$ (more radial draws), so the operating point depends on whether Gram fidelity or prediction is the goal. This operationalizes the radial-vs-sketch split of Theorem~\ref{thm:ts_opnorm}.

\begin{table}[h]
\centering
\caption{Deployed-RAY allocation at fixed feature dimension $M=8192$ on an off-sphere ball (relative Frobenius Gram error, mean over $3$ seeds): at fixed $M=D(m{+}d{+}1)$ the sketch size $m$ trades against the radial-draw count $D$. The error is minimized at an intermediate $m$ (small $m$ leaves a large sketch term, large $m$ starves the radial draws), and the optimal $m$ grows slowly with $d$.}
\label{tab:dm}
\begin{tabular}{lccc}
\toprule
$m$ & $d{=}16$ & $d{=}64$ & $d{=}256$ \\
\midrule
$16$  & $0.212$ & $0.252$ & $0.424$ \\
$32$  & $\mathbf{0.173}$ & $\mathbf{0.216}$ & $0.397$ \\
$64$  & $0.173$ & $0.216$ & $0.305$ \\
$128$ & $0.185$ & $0.228$ & $0.376$ \\
$256$ & $0.240$ & $0.281$ & $0.343$ \\
$512$ & $0.520$ & $0.488$ & $0.423$ \\
\bottomrule
\end{tabular}
\end{table}

Sketching also removes the $O(d^2)$ representation floor outright (Table~\ref{tab:runtime_d}). At $d{=}1024$ the exact-modulation feature needs $33.7$\,GB per the $N{=}1000$ representation and cannot be built, while the sketched feature is $74$\,MB and builds in $28$\,ms, a direct demonstration that sketching removes the $O(d^2)$ floor (Limitation~(v)).

\begin{table}[h]
\centering
\caption{Feature-construction cost vs.\ dimension $d$ at fixed radial draws $D{=}8$, sketch $m{=}128$, $N{=}1000$ (off-sphere ball). Exact modulation's per-point feature dimension $D\,d_b$ grows as $O(d^2)$ and becomes impossible to build; the sketched feature $D(m{+}d{+}1)$ stays linear in $d$. ``--'' marks a feature too large to materialize ($>2$M coordinates/point).}
\label{tab:runtime_d}
\resizebox{\textwidth}{!}{%
\begin{tabular}{lcccccc}
\toprule
$d$ & exact dim & exact mem (GB) & exact build (s) & sketch dim & sketch mem (MB) & sketch build (s) \\
\midrule
$64$   & $17160$   & $0.14$  & $0.043$ & $1544$ & $12.4$ & $0.005$ \\
$128$  & $67080$   & $0.54$  & $0.143$ & $2056$ & $16.4$ & $0.007$ \\
$256$  & $265224$  & $2.12$  & $1.585$ & $3080$ & $24.6$ & $0.011$ \\
$512$  & $1054728$ & $8.44$  & $8.202$ & $5128$ & $41.0$ & $0.017$ \\
$1024$ & $4206600$ & $33.65$ & --      & $9224$ & $73.8$ & $0.028$ \\
\bottomrule
\end{tabular}}
\end{table}

\subsection{Coupled-target preference}
\label{sec:exp_necessity}

The comparisons so far measure approximation quality, not whether the $\yat$-kernel is the right kernel. We test that inductive bias directly: the kernel couples alignment and proximity, so it should be preferred exactly when the target benefits from both, and not otherwise. On off-sphere data ($\|x\|\in[0.3,1.5]$, $d=16$) we build three regression targets from small atom sums and compare kernel ridge regression with the Gaussian, IMQ, degree-2 polynomial, and $\yat$ kernels plus RAY (Table~\ref{tab:necessity}). The two single-factor controls are kernel-natural (\emph{proximity} $y=\sum_k a_k/(\|x-v_k\|^2+\varepsilon_0)$ and \emph{alignment} $y=\sum_k a_k(u_k^\top x)^2$), so they should, and do, favor the distance kernels and the polynomial kernel respectively. The \emph{coupled} target $y=\sum_k a_k\tanh(2u_k^\top x)\,e^{-\|x-v_k\|}$ is deliberately \emph{not} of the $\yat$-kernel's form: it multiplies a $\tanh$ alignment by a Laplace proximity, matching no candidate kernel, so a $\yat$ win there cannot be an artifact of $\yat$-generated data. The $\yat$-kernel wins the coupled target and only there. The evidence is therefore a controlled compatibility result for the alignment$\times$proximity coupling, not a general necessity theorem. RAY inherits the preference once its approximation has converged: at $D=4000$ it reaches $0.093$ on the coupled target, below both distance kernels, approaching the exact $\yat$-kernel's $0.087$; at smaller $D$ the residual approximation error keeps it above them, consistent with the $O(1/\sqrt D)$ convergence measured elsewhere.

\begin{table}[h]
\centering
\caption{Kernel ridge regression test RMSE ($\downarrow$) on off-sphere targets ($d=16$, varying norms, mean over $5$ seeds, RAY at $D=4000$, best per row in bold). The coupled target (tanh alignment $\times$ Laplace proximity) matches no candidate kernel; the $\yat$-kernel is best only there, while distance kernels win the proximity control and the polynomial wins the alignment control. RAY tracks the exact $\yat$-kernel, beating both distance kernels on the coupled target.}
\label{tab:necessity}
\begin{tabular}{lccccc}
\toprule
Target & Gaussian & IMQ & polynomial & $\yat$ (exact) & RAY \\
\midrule
coupled (needs both) & $0.098$ & $0.099$ & $0.350$ & $\mathbf{0.087}$ & $0.093$ \\
proximity only       & $0.341$ & $\mathbf{0.141}$ & $0.612$ & $0.237$ & $0.278$ \\
alignment only       & $0.349$ & $0.363$ & $\mathbf{0.001}$ & $0.268$ & $0.268$ \\
\bottomrule
\end{tabular}
\end{table}

The preference survives fair per-kernel tuning and has the expected budget caveat: exact $\yat$ keeps the lead on the coupled target after hyperparameter tuning, but a small fixed feature budget can under-resolve the $O(d^2)$ modulation and let cheaper single-factor RFFs win. Those controls, together with the gating diagnostic that explains the effect, are in Appendix~\ref{sec:exp_controls}. The main point is the scoped one Table~\ref{tab:necessity} isolates: the kernel is useful when the target actually couples alignment and proximity, not as a universal replacement for simpler kernels.

\subsection{Large-scale primal training on HIGGS}
\label{sec:exp_higgs}

RAY is a primal feature map, never forming a Gram. We close by exercising that property where it is forced: HIGGS \citep{baldi2014searching}, $11{,}000{,}000$ examples with $d=28$ physics features, a scale at which the $N\times N$ Gram ($\sim10^{14}$ entries) cannot be stored. Compressed Bernstein--Schur features train as a memory-flat streaming primal (peak $8.5$\,GB, no Gram) at this scale, with compression matching or beating exact modulation at every budget and decisively at the largest, where the exact polynomial starves the radial factor to a handful of draws whose seed spread is itself large; Gaussian RFF leads on AUC, as expected on a smooth detector task with no alignment$\times$proximity coupling to exploit. The point is streaming feasibility where the Gram is impossible, not peak AUC; the full sweep, protocol, and table are in Appendix~\ref{app:further} (Table~\ref{tab:higgs}).

\section{Discussion}
\label{sec:discussion}

The biased $\yat$-kernel fits neither random-feature template (it is neither shift-invariant nor a dot-product kernel (Proposition~\ref{prop:nonstationary})) yet RAY linearizes it with operator-norm guarantees, and for the entire Bernstein--Schur \emph{class} rather than a single kernel (Theorem~\ref{thm:bernstein_schur}). The Schur factorization is what makes this work: it turns an intractable product into two textbook objects, an exact finite modulation and one completely monotone radial factor, each randomized by its native tool. The matrix-Bernstein analysis (Theorem~\ref{thm:bernstein}) then controls the error through the top eigenvalues and an intrinsic dimension, not the crude $N\max_{ij}$ route, and the deployed doubly-randomized estimator inherits that guarantee up to a single tunable sketch term (Theorem~\ref{thm:ts_opnorm}). Whitening that argument at the ridge sharpens it into a deployment condition: the effective dimension $d_{\mathrm{eff}}(\lambda)=\tr(KA^{-1})$ is the \emph{exact} intrinsic dimension of the whitened matrix variance, so $O\bigl((1+\|P\|_{\mathrm{op}}/\lambda)\log(d_{\mathrm{eff}}/\delta)\bigr)$ radial draws preserve the kernel-ridge coefficients and the objective value to constants (Theorem~\ref{thm:krr_whitened}, Corollary~\ref{cor:krr_highprob}), and tilting the radial draw by the whitened leverage $\bar d_\lambda(\theta)=\psi_\theta^\top(A^{-1}\circ P)\psi_\theta$, whose mean is exactly $d_{\mathrm{eff}}(\lambda)$, improves the count to $O\bigl((1+d_{\mathrm{eff}})\log(d_{\mathrm{eff}}/\delta)\bigr)$ (Theorem~\ref{thm:krr_leverage}). This entire matrix-level guarantee set, operator norm through ridge condition, holds verbatim for every Bernstein--Schur kernel with the modulation Gram in place of the polynomial one (Theorem~\ref{thm:class_bernstein}). The tensor-product closure these pieces rely on is classical \citep{aronszajn1950theory}; the contribution is the factorization that brings it to bear on a kernel outside both templates, and the class-level concentration analysis of the result.

The modulation feature is interchangeable: Theorem~\ref{thm:ts_opnorm} sees it only as PSD and an $\eta$-spectral approximation of $P$, so any unbiased or low-rank such feature (TensorSketch \citep{avron2014subspace}, random-Maclaurin products \citep{kar2012random}, or anchor squares) plugs into the same bound. On the sphere, where the modulation is a dot-product kernel, anchor squares are exact with few anchors, the route of the spherical $\yat$-attention of \citet{luna2026slay}; off the sphere they are biased ($\mathbb{E}_a(a^\top x)^2(a^\top w)^2=\|x\|^2\|w\|^2+2(x^\top w)^2\ne(x^\top w+b)^2$), so we default to an unbiased sketch. The choice is geometry, not analysis.

Computationally, via the Kronecker structure
\begin{equation}\label{eq:kronecker}
z(x)^\top z(w)=\frac{(x^\top w+b)^2}{\varepsilon D}\sum_{j=1}^D\psi_{t_j}(x)^\top\psi_{t_j}(w),
\end{equation}
a kernel evaluation costs $O(D(D'+d))$ without forming the $M_b$-dimensional feature. In primal form the estimator never materializes the $N\times N$ Gram, paying $O(N\,D\,d_b)$ feature memory; this beats the $O(N^2)$ Gram whenever $D\,d_b\ll N$. The polynomial dimension $d_b=O(d^2)$ is the only term scaling poorly in $d$. The deployed estimator sketches it to $m\ll d^2$ (Section~\ref{sec:exp_ts}), making the feature dimension $Dm$ at the price of one tunable sketch term (Theorem~\ref{thm:ts_opnorm}).

Several caveats bound the method. The variance constant $(R^2+b)^4$ exceeds the $R^4$ of Gaussian RFF (the price of the polynomial factor and the bias), and it is sharp, not a loose bound: it factors out of the variance of \emph{any} exact-modulation estimator with equality (Proposition~\ref{prop:variance_sharp}), so the realized Gram error grows weakly with $d$ through that constant rather than the rate (Appendix~\ref{sec:exp_gram}). The estimator needs bounded-norm inputs, since the unbounded numerator destabilizes the kernel on raw data, though the off-sphere ball (Section~\ref{sec:exp_offsphere}) shows exact spherical support is not required; the data-independent importance proposal is likewise suboptimal when distances spread widely. The $O(d^2)$ floor of the exact-modulation feature is removed by sketching (Theorem~\ref{thm:ts_opnorm}), worthwhile except at very low $d$ or very high target accuracy, where one takes $m\to\infty$. And at moderate $N$ and fixed representation size, adaptive Nystr\"om can be more accurate, fitting landmarks to the observed sample; RAY does not aim to replace landmark methods, but to provide data-independent, unbiased, streaming features.

Several questions remain open. \emph{(1) Risk.} The leverage count of Theorem~\ref{thm:krr_leverage} closes the draw-count side of the ridge question; a minimax excess-risk theorem now needs exactly two more pieces: a deployment-grade approximation of the tilt $\bar d_\lambda$ with its error folded into the bound, and control of the approximate kernel's effective dimension $d_{\mathrm{eff}}(K^*_D)$ (Remark~\ref{rmk:risk}), the Rudi--Rosasco-style feature-covariance refinement. A matching lower bound, that plain $\Exp(\varepsilon)$ sampling \emph{requires} $\Omega(\|P\|_{\mathrm{op}}/\lambda)$ draws on worst-case data (the empirical gap in Theorem~\ref{thm:krr_leverage}'s validation points that way), would certify that the tilt is necessary and not merely sufficient. \emph{(2) Is the prefactor fundamental?} Proposition~\ref{prop:variance_sharp} pins the $(R^2+b)^4$ variance prefactor with equality \emph{within} the factored family (exact modulation times any unbiased radial estimate); whether any unbiased, data-independent feature map of $k_{\yat,b}$ can beat it, for instance by correlating the modulation and radial randomness so their errors cancel, is open. Sketching only adds variance (Proposition~\ref{prop:ts_variance}) and normalizing changes the kernel (Proposition~\ref{prop:normalized}), so a positive answer would have to leave the factored family without leaving unbiasedness; we conjecture the prefactor is optimal. \emph{(3) The peaked regime.} Small $\varepsilon$ with aligned data defeats both estimators we have: the trigonometric error grows with attention sharpness (Appendix~\ref{sec:exp_attention}) and positive features have infinite variance precisely there (Proposition~\ref{prop:pos_dichotomy}). A finite-variance estimator for this regime, e.g.\ a FAVOR-type mechanism with data-dependent normalization, is open. \emph{(4) Geometry.} Can the variance bound be tightened on low-dimensional manifolds, where the data occupies a thin shell and the intrinsic dimension of Theorem~\ref{thm:bernstein} is empirically tiny (Appendix~\ref{sec:exp_gram})?

\section{Conclusion}
\label{sec:conclusion}

We gave a random-feature scheme for \emph{Bernstein--Schur kernels} (finite-feature modulations of completely monotone radial kernels) and its flagship, the biased $\yat$-kernel $k_{\yat,b}(w,x)=(w^\top x+b)^2/(\|w-x\|^2+\varepsilon)$, a kernel that is neither shift-invariant nor a dot-product kernel and so fits neither standard random-feature template. The scheme reads the kernel as a Schur product and randomizes \emph{both} factors: it samples the radial Bernstein--Widder scale $t\sim\Exp(\varepsilon)$ and applies random Fourier features to the resulting Gaussian, and it sketches the modulation, so the deployed feature has dimension $Dm$, free of the $O(d^2)$ size of the exact polynomial feature. Keeping the modulation exact is the analyzable limit ($m\to\infty$): there the estimator is unbiased with $O((R^2+b)^4/(D\varepsilon^2))$ variance and a radial sample count free of explicit dimension for a fixed dataset, and a matrix-Bernstein operator-norm bound holds; conditioning on the sketch carries that guarantee to the doubly-randomized estimator up to a single tunable sketch term (Theorem~\ref{thm:ts_opnorm}). Whitening the matrix-Bernstein argument at the ridge upgrades the operator-norm bound to a high-probability kernel-ridge guarantee, in which the effective dimension $d_{\mathrm{eff}}(\lambda)$ is the exact intrinsic dimension of the matrix variance and $O\bigl((1+\|P\|_{\mathrm{op}}/\lambda)\log(d_{\mathrm{eff}}/\delta)\bigr)$ radial draws preserve the ridge solution and objective value (Theorem~\ref{thm:krr_whitened}); tilting the radial draw by a closed-form whitened leverage function, whose mean over the base law is exactly $d_{\mathrm{eff}}(\lambda)$, improves the count to $O\bigl((1+d_{\mathrm{eff}})\log(d_{\mathrm{eff}}/\delta)\bigr)$ (Theorem~\ref{thm:krr_leverage}). Experiments confirm the $O(1/\sqrt D)$ rate, the $(R^2+b)^4$ bias law, and the matched-landmark off-sphere behavior, where $k_{\yat,b}$ is genuinely non-dot-product and no dot-product reduction is available. Cost-matched, adaptive data-dependent Nystr\"om can be more accurate at moderate $N$, while RAY provides data-independent, unbiased, streaming primal features for a nonstationary product kernel. The same doubly-randomized construction applies unchanged to the whole Bernstein--Schur class, and so does the full matrix-level guarantee set, operator norm through ridge condition, with the modulation Gram in place of the polynomial one (Theorem~\ref{thm:class_bernstein}); the modulation randomizer (sketch, random-Maclaurin, or anchor features) is a geometry-dependent choice, with anchor features the natural primitive on the sphere and an unbiased sketch off it.

\paragraph{Reproducibility.} The code for every experiment, with the paper's figures and a project page, is available online.\footnote{\url{https://www.tahabouhsine.com/ray}} The per-table datasets, splits, seed counts, and swept grids are collected in Appendix~\ref{app:exp_details}, and each result table states its own seed count and error bars in its caption.

\bibliographystyle{tmlr}
\bibliography{references}

\appendix

\section{Proofs of the Approximation Guarantees}
\label{app:proofs}

Both results rest on a single per-scale term $Y=(\psi_t(x)^\top\psi_t(w))(p(x)^\top p(w))$, whose polynomial factor is exact, $p(x)^\top p(w)=(x^\top w+b)^2/\varepsilon$, and whose Gaussian factor is an unbiased $D'$-sample estimate of $g_t(x,w)$. Two facts about it are used repeatedly: the a.s.\ bound $|\psi_t(x)^\top\psi_t(w)|\le\|\psi_t(x)\|\|\psi_t(w)\|\le2$ (each $\|\psi_t\|^2=\tfrac2{D'}\sum_\ell\cos^2\le2$), so $|Y|\le 2(R^2+b)^2/\varepsilon$; and the second moment $\mathbb{E}[(\psi_t(x)^\top\psi_t(w))^2]\le\mathbb{E}_T[e^{-2T\|x-w\|^2}]+O(1/D')$. Corollary~\ref{thm:gram_concentration} needs no separate proof: it is $\|A\|_{\mathrm{op}}\le\|A\|_F\le N\max_{ij}|A_{ij}|$ applied to Theorem~\ref{thm:uniform}.

\begin{theorem}[Variance envelope]\label{thm:variance}
Let $V_D=\Var[z(x)^\top z(w)]$. Accounting for both the inner ($D'$) and outer ($D$) sampling,
\begin{equation}\label{eq:variance_bound}
V_D \le \frac{(\|x\|^2+b)^2(\|w\|^2+b)^2}{D\,\varepsilon^2}\!\left(\frac{\varepsilon}{2\|x-w\|^2+\varepsilon}+\frac{3}{2D'}\right) \le \frac{(R^2+b)^4}{D\,\varepsilon^2}\!\left(1+\frac{3}{2D'}\right).
\end{equation}
The $3/(2D')$ term is the inner random-Fourier-feature variance; it does not vanish at the recommended $D'=1$ (where the constant is $5/2$), only as $D'\to\infty$.
\end{theorem}
\begin{proof}[Proof of Theorem~\ref{thm:variance}]
Write $z(x)^\top z(w)=\frac1D\sum_j Y_j$; the $Y_j$ are i.i.d., so $V_D=\frac1D\Var[Y_1]\le\frac1D\mathbb{E}[Y_1^2]=\frac1D(p(x)^\top p(w))^2\,\mathbb{E}[(\psi_T(x)^\top\psi_T(w))^2]$, the polynomial factor being deterministic. It obeys $(p(x)^\top p(w))^2=(x^\top w+b)^4/\varepsilon^2\le(\|x\|^2+b)^2(\|w\|^2+b)^2/\varepsilon^2$ by Cauchy--Schwarz and $b\ge0$. For the Gaussian factor, the law of total variance over $\omega\mid T$ then $T$ gives $\mathbb{E}[(\psi_T(x)^\top\psi_T(w))^2]=\mathbb{E}_T[g_T(x,w)^2]+\mathbb{E}_T[v_T]/D'$, where $v_t=\Var_\omega[2\cos(\omega^\top x+\beta)\cos(\omega^\top w+\beta)]$. The outer term is $\mathbb{E}_{T\sim\Exp(\varepsilon)}[e^{-2T\|x-w\|^2}]=\varepsilon/(2\|x-w\|^2+\varepsilon)$. For the inner term, $\mathbb{E}_{\omega,\beta}[(2\cos(\omega^\top x+\beta)\cos(\omega^\top w+\beta))^2]=1+\tfrac12 e^{-4t\|x-w\|^2}\le\tfrac32$, so $v_t\le\tfrac32$ and $\mathbb{E}_T[v_T]\le\tfrac32$. Substituting and dividing by $D$ gives~\eqref{eq:variance_bound}; the second form uses $\varepsilon/(2\|x-w\|^2+\varepsilon)\le1$. (The $1/\varepsilon^2$ comes from the polynomial normalization $p=\varepsilon^{-1/2}p_b$.)
\end{proof}

\begin{proposition}[Sharpness of the variance prefactor]\label{prop:variance_sharp}
For any unbiased estimator of $k_{\yat,b}(x,w)$ of the factored form $\widehat k=\bigl(p(x)^\top p(w)\bigr)\,\widehat h$, with the modulation inner product kept exact (deterministic) and $\widehat h$ any unbiased estimator of $h_\varepsilon(x,w)$,
\[
\Var[\widehat k]=\frac{(x^\top w+b)^4}{\varepsilon^2}\,\Var[\widehat h].
\]
The prefactor $(x^\top w+b)^4/\varepsilon^2\le(R^2+b)^4/\varepsilon^2$ is therefore common to the entire factored family and is attained with \emph{equality}, not a slack upper bound: Theorem~\ref{thm:exact_variance} is the instance with $\widehat h$ the flat radial estimator, and no sharper analysis can remove the prefactor while the modulation is kept exact. Lowering it requires leaving the family, by normalizing the modulation (Proposition~\ref{prop:normalized}, a different, cosine-rescaled kernel) or randomizing it (Proposition~\ref{prop:ts_variance}, which only adds variance).
\end{proposition}
\begin{proof}
The modulation factor $p(x)^\top p(w)=(x^\top w+b)^2/\varepsilon$ is deterministic, so it leaves the variance as its square; the inequality is Cauchy--Schwarz $(x^\top w+b)^2\le(\|x\|^2+b)(\|w\|^2+b)\le(R^2+b)^2$ with $b\ge0$.
\end{proof}

\begin{proposition}[Explicit $D,D'$ variance and optimal allocation]\label{prop:budget}
Let $v_t(x,w)=\Var_\omega[\,2\cos(\omega^\top x+\beta)\cos(\omega^\top w+\beta)\,]$ be the variance of one trigonometric feature pair at scale $t$, and set $V_{\mathrm{out}}=\Var_{T\sim\Exp(\varepsilon)}[g_T(x,w)]$, $V_{\mathrm{in}}=\mathbb{E}_{T\sim\Exp(\varepsilon)}[v_T(x,w)]$. Then
\begin{equation}\label{eq:budget_var}
\Var[z(x)^\top z(w)]=\frac{(x^\top w+b)^4}{\varepsilon^2}\Bigl(\frac{V_{\mathrm{out}}}{D}+\frac{V_{\mathrm{in}}}{D\,D'}\Bigr).
\end{equation}
At a fixed feature budget $B=DD'$ the right side is minimized by $D'=1$ (hence $D=B$) whenever $V_{\mathrm{out}}>0$: the inner term $V_{\mathrm{in}}/B$ is fixed by the budget, while the outer term $V_{\mathrm{out}}/D$ decreases as $1/D$. If $V_{\mathrm{out}}=0$ (for example at $x=w$), every allocation with the same $DD'$ ties.
\end{proposition}
\begin{proof}[Proof of Proposition~\ref{prop:budget}]
The polynomial factor is deterministic, so $\Var[Y_1]=(p(x)^\top p(w))^2\Var[\psi_T(x)^\top\psi_T(w)]$ with $(p(x)^\top p(w))^2=(x^\top w+b)^4/\varepsilon^2$. By the law of total variance, conditioning on the scale $T$ and then on the $D'$ inner frequencies,
\[
\Var[\psi_T(x)^\top\psi_T(w)]=\mathbb{E}_T\bigl[\Var_\omega(\psi_T(x)^\top\psi_T(w)\mid T)\bigr]+\Var_T\bigl(\mathbb{E}_\omega[\psi_T(x)^\top\psi_T(w)\mid T]\bigr).
\]
The conditional mean is $g_T(x,w)$, giving the outer term $\Var_T(g_T)=V_{\mathrm{out}}$. The conditional object is an average of $D'$ i.i.d.\ feature pairs, so its variance is $v_T/D'$, giving the inner term $\mathbb{E}_T[v_T]/D'=V_{\mathrm{in}}/D'$. Dividing by $D$ (the $D$ blocks are i.i.d.) yields~\eqref{eq:budget_var}. With $DD'=B$ fixed, $V_{\mathrm{in}}/(DD')=V_{\mathrm{in}}/B$ is constant and $V_{\mathrm{out}}/D$ is decreasing in $D$, so $D'=1$ is optimal.
\end{proof}

\begin{theorem}[Uniform Gram approximation]\label{thm:uniform}
For $X=\{x_1,\dots,x_N\}$ with $\|x_i\|\le R$, the exact-modulation estimator with $D$ outer samples satisfies, with probability at least $1-\delta$,
\begin{equation}\label{eq:uniform_bound}
\sup_{i,j\in[N]}\bigl|z(x_i)^\top z(x_j)-k_{\yat,b}(x_i,x_j)\bigr| \le \frac{(R^2+b)^2}{\varepsilon}\sqrt{\frac{8\log(2N^2/\delta)}{D}}.
\end{equation}
The constant uses the a.s.\ bound $|\psi_t(x)^\top\psi_t(w)|\le2$, not the expectation-level value $1$.
\end{theorem}
\begin{corollary}[Sample complexity]\label{cor:sample_complexity}
$D=O\!\bigl((R^2+b)^4\varepsilon^{-2}\tau^{-2}\log(N/\eta)\bigr)$ outer samples suffice for entrywise error at most $\tau$ with probability $1-\eta$, independent of dimension $d$.
\end{corollary}
\begin{corollary}[Operator-norm control]\label{thm:gram_concentration}
Since $\|A\|_{\mathrm{op}}\le\|A\|_F\le N\max_{ij}|A_{ij}|$, Theorem~\ref{thm:uniform} gives $\|K_D-K\|_{\mathrm{op}}\le N(R^2+b)^2\varepsilon^{-1}\sqrt{8\log(2N^2/\delta)/D}$ with probability $1-\delta$. The factor $N$ is the price of this elementary route; Theorem~\ref{thm:bernstein} removes it via matrix concentration, replacing $N\max_{ij}|\cdot|$ with the top eigenvalues of the kernel and modulation Gram matrices and an intrinsic dimension (which can still scale with $N$ in the worst case).
\end{corollary}
\begin{proof}[Proof of Theorem~\ref{thm:uniform}]
Fix $(x_i,x_j)$; the estimator averages $D$ i.i.d.\ terms $Y_k\in[-c,c]$ with $c=2(R^2+b)^2/\varepsilon$ by the a.s.\ bound above. Hoeffding's inequality gives $\mathbb{P}(|z(x_i)^\top z(x_j)-k_{\yat,b}|>s)\le2\exp(-Ds^2/(2c^2))=2\exp\!\bigl(-Ds^2\varepsilon^2/(8(R^2+b)^4)\bigr)$. A union bound over the $\le N^2$ pairs and solving $2N^2\exp(\cdot)\le\delta$ for $s$ yields~\eqref{eq:uniform_bound}.
\end{proof}

\begin{lemma}[Schur multiplier bound]\label{lem:schur}
Let $A\succeq0$ with $\max_i A_{ii}\le a$. \textbf{(a)} If $P\succeq0$, then $0\preceq A\circ P\preceq a\|P\|_{\mathrm{op}}I$. \textbf{(b)} If $E=E^\top$ is symmetric (not necessarily PSD), then $\|A\circ E\|_{\mathrm{op}}\le a\|E\|_{\mathrm{op}}$.
\end{lemma}
\begin{proof}
Write $A=UU^\top=\sum_r u_r u_r^\top$ with $u_r$ the columns of $U$; then $A\circ E=\sum_r(u_r u_r^\top)\circ E=\sum_r D_r E D_r$ with $D_r=\operatorname{diag}(u_r)$, and the diagonal matrix $\sum_r D_r^2$ has entries $\sum_r(u_r)_i^2=(UU^\top)_{ii}=A_{ii}\le a$, so $\sum_r D_r^2\preceq aI$. For any unit vector $x$ and symmetric $E$,
\[
\bigl|x^\top(A\circ E)x\bigr|=\Bigl|\sum_r(D_r x)^\top E(D_r x)\Bigr|\le\|E\|_{\mathrm{op}}\sum_r\|D_r x\|^2=\|E\|_{\mathrm{op}}\,x^\top\Bigl(\textstyle\sum_r D_r^2\Bigr)x\le a\|E\|_{\mathrm{op}},
\]
which is~(b). Taking $E=P\succeq0$ removes the absolute value (each summand is nonnegative), giving $0\le x^\top(A\circ P)x\le a\|P\|_{\mathrm{op}}$ and $A\circ P\succeq0$ by the Schur product theorem, which is~(a).
\end{proof}

\begin{proof}[Proof of Theorem~\ref{thm:bernstein}]
Write $K_D-K=\sum_{j=1}^D Y_j$ with $Y_j=D^{-1}(K^{(j)}-K)$ and $K^{(j)}=(\Psi_j\Psi_j^\top)\circ P$ the Gram of the $j$-th radial draw; the $Y_j$ are i.i.d.\ symmetric with $\mathbb{E}[Y_j]=0$ (Theorem~\ref{thm:unbiased}). Two structural facts: $K^{(j)}\succeq0$, being a Schur product of the PSD matrices $\Psi_j\Psi_j^\top$ and $P$; and $\|K^{(j)}\|_{\mathrm{op}}\le2\|P\|_{\mathrm{op}}$ by Lemma~\ref{lem:schur} with $(\Psi_j\Psi_j^\top)_{ii}=\|\psi_{t_j}(x_i)\|^2\le2$. Likewise $\|K\|_{\mathrm{op}}=\|R\circ P\|_{\mathrm{op}}\le\|P\|_{\mathrm{op}}$ by Lemma~\ref{lem:schur}, with $R$ the radial Gram (PSD, unit diagonal).

\emph{A.s.\ bound.} $\|Y_j\|_{\mathrm{op}}\le D^{-1}(\|K^{(j)}\|_{\mathrm{op}}+\|K\|_{\mathrm{op}})\le 3\|P\|_{\mathrm{op}}/D=:L$.

\emph{Matrix variance.} For PSD $M$, $M^2\preceq\|M\|_{\mathrm{op}}M$; with $\|K^{(j)}\|_{\mathrm{op}}\le2\|P\|_{\mathrm{op}}$ this gives $(K^{(j)})^2\preceq2\|P\|_{\mathrm{op}}K^{(j)}$, so $\mathbb{E}[(K^{(j)})^2]\preceq2\|P\|_{\mathrm{op}}K$ and $\mathbb{E}[(K^{(j)}-K)^2]=\mathbb{E}[(K^{(j)})^2]-K^2\preceq2\|P\|_{\mathrm{op}}K$. Hence
\[
V=\sum_{j=1}^D\mathbb{E}[Y_j^2]=\tfrac1D\,\mathbb{E}[(K^{(1)}-K)^2]\preceq\tfrac{2\|P\|_{\mathrm{op}}}{D}\,K,\qquad \|V\|_{\mathrm{op}}\le\tfrac{2\|P\|_{\mathrm{op}}\|K\|_{\mathrm{op}}}{D}=:v.
\]

\emph{Conclusion (expectation).} We use the expectation form obtained by integrating the two-sided intrinsic tail below, rather than substituting $d_{\mathrm{int}}$ into the ambient-dimension expectation theorem. Let $\ell_0=\log(8d_{\mathrm{int}})$ and $s(u)=\sqrt{2vu}+\tfrac23Lu$. The tail bound gives $\mathbb{P}\{\|\sum_jY_j\|_{\mathrm{op}}>s(\ell_0+r)\}\le e^{-r}$ for $r\ge0$; integrating this quantile bound and using $\ell_0\ge1$ gives
\[
\mathbb{E}\|\textstyle\sum_jY_j\|_{\mathrm{op}}\le s(\ell_0)+\int_0^\infty e^{-r}s'(\ell_0+r)\,dr
\le 3\sqrt{v\ell_0}+2L\ell_0.
\]
Substituting $v=2\|P\|_{\mathrm{op}}\|K\|_{\mathrm{op}}/D$ and $L=3\|P\|_{\mathrm{op}}/D$ yields the stated bound; $d_{\mathrm{int}}\le\operatorname{rank}(V)\le N$ and $\|P\|_{\mathrm{op}}\le\operatorname{tr}(P)\le N(R^2+b)^2/\varepsilon$ give the worst-case comparison to Corollary~\ref{thm:gram_concentration}.

\emph{Conclusion (tail).} With the \emph{same} $v$ and $L$, the tail form of the intrinsic matrix Bernstein inequality \citep[Thm.~7.3.1]{tropp2015matrix} bounds the top eigenvalue. Applying it to both $\sum_jY_j$ and $-\sum_jY_j$ and union bounding gives $\mathbb{P}\{\|\sum_jY_j\|_{\mathrm{op}}\ge s\}\le 8d_{\mathrm{int}}\exp\!\bigl(-\tfrac{s^2/2}{v+Ls/3}\bigr)$ for $s\ge\sqrt v+L/3$. Setting the right side to $\delta$ and writing $\ell=\log(8d_{\mathrm{int}}/\delta)$, the quadratic $s^2-\tfrac23 L\ell\,s-2v\ell=0$ has positive root $s=\tfrac13 L\ell+\sqrt{(\tfrac13 L\ell)^2+2v\ell}\le\sqrt{2v\ell}+\tfrac23 L\ell$. Because $\ell\ge\log 8>1$, this root obeys the range restriction: $\sqrt{2v\ell}\ge\sqrt v$ and $\tfrac23L\ell\ge L/3$. Hence, with probability at least $1-\delta$, $\|\sum_jY_j\|_{\mathrm{op}}\le\sqrt{2v\log(8d_{\mathrm{int}}/\delta)}+\tfrac23 L\log(8d_{\mathrm{int}}/\delta)$, which is~\eqref{eq:bernstein_tail} after substituting $v$ and $L$ (the linear coefficient is $\tfrac23 L=2\|P\|_{\mathrm{op}}/D$).
\end{proof}

\begin{proof}[Proof of Theorem~\ref{thm:exact_variance}]
Write $Y=2\cos(\omega^\top x+\beta)\cos(\omega^\top w+\beta)=\cos(\omega^\top(x-w))+\cos(\omega^\top(x+w)+2\beta)$. The second term has zero mean over $\beta\sim\mathrm{Unif}[0,2\pi]$, so with $\omega\mid T\sim\mathcal{N}(0,2TI_d)$ and the Gaussian characteristic function $\mathbb{E}_\omega[\cos(\omega^\top u)]=e^{-T\|u\|^2}$, $\mathbb{E}[Y\mid T]=e^{-Tr}$ and $\mathbb{E}[Y]=\mathbb{E}_{T\sim\Exp(\varepsilon)}[e^{-Tr}]=\varepsilon/(\varepsilon+r)$. For the second moment, $\mathbb{E}_\beta[Y^2\mid\omega]=\cos^2(\omega^\top(x-w))+\tfrac12$, and $\mathbb{E}_\omega[\cos^2(\omega^\top(x-w))\mid T]=\tfrac12(1+e^{-4Tr})$ (using $\cos^2\theta=\tfrac12(1+\cos2\theta)$ and $\mathbb{E}[\cos(2\omega^\top(x-w))\mid T]=e^{-4Tr}$). Hence $\mathbb{E}[Y^2\mid T]=1+\tfrac12 e^{-4Tr}$ and $\mathbb{E}[Y^2]=1+\tfrac12\varepsilon/(\varepsilon+4r)$. Thus $\Var(Y)=1+\tfrac12\varepsilon/(\varepsilon+4r)-(\varepsilon/(\varepsilon+r))^2$, and $\widehat{k}_D$ averages $D$ i.i.d.\ copies scaled by $a/\varepsilon$, giving $\Var[\widehat{k}_D]=(a/\varepsilon)^2\Var(Y)/D$.
\end{proof}

\begin{proof}[Proof of Theorem~\ref{thm:krr_spectral}]
Let $B=A^{-1/2}EA^{-1/2}$, so $\|B\|_{\mathrm{op}}\le\rho$ gives $-\rho I\preceq B\preceq\rho I$, and multiplying by $A^{1/2}$ on both sides gives $-\rho A\preceq E\preceq\rho A$; since $K_D+\lambda I=A+E$ this is the spectral sandwich. For the coefficients, $\tilde\alpha-\hat\alpha=(A+E)^{-1}y-A^{-1}y=-(A+E)^{-1}E\hat\alpha$, so $A^{1/2}(\tilde\alpha-\hat\alpha)=-(I+B)^{-1}B\,A^{1/2}\hat\alpha$. As $\|B\|_{\mathrm{op}}\le\rho<1$, $\|(I+B)^{-1}\|_{\mathrm{op}}\le(1-\rho)^{-1}$, hence $\|\tilde\alpha-\hat\alpha\|_A=\|A^{1/2}(\tilde\alpha-\hat\alpha)\|_2\le\frac{\rho}{1-\rho}\|A^{1/2}\hat\alpha\|_2=\frac{\rho}{1-\rho}\|\hat\alpha\|_A$.
\end{proof}

\begin{proof}[Proof of Theorem~\ref{thm:krr_whitened}]
Write $A^{-1/2}(K_D-K)A^{-1/2}=\sum_{j=1}^D\widetilde Y_j$ with $\widetilde Y_j=D^{-1}A^{-1/2}(K^{(j)}-K)A^{-1/2}$, i.i.d.\ symmetric with mean zero. The three ingredients of the proof of Theorem~\ref{thm:bernstein} whiten as follows.

\emph{(a) A.s.\ bound.} $K^{(j)}\succeq0$ with $\|K^{(j)}\|_{\mathrm{op}}\le2\|P\|_{\mathrm{op}}$ gives $0\preceq A^{-1/2}K^{(j)}A^{-1/2}\preceq2\|P\|_{\mathrm{op}}A^{-1}\preceq(2\|P\|_{\mathrm{op}}/\lambda)I$, and $K\preceq A$ gives $0\preceq A^{-1/2}KA^{-1/2}\preceq I$; hence $\|\widetilde Y_j\|_{\mathrm{op}}\le L_\lambda:=(1+2\|P\|_{\mathrm{op}}/\lambda)/D$.

\emph{(b) Variance majorant.} For PSD $M$, $MA^{-1}M=M^{1/2}(M^{1/2}A^{-1}M^{1/2})M^{1/2}\preceq\lambda^{-1}\|M\|_{\mathrm{op}}M$; with $M=K^{(j)}$ and $\|K^{(j)}\|_{\mathrm{op}}\le2\|P\|_{\mathrm{op}}$, $\mathbb{E}[K^{(j)}A^{-1}K^{(j)}]\preceq(2\|P\|_{\mathrm{op}}/\lambda)\,K$. Since $\mathbb{E}[(K^{(j)}-K)A^{-1}(K^{(j)}-K)]=\mathbb{E}[K^{(j)}A^{-1}K^{(j)}]-KA^{-1}K\preceq\mathbb{E}[K^{(j)}A^{-1}K^{(j)}]$,
\[
\sum_{j}\mathbb{E}[\widetilde Y_j^2]=\tfrac1D\,A^{-1/2}\,\mathbb{E}\bigl[(K^{(1)}-K)A^{-1}(K^{(1)}-K)\bigr]A^{-1/2}\;\preceq\;\widetilde V:=\frac{2\|P\|_{\mathrm{op}}}{\lambda D}\,A^{-1/2}KA^{-1/2}.
\]

\emph{(c) Intrinsic dimension.} $\|\widetilde V\|_{\mathrm{op}}=\frac{2\|P\|_{\mathrm{op}}}{\lambda D}\kappa_\lambda$ and $\operatorname{tr}(\widetilde V)=\frac{2\|P\|_{\mathrm{op}}}{\lambda D}\,d_{\mathrm{eff}}(\lambda)$, so $\operatorname{intdim}(\widetilde V)=d_{\mathrm{eff}}(\lambda)/\kappa_\lambda=\tilde d_\lambda$.

The intrinsic-dimension matrix Bernstein inequality \citep[Thm.~7.3.1]{tropp2015matrix} accepts the semidefinite majorant $\widetilde V$; applied to $\pm\sum_j\widetilde Y_j$ (the union over the two tails giving the prefactor $8$) and inverted by the quadratic-root computation in the proof of Corollary~\ref{cor:bernstein_tail}, it yields, with probability at least $1-\delta$ and $\ell=\log(8\tilde d_\lambda/\delta)$, $\rho_D\le\sqrt{2\|\widetilde V\|_{\mathrm{op}}\ell}+\tfrac23L_\lambda\ell$, which is~\eqref{eq:whitened_tail}. For the count: $2\sqrt{\kappa_\lambda\|P\|_{\mathrm{op}}\ell/(\lambda D)}\le\rho_0/2$ iff $D\ge16\kappa_\lambda\|P\|_{\mathrm{op}}\ell/(\lambda\rho_0^2)$, and $\tfrac23(1+2\|P\|_{\mathrm{op}}/\lambda)\ell/D\le\rho_0/2$ iff $D\ge\tfrac43(1+2\|P\|_{\mathrm{op}}/\lambda)\ell/\rho_0$; both follow from $D\ge16\rho_0^{-2}(1+\|P\|_{\mathrm{op}}/\lambda)\ell$ using $\kappa_\lambda\le1$, $1+2s\le2(1+s)$, and $\rho_0\le1$, giving~\eqref{eq:whitened_count}.
\end{proof}

\begin{proof}[Proof of Theorem~\ref{thm:krr_leverage}]
\emph{The leverage identity.} With $D_\psi=\diag(\psi_\theta)$, $K^{(\theta)}=D_\psi PD_\psi$, so $\operatorname{tr}(A^{-1}K^{(\theta)})=\operatorname{tr}(D_\psi A^{-1}D_\psi P)=\psi_\theta^\top(A^{-1}\circ P)\psi_\theta$. The matrix $A^{-1}\circ P$ is a Schur product of PSD matrices, hence PSD, so $\bar d_\lambda\ge0$; and $\mathbb{E}_\pi[\psi_\theta\psi_\theta^\top]=R$ (the unit-diagonal radial Gram) gives $\mathbb{E}_\pi[\bar d_\lambda]=\operatorname{tr}(A^{-1}(R\circ P))=\operatorname{tr}(A^{-1}K)=d_{\mathrm{eff}}(\lambda)$, so $\pi^*_\lambda$ is a probability law. On $\{\bar d_\lambda=0\}$ the PSD matrix $A^{-1/2}K^{(\theta)}A^{-1/2}$ has zero trace, hence is zero, so $K^{(\theta)}=0$ ($A^{-1/2}$ is invertible) and discarding these draws loses nothing. \emph{Unbiasedness:} $\mathbb{E}_{\pi^*_\lambda}\bigl[\tfrac{d_{\mathrm{eff}}}{\bar d_\lambda}K^{(\theta)}\bigr]=\int K^{(\theta)}\mathbbm{1}\{\bar d_\lambda>0\}\,d\pi=K$.

\emph{Whitened ingredients.} Set $M_\theta=\tfrac{d_{\mathrm{eff}}}{\bar d_\lambda(\theta)}A^{-1/2}K^{(\theta)}A^{-1/2}\succeq0$ and $\widetilde Y_j=D^{-1}\bigl(M_{\theta_j}-A^{-1/2}KA^{-1/2}\bigr)$, i.i.d.\ with mean zero under $\pi^*_\lambda$. \emph{(a) A.s.\ bound.} On the support of $\pi^*_\lambda$, $\|M_\theta\|_{\mathrm{op}}\le\operatorname{tr}(M_\theta)=\tfrac{d_{\mathrm{eff}}}{\bar d_\lambda}\operatorname{tr}(A^{-1}K^{(\theta)})=d_{\mathrm{eff}}$ (a PSD matrix's norm is at most its trace, and the tilt normalizes the trace exactly); with $\|A^{-1/2}KA^{-1/2}\|_{\mathrm{op}}\le1$, $\|\widetilde Y_j\|_{\mathrm{op}}\le L^*:=(1+d_{\mathrm{eff}})/D$. \emph{(b) Variance majorant.} $M_\theta^2\preceq\|M_\theta\|_{\mathrm{op}}M_\theta\preceq d_{\mathrm{eff}}M_\theta$ and $\mathbb{E}_{\pi^*_\lambda}[M_\theta]=A^{-1/2}KA^{-1/2}$, so
\[
\sum_{j=1}^D\mathbb{E}[\widetilde Y_j^2]\preceq\tfrac1D\,\mathbb{E}_{\pi^*_\lambda}[M_\theta^2]\preceq\widetilde V^*:=\frac{d_{\mathrm{eff}}}{D}\,A^{-1/2}KA^{-1/2}.
\]
\emph{(c) Intrinsic dimension.} $\widetilde V^*$ is the same core matrix $A^{-1/2}KA^{-1/2}$ as in the proof of Theorem~\ref{thm:krr_whitened}, rescaled, so $\operatorname{intdim}(\widetilde V^*)=\tilde d_\lambda$ exactly and $\|\widetilde V^*\|_{\mathrm{op}}=d_{\mathrm{eff}}\kappa_\lambda/D$.

The intrinsic-dimension matrix Bernstein inequality applied to $\pm\sum_j\widetilde Y_j$ and inverted as in Corollary~\ref{cor:bernstein_tail} gives, with probability at least $1-\delta$ and $\ell=\log(8\tilde d_\lambda/\delta)$, $\rho_D\le\sqrt{2\|\widetilde V^*\|_{\mathrm{op}}\ell}+\tfrac23L^*\ell$, which is~\eqref{eq:leverage_tail}. For the count: $\sqrt{2d_{\mathrm{eff}}\kappa_\lambda\ell/D}\le\rho_0/2$ iff $D\ge8d_{\mathrm{eff}}\kappa_\lambda\ell/\rho_0^2$, and $\tfrac23(1+d_{\mathrm{eff}})\ell/D\le\rho_0/2$ iff $D\ge\tfrac43(1+d_{\mathrm{eff}})\ell/\rho_0$; both follow from $D\ge8\rho_0^{-2}(1+d_{\mathrm{eff}})\ell$ using $\kappa_\lambda\le1$ and $\rho_0\le1$, giving~\eqref{eq:leverage_count}.
\end{proof}

\begin{proof}[Proof of Corollary~\ref{cor:krr_highprob}]
The sandwich and the coefficient bound are Theorem~\ref{thm:krr_spectral} on the event $\rho_D\le\rho_0$. For the objective value, the representer computation gives the optimal value $\lambda y^\top(K+\lambda I)^{-1}y$ for the exact problem (substitute $\hat\alpha=(K+\lambda I)^{-1}y$: $\|K\hat\alpha-y\|^2+\lambda\hat\alpha^\top K\hat\alpha=\lambda^2y^\top A^{-2}y+\lambda y^\top A^{-1}KA^{-1}y=\lambda y^\top A^{-1}y$) and likewise $\lambda y^\top(K_D+\lambda I)^{-1}y$ for the feature-space ridge problem with Gram $K_D$. Operator anti-monotonicity of the inverse on $(1-\rho_0)A\preceq K_D+\lambda I\preceq(1+\rho_0)A$ gives $(1+\rho_0)^{-1}A^{-1}\preceq(K_D+\lambda I)^{-1}\preceq(1-\rho_0)^{-1}A^{-1}$; evaluating the quadratic form at $y$ finishes.
\end{proof}

\begin{proof}[Proof of Corollary~\ref{cor:krr_deployed}]
Proposition~\ref{prop:ridge_sketch} gives, on the sketch event, the deterministic bound $\rho_{\mathrm{sk}}:=\|A^{-1/2}(K_S-K)A^{-1/2}\|_{\mathrm{op}}\le\eta$ for $K_S=\widehat P_m\circ R$. Conditioned on the sketch, $\widehat K_{D,m}$ is an exact-modulation estimate of $K_S$ with modulation Gram $\widehat P_m$ and $\|\widehat P_m\|_{\mathrm{op}}\le(1+\eta)\|P\|_{\mathrm{op}}$. Applying Theorem~\ref{thm:krr_whitened} to the conditioned pair $(\widehat P_m,K_S)$, whitened by $A_S=K_S+\lambda I$, with target $\rho_0/2$ and intrinsic dimension $\tilde d_{\lambda,S}$ gives
\[
\bigl\|A_S^{-1/2}(\widehat K_{D,m}-K_S)A_S^{-1/2}\bigr\|_{\mathrm{op}}\le\rho_0/2
\]
under the count in the statement. The ridge transfer also gives $A_S\preceq(1+\eta)A$, so the same radial error measured in the $A$-whitened norm is at most $(1+\eta)\rho_0/2$. Therefore
\[
\bigl\|A^{-1/2}(\widehat K_{D,m}-K)A^{-1/2}\bigr\|_{\mathrm{op}}
\le (1+\eta)\rho_0/2+\eta \le \rho_0,
\]
where the last inequality follows from $\eta\le\rho_0/4$ and $\rho_0\le1$. The final claim is Corollary~\ref{cor:krr_highprob} applied on the combined event.
\end{proof}

\begin{proposition}[Normalized estimator]\label{prop:normalized}
Let $q_b(x)=p_b(x)/(\|x\|^2+b)$, so $\|q_b(x)\|=1$ (for $b>0$ this is defined on all of $\mathbb{R}^d$; for $b=0$ we restrict to domains excluding the origin). Then $q_b(x)^\top q_b(w)=(x^\top w+b)^2/\bigl((\|x\|^2+b)(\|w\|^2+b)\bigr)$, and the normalized kernel $\bar k_{\yat,b}(x,w)=q_b(x)^\top q_b(w)\,(\|x-w\|^2+\varepsilon)^{-1}$ is positive definite. Its flat estimator (replace $p_b$ by $q_b$) is unbiased with, by Theorem~\ref{thm:exact_variance} and $\|q_b\|=1$, $\Var[\widehat{\bar k}_D(x,w)]\le\frac{1}{D\varepsilon^2}(1+\tfrac12\tfrac{\varepsilon}{\varepsilon+4\|x-w\|^2})\le\frac{3}{2D\varepsilon^2}$, a bound free of $R$ and $b$.
\end{proposition}
\begin{proof}[Proof of Proposition~\ref{prop:normalized}]
$\|p_b(x)\|^2=(\|x\|^2+b)^2$ (Proposition~\ref{prop:biased_feature}), so $q_b=p_b/(\|x\|^2+b)$ has unit norm and $q_b(x)^\top q_b(w)=p_b(x)^\top p_b(w)/((\|x\|^2+b)(\|w\|^2+b))=(x^\top w+b)^2/((\|x\|^2+b)(\|w\|^2+b))$. The normalized kernel is a Schur product of this PSD kernel and $h_\varepsilon$, hence PSD; the variance bound is Theorem~\ref{thm:exact_variance} with $a=q_b(x)^\top q_b(w)\le1$.
\end{proof}

\begin{proposition}[RAY error decomposition]\label{prop:ts_variance}
Sketch only the quadratic term: $\widehat p_m(x,w)=\mathrm{TS}_2(x)^\top\mathrm{TS}_2(w)+2b\,x^\top w+b^2$ with $\mathbb{E}[\mathrm{TS}_2(x)^\top\mathrm{TS}_2(w)]=(x^\top w)^2$, so $\mathbb{E}[\widehat p_m]=p=(x^\top w+b)^2$ (the linear and constant terms are kept exact and add no sketch variance). Let $\widehat p_m$ be independent of the unbiased radial estimate $\widehat h_D$ of $h(x,w)=(\|x-w\|^2+\varepsilon)^{-1}$. Then $\widehat k_{D,m}=\widehat p_m\widehat h_D$ is unbiased for $k_{\yat,b}$ and
\[
\Var[\widehat k_{D,m}]=\underbrace{p^2\Var[\widehat h_D]}_{\text{radial Monte Carlo}}+\underbrace{h^2\Var[\widehat p_m]}_{\text{polynomial sketch}}+\underbrace{\Var[\widehat p_m]\,\Var[\widehat h_D]}_{\text{interaction}}.
\]
The degree-2 TensorSketch bound $\Var[\widehat p_m]\le C\,\|x\|^4\|w\|^4/m$ (the bias terms being exact) makes the sketch terms vanish as $m\to\infty$, recovering the exact-modulation estimator.
\end{proposition}
\begin{proof}[Proof of Proposition~\ref{prop:ts_variance}]
By independence $\mathbb{E}[\widehat p_m\widehat h_D]=ph=k$ and $\mathbb{E}[\widehat p_m^2\widehat h_D^2]=\mathbb{E}[\widehat p_m^2]\mathbb{E}[\widehat h_D^2]=(p^2+\Var\widehat p_m)(h^2+\Var\widehat h_D)$; subtracting $p^2h^2$ gives the three-term identity.
\end{proof}

\begin{proposition}[Optimal sketch size]\label{prop:optimal_m}
At a fixed feature budget $M=D(m+d+1)$, the deployed variance of Proposition~\ref{prop:ts_variance} is, to leading order in the two sampling counts,
\[
V(m)=\underbrace{\frac{A\,(m+d+1)}{M}}_{\text{radial},\ A=p^2V_{\mathrm{rad}}}+\underbrace{\frac{B}{m}}_{\text{sketch},\ B=h^2C\|x\|^4\|w\|^4}+O\!\Bigl(\frac1M\Bigr),
\]
with $A$ the radial Monte-Carlo constant ($V_{\mathrm{rad}}=\Var[\widehat h_1]$) and $B$ the degree-2 TensorSketch constant. It is minimized at the interior point
\[
m^\star=\sqrt{\frac{B\,M}{A}},\qquad V(m^\star)=\frac{2\sqrt{AB}}{\sqrt M}+\frac{A(d+1)}{M},
\]
so the optimal sketch size grows as $m^\star\propto\sqrt M$ with the budget and as $\sqrt{B/A}$ with the sketch-to-radial constant ratio; its $d$-dependence enters only through that ratio, hence is slow.
\end{proposition}
\begin{proof}
Substitute $D=M/(m+d+1)$ into the radial term $A/D$ of Proposition~\ref{prop:ts_variance} and keep the budget-independent sketch term $B/m$; then $dV/dm=A/M-B/m^2=0$ gives $m^\star=\sqrt{BM/A}$, and substitution gives $V(m^\star)$.
\end{proof}
\noindent Fitting $A,B$ to the deployed-allocation data behind Table~\ref{tab:dm} by this two-term law confirms the rule: it explains $88$--$90\%$ of the squared-error variance ($R^2$), the predicted $m^\star=\sqrt{BM/A}$ lands at or within one grid step of the empirical error minimum, and the ratio $\sqrt{B/A}$ rises only from $0.41$ to $0.58$ as $d$ goes from $16$ to $256$, the slow $d$-dependence the rule predicts (Appendix~\ref{app:exp_details}, \texttt{sketch\_size\_rule}).

\begin{proposition}[The IMQ factor is a finite difference in $b$]\label{prop:imq_findiff}
For every $x,w$, every $b\ge0$, and every step $h>0$,
\begin{equation}\label{eq:imq_findiff}
\frac{k_{\yat,b+2h}(w,x)-2\,k_{\yat,b+h}(w,x)+k_{\yat,b}(w,x)}{2h^2}=\frac{1}{\|w-x\|^2+\varepsilon}=h_\varepsilon(w,x),
\end{equation}
exactly, with no limit required and all three biases $b,\,b+h,\,b+2h$ inside the kernel's domain $b\ge0$.
\end{proposition}
\begin{proof}
The numerator $(w^\top x+b)^2$ is a quadratic polynomial in $b$, so any second difference equals the constant $2h^2$; the forward difference keeps all biases in $b\ge0$ (the centered form would require $b\ge h$). Dividing by the $b$-independent denominator $\|w-x\|^2+\varepsilon$ yields $h_\varepsilon$.
\end{proof}

\section{Kronecker Implementation}
\label{app:implementation}

For kernel evaluation only the inner product is needed, and it factors as~\eqref{eq:kronecker}, computable in $O(D(D'+d))$ per pair. For the full Gram matrix, Algorithm~\ref{alg:gram} assembles it via Hadamard products.

\begin{algorithm}[h]
\caption{RAY Gram-matrix approximation}
\label{alg:gram}
\begin{algorithmic}[1]
\REQUIRE Data $X=\{x_1,\dots,x_N\}$, parameters $D,D',\varepsilon,b$
\STATE Draw $t_1,\dots,t_D\sim\Exp(\varepsilon)$
\STATE Compute polynomial Gram $P_{ij}=(x_i^\top x_j+b)^2/\varepsilon$ \hfill $O(N^2 d)$
\STATE $K_{\mathrm{approx}}\gets\mathbf{0}_{N\times N}$
\FOR{$j=1,\dots,D$}
\STATE Draw $\omega_{j,\ell}\sim\mathcal{N}(0,2t_j I_d)$, $\beta_{j,\ell}\sim\mathrm{Unif}([0,2\pi])$
\STATE Build $\Psi_j\in\mathbb{R}^{N\times D'}$, $\Psi_j[i,\ell]=\sqrt{2/D'}\cos(\omega_{j,\ell}^\top x_i+\beta_{j,\ell})$
\STATE $K_{\mathrm{approx}}\mathrel{+}=(\Psi_j\Psi_j^\top)\circ P/D$ \hfill $O(N^2 D'+ND'd)$
\ENDFOR
\RETURN $K_{\mathrm{approx}}$
\end{algorithmic}
\end{algorithm}

\section{Spectral Interpretation: Bernstein--Schur Kernels}
\label{app:spectral}

The scheme draws $t$ from the Bernstein--Widder mixing measure of the completely monotone radial factor $h_\varepsilon$. This is not a Bochner spectral measure ($k_{\yat,b}$ is not shift-invariant) but plays the analogous role, and the construction depends on $h_\varepsilon$ only through this measure. Here we prove the general Theorem~\ref{thm:bernstein_schur}.

\begin{proof}[Proof of Theorem~\ref{thm:bernstein_schur}]
For one draw set $Y=2\cos(\omega^\top x+\beta)\cos(\omega^\top w+\beta)$; conditioned on $T$ the standard RFF identity gives $\mathbb{E}[Y\mid T]=e^{-T\|x-w\|^2}$, so $\mathbb{E}[m_f\,Y]=m_f\,\mathbb{E}_{T\sim\nu/m_f}[e^{-T\|x-w\|^2}]=\int_0^\infty e^{-t\|x-w\|^2}d\nu(t)=f(\|x-w\|^2)$. With $u(x)^\top u(w)=p(x,w)$, $\mathbb{E}[z(x)^\top z(w)]=p(x,w)f(\|x-w\|^2)=k(x,w)$. For the variance, $|m_f\,Y\,u(x)^\top u(w)|\le 2m_fB^2$ and $\mathbb{E}[Y^2\mid T]\le\tfrac32$ (as in the proof of Theorem~\ref{thm:exact_variance}), so $\Var[m_f\,Y\,u(x)^\top u(w)]\le m_f^2B^4\,\mathbb{E}[Y^2]\le\tfrac32 m_f^2B^4$; averaging $D$ i.i.d.\ copies divides by $D$. For the uniform bound, each summand lies in an interval of length $4m_fB^2$, so Hoeffding gives $\mathbb{P}(|z(x_i)^\top z(x_j)-k(x_i,x_j)|\ge s)\le2\exp(-Ds^2/(8m_f^2B^4))$, and a union bound over $\le N^2$ pairs yields the claim.
\end{proof}
\noindent The biased $\yat$-kernel is the case $u=p_b$ ($d_p=d(d{+}1)/2+d+1$), $f(r)=(r+\varepsilon)^{-1}$ with $d\nu(t)=e^{-\varepsilon t}\,dt$, so $m_f=1/\varepsilon$ and $\nu/m_f=\Exp(\varepsilon)$; nontrivial members are tabulated in the main text (Table~\ref{tab:bernstein_schur}).

\begin{proof}[Proof of Theorem~\ref{thm:class_bernstein}]
The class estimator of Theorem~\ref{thm:bernstein_schur} has per-draw Gram $K^{(j)}=(\psi_j\psi_j^\top)\circ P_u$ with $\psi_j\in\mathbb{R}^N$, $\psi_j[i]=\sqrt2\cos(\omega_j^\top x_i+\beta_j)$, and $P_u=m_fG_u$ (the mass $m_f$ entering through the $\sqrt{m_f/D}$ scaling of the feature). The four cited proofs consume exactly three structural facts about the triple $(K^{(j)},P,K)$; we verify each with $P_u$ in place of $P$.

\emph{(i) PSD per-draw Gram with bounded operator norm.} $\psi_j\psi_j^\top\succeq0$ is rank one with diagonal $2\cos^2(\omega_j^\top x_i+\beta_j)\le2$, and $G_u=[u(x_i)^\top u(x_j)]\succeq0$ is a Gram matrix, so $K^{(j)}\succeq0$ by the Schur product theorem and $\|K^{(j)}\|_{\mathrm{op}}\le2\|P_u\|_{\mathrm{op}}$ by Lemma~\ref{lem:schur}(a) with $a=2$.

\emph{(ii) Unbiasedness through a unit-diagonal PSD radial Gram.} Conditional on $T_j$, the RFF identity gives $\mathbb{E}[(\psi_j\psi_j^\top)_{ik}\mid T_j]=e^{-T_j\|x_i-x_k\|^2}$; averaging over $T_j\sim\nu/m_f$ yields $\mathbb{E}[K^{(j)}]=R_u\circ P_u=K$ with $(R_u)_{ik}=f(\|x_i-x_k\|^2)/m_f$. The radial Gram $R_u$ is PSD (a nonnegative mixture of Gaussian Grams under $\nu/m_f$) with unit diagonal ($f(0)=m_f$), so Lemma~\ref{lem:schur}(a) also gives $K\succeq0$ and $\|K\|_{\mathrm{op}}\le\|P_u\|_{\mathrm{op}}$.

\emph{(iii) The variance step.} For PSD $M$, $M^2\preceq\|M\|_{\mathrm{op}}M$ and $MA^{-1}M\preceq\lambda^{-1}\|M\|_{\mathrm{op}}M$; with (i) these give $\mathbb{E}[(K^{(j)}-K)^2]\preceq2\|P_u\|_{\mathrm{op}}K$ and $\mathbb{E}[(K^{(j)}-K)A^{-1}(K^{(j)}-K)]\preceq(2\|P_u\|_{\mathrm{op}}/\lambda)K$, the two variance majorants.

Substituting (i)--(iii) into the proofs of Theorem~\ref{thm:bernstein} and Corollary~\ref{cor:bernstein_tail} (a.s.\ bound $L=3\|P_u\|_{\mathrm{op}}/D$, variance $v=2\|P_u\|_{\mathrm{op}}\|K\|_{\mathrm{op}}/D$) and of Theorem~\ref{thm:krr_whitened} and Corollary~\ref{cor:krr_highprob} (whitened a.s.\ bound $L_\lambda=(1+2\|P_u\|_{\mathrm{op}}/\lambda)/D$, majorant $\frac{2\|P_u\|_{\mathrm{op}}}{\lambda D}A^{-1/2}KA^{-1/2}$, whose intrinsic dimension is again exactly $\tilde d_\lambda$) reproduces every statement verbatim with $P\mapsto P_u$. The leverage theorem transfers identically: its proof consumes, beyond (i)--(iii), only the identity $\mathbb{E}_\pi[\psi\psi^\top]=R_u$ (the unit-diagonal radial Gram of (ii)), so $\bar d_\lambda(\theta)=\psi_\theta^\top(A^{-1}\circ P_u)\psi_\theta$ has mean $d_{\mathrm{eff}}(\lambda)$ and Theorem~\ref{thm:krr_leverage} holds with $P\mapsto P_u$. For $k_{\yat,b}$, $u=p_b$ and $m_f=1/\varepsilon$ give $P_u=[(x_i^\top x_j+b)^2/\varepsilon]=P$.
\end{proof}
\noindent For the polynomially modulated Mat\'ern-$\tfrac12$ kernel $(x^\top w+b)^q\,e^{-\|x-w\|/\sigma}$ of Theorem~\ref{thm:class_bernstein}: the radial factor $f(r)=e^{-\sqrt r/\sigma}$ is completely monotone in $r=\|x-w\|^2$ with the L\'evy/inverse-Gaussian Bernstein representation
\[
e^{-\sqrt r/\sigma}=\int_0^\infty \frac{1}{2\sigma\sqrt{\pi}}\,t^{-3/2}e^{-1/(4\sigma^2t)}\,e^{-tr}\,dt,
\]
mass $m_f=f(0)=1$, and exact sampler $T=1/(2\sigma^2Z^2)$, $Z\sim\mathcal{N}(0,1)$ (the L\'evy law of parameter $c=1/(2\sigma^2)$ is the law of $c/Z^2$); the modulation Gram is $G_u=[(x_i^\top x_j+b)^q]$ with $\|u(x)\|^2=(\|x\|^2+b)^q\le(R^2+b)^q$, so $P_u=G_u$ and every bound of Theorem~\ref{thm:class_bernstein} applies with $\|P_u\|_{\mathrm{op}}\le N(R^2+b)^q$ in the worst case and the data-adaptive $\|P_u\|_{\mathrm{op}}$ in general.

\section{Further Validations}
\label{app:further}

\subsection{Experiment controls and scoped deployment checks}
\label{sec:exp_controls}

\paragraph{Fair-cost comparison.}
\label{sec:exp_faircost}
Matched draws understate RAY's cost; matched representation is the honest axis. The Nystr\"om methods here are strong \emph{data-dependent} deployment baselines (Table~\ref{tab:method_taxonomy}), not data-independent feature maps. On the off-sphere bounded ball (Table~\ref{tab:offsphere_faircost}; $d=64$, $\|x\|\in[0.3,1.5]$, coupled target, $M=d_b=2145$), RAY far outperforms the budget-starved exact-modulation limit, but adaptive k-means and ridge-leverage Nystr\"om of the \emph{exact} $\yat$-kernel are most accurate. This is the main deployment caveat: at moderate $N$ and fixed representation size, adaptive landmarks can be better unless streaming, unbiasedness, or pre-data features are required.

\begin{table}[h]
\centering
\caption{Off-sphere fair-cost: KRR test RMSE ($\downarrow$) on a bounded ball ($d=64$, $\|x\|\in[0.3,1.5]$, coupled target, mean $\pm$ std over $3$ seeds) at matched representation dimension. RAY beats the budget-starved exact-modulation limit; k-means and ridge-leverage Nystr\"om of the exact $\yat$-kernel are most accurate; Gaussian RFF approximates a different kernel. The sphere-normalized version is Table~\ref{tab:faircost}.}
\label{tab:offsphere_faircost}
\begin{tabular}{lcccc}
\toprule
Method & dim & RMSE & memory (MB) & build (s) \\
\midrule
exact modulation & $2145$ & $1.114{\scriptstyle\pm.360}$ & $25.7$ & $0.020$ \\
RAY ($m{=}128$) & $2123$ & $0.473{\scriptstyle\pm.010}$ & $25.5$ & $\mathbf{0.015}$ \\
k-means Nystr\"om & $1500$ & $\mathbf{0.099}{\scriptstyle\pm.003}$ & $\mathbf{18.0}$ & $1.887$ \\
rls-leverage Nystr\"om & $1500$ & $\mathbf{0.099}{\scriptstyle\pm.003}$ & $\mathbf{18.0}$ & $0.928$ \\
\midrule
Gaussian RFF (different kernel) & $2145$ & $0.260{\scriptstyle\pm.003}$ & $25.7$ & $0.024$ \\
\bottomrule
\end{tabular}
\end{table}

\paragraph{Tuned and matched-budget controls for the coupling test.}
The coupled-target preference of Table~\ref{tab:necessity} survives fair per-kernel tuning (Table~\ref{tab:tuned}): with every kernel's $b,\varepsilon,\lambda$ grid-searched on a held-out split, the exact $\yat$-kernel is best on the coupled target and statistically tied on \texttt{digits}, where one factor already suffices.

\begin{table}[h]
\centering
\caption{Validation-tuned downstream comparison (every kernel's $b,\varepsilon,\lambda$ grid-searched on a held-out split; off-sphere; mean $\pm$ std over $3$ seeds). The $\yat$-kernel's coupled-target advantage survives fair tuning: it wins the synthetic coupled (alignment$\times$proximity) target and ties on \texttt{digits} where one factor suffices. RAY is the deployed sketched approximation at a fixed budget.}
\label{tab:tuned}
\begin{tabular}{lcc}
\toprule
kernel & coupled (RMSE $\downarrow$) & digits (acc $\uparrow$) \\
\midrule
Gaussian   & $0.038{\scriptstyle\pm.001}$ & $0.980{\scriptstyle\pm.007}$ \\
IMQ        & $0.036{\scriptstyle\pm.002}$ & $0.979{\scriptstyle\pm.004}$ \\
polynomial & $0.114{\scriptstyle\pm.008}$ & $\mathbf{0.982}{\scriptstyle\pm.006}$ \\
$\yat$ (exact) & $\mathbf{0.034}{\scriptstyle\pm.001}$ & $0.978{\scriptstyle\pm.004}$ \\
RAY (sketched) & $0.078{\scriptstyle\pm.011}$ & $0.974{\scriptstyle\pm.005}$ \\
Nystr\"om-$\yat$ & $0.043{\scriptstyle\pm.003}$ & $0.979{\scriptstyle\pm.003}$ \\
\bottomrule
\end{tabular}
\end{table}

At small fixed feature dimension the story changes because resolving the coupling through random features costs the polynomial floor. Table~\ref{tab:coupled_matched} shows this explicitly: the exact $\yat$ kernel keeps the lead on the coupled target, but sketched RAY is under-resolved at $M=4096$ and a cheaper single-factor RFF can score better. This does not negate the kernel preference; it identifies the regime where the feature construction is meant to be used.

\begin{table}[h]
\centering
\caption{Matched-dimension random-feature comparison (RMSE $\downarrow$, off-sphere $d{=}32$, $M{=}4096$, mean of $3$ seeds). At this small fixed budget the $O(d^2)$ polynomial floor leaves the coupling under-resolved, so a cheaper RFF for a simpler kernel scores better while the exact $\yat$-kernel keeps the lead; RAY's regime is large-$N$ streaming, not a fixed-budget bake-off.}
\label{tab:coupled_matched}
\begin{tabular}{lccccc}
\toprule
target & RAY (sketched) & Gaussian RFF & IMQ RFF & poly sketch & exact $\yat$ \\
\midrule
coupled   & $0.103$ & $0.038$ & $0.039$ & $0.143$ & $\mathbf{0.025}$ \\
proximity & $0.477$ & $0.214$ & $\mathbf{0.102}$ & $0.447$ & $0.061$ \\
alignment & $0.131$ & $0.102$ & $0.108$ & $\mathbf{0.016}$ & $0.088$ \\
\bottomrule
\end{tabular}
\end{table}

\paragraph{The alignment numerator as signal modulation.}
\label{sec:exp_gate}
The coupled-target preference is explained by the modulation--radial decomposition (Proposition~\ref{prop:gate}): radial Monte-Carlo noise enters each Gram entry scaled by the alignment modulation $p_b$, so the product sharpens \emph{between}-pair discrimination while leaving any single pair's relative error unchanged. On off-sphere data ($d=16$) we score three pair types (\emph{true} (close and aligned), \emph{radial distractors} (close but weakly aligned), and \emph{alignment distractors} (aligned but far)) by the AUC separating true from each ($400$ pairs/type). A radial-only kernel (IMQ) false-positives on radial distractors (AUC $0.22$); an alignment-only kernel (degree-2 polynomial) false-positives on alignment distractors (AUC $0.00$); only the $\yat$ product suppresses both (AUC $1.00$/$1.00$), the exact-modulation estimator inherits this almost exactly, and the deployed RAY (sketched modulation) partially survives on the alignment distractors (AUC $0.73$).

\paragraph{Linear-time streaming attention.}
\label{sec:exp_attention}
Attention is a kernel smoother over tokens, $\mathrm{attn}(q_i)=\sum_j k(q_i,k_j)v_j/\sum_j k(q_i,k_j)$, at $O(N^2)$ cost. An explicit feature map collapses it to $O(NM)$; with $\phi$ such that $\phi(q)^\top\phi(k)\approx k(q,k)$, the smoother factorizes as $\phi(q_i)^\top(\sum_j\phi(k_j)v_j^\top)/\phi(q_i)^\top\sum_j\phi(k_j)$. RAY supplies such a map for $k_{\yat,b}$ and is causal through the recurrent state $S_t=S_{t-1}+\phi(k_t)v_t^\top$. A landmark method cannot serve here because its landmarks are chosen from the sequence.

We measure fidelity and scaling, not language quality. The linear-attention output matches exact $\yat$-attention with a median per-token error that falls with the feature dimension $M$ (Figure~\ref{fig:attention}a, $d=32$: $0.45$ at $M{=}1552$ down to $0.12$ at $M{=}12416$), and the induced attention-weight matrix tracks it ($0.57\to0.12$). The approximation does not drift upward with context length, but it does depend on sharpness: diffuse attention is easy, peaked attention is the hard regime for the radial RFF (Figure~\ref{fig:attention}b). The payoff is the $O(N^2)$ wall (Figure~\ref{fig:attention}c): the exact $N\times N$ attention matrix reaches $137$\,GB at $N{=}131072$ while RAY stays linear and a single recurrent state decodes causally.

The sign indefiniteness has a principled fix, and whether that fix also cures the peaked regime has a sharp, negative answer. The trigonometric radial feature is the source both of the peaked-regime error and of an estimator that can return negative attention weights; a positive feature removes the second problem (Proposition~\ref{prop:positive}), and Proposition~\ref{prop:pos_dichotomy} shows it cannot remove the first.

\begin{proposition}[Positive features give nonnegative attention]\label{prop:positive}
The Gaussian radial factor admits a positive random feature: for $t>0$ and $\omega\sim\mathcal N(0,2tI_d)$,
\begin{equation}\label{eq:favor}
e^{-t\|x-w\|^2}=\mathbb{E}_\omega\bigl[\phi^+_t(x)\,\phi^+_t(w)\bigr],\qquad \phi^+_t(x)=\exp\!\bigl(\omega^\top x-2t\|x\|^2\bigr)>0,
\end{equation}
the FAVOR$^+$ construction of \citet{choromanski2021rethinking}, with controlled relative error where the kernel is small. Combined with a nonnegative modulation feature, anchor squares $\phi_a(x)=(a^\top x)^2\ge0$ (exact on the sphere, Section~\ref{sec:discussion}), the product feature $\Phi(x)=\phi^+_t(x)\,\phi_a(x)$ is entrywise positive, so the estimated $\yat$-attention weights $\Phi(q)^\top\Phi(k)\ge0$ are nonnegative by construction, which the signed trigonometric/TensorSketch estimator cannot guarantee.
\end{proposition}
\begin{proof}
For $\omega\sim\mathcal N(0,2tI_d)$, $\mathbb{E}_\omega[e^{\omega^\top(x+w)}]=e^{t\|x+w\|^2}$, so $\mathbb{E}_\omega[\phi^+_t(x)\phi^+_t(w)]=e^{-2t\|x\|^2-2t\|w\|^2}e^{t\|x+w\|^2}=e^{-t\|x-w\|^2}$, which is~\eqref{eq:favor}; positivity of $\phi^+_t$ and of $\phi_a$ is immediate, and a product of nonnegative features has a nonnegative inner product.
\end{proof}

\begin{proposition}[Variance dichotomy: the mixing law caps positive features]\label{prop:pos_dichotomy}
Let $\widehat g^+=\phi^+_T(x)\,\phi^+_T(w)$ with $T\sim\Exp(\varepsilon)$, $\omega\mid T\sim\mathcal N(0,2TI_d)$, the positive-feature estimate of the rescaled radial factor. Then $\mathbb{E}[\widehat g^+]=\varepsilon/(\|x-w\|^2+\varepsilon)$, but
\[
\mathbb{E}\bigl[(\widehat g^+)^2\bigr]=\mathbb{E}_T\bigl[e^{8Tx^\top w}\bigr]=
\begin{cases}
\dfrac{\varepsilon}{\varepsilon-8x^\top w}, & 8x^\top w<\varepsilon,\\[4pt]
+\infty, & 8x^\top w\ge\varepsilon,
\end{cases}
\]
so the positive-feature radial estimator has \emph{infinite variance} on every pair with $x^\top w\ge\varepsilon/8$. Truncating the scale at $T\le T_{\max}$ (relative bias at most $e^{-\varepsilon T_{\max}}$ on the radial factor) caps the second moment at $e^{8T_{\max}x^\top w}$, which remains exponential in $x^\top w/\varepsilon$. The trigonometric estimator instead satisfies $\mathbb{E}[\widehat g^2\mid T]=1+\tfrac12e^{-4T\|x-w\|^2}\le\tfrac32$ uniformly in the scale and the pair.
\end{proposition}
\begin{proof}
Conditionally on $T=t$, $(\widehat g^+)^2=e^{2\omega^\top(x+w)}e^{-4t(\|x\|^2+\|w\|^2)}$, and $\mathbb{E}_\omega[e^{a^\top\omega}]=e^{t\|a\|^2}$ for $\omega\sim\mathcal N(0,2tI_d)$ gives $\mathbb{E}_\omega[(\widehat g^+)^2\mid t]=e^{4t\|x+w\|^2-4t\|x\|^2-4t\|w\|^2}=e^{8tx^\top w}$. Integrating against $\varepsilon e^{-\varepsilon t}\,dt$ gives $\varepsilon/(\varepsilon-8x^\top w)$ when $8x^\top w<\varepsilon$ and diverges otherwise; on $\{T\le T_{\max}\}$ the conditional bound $e^{8tx^\top w}\le e^{8T_{\max}x^\top w}$ (for $x^\top w\ge0$; for $x^\top w<0$ it is at most $1$) caps the second moment, and the truncated integral $\int_0^{T_{\max}}\varepsilon e^{-\varepsilon t}e^{-tr}\,dt$ misses at most $e^{-(\varepsilon+r)T_{\max}}\le e^{-\varepsilon T_{\max}}$ of the total in relative terms. The trigonometric bound is in the proof of Theorem~\ref{thm:exact_variance}.
\end{proof}
\noindent Empirically, the dichotomy is sharp: below the threshold the finite-side prediction is approached, while above it the empirical second moment grows without plateau and the positive-feature Gram estimator becomes heavy-tailed or catastrophic; the signed trigonometric estimator keeps the Monte-Carlo rate. Positive features are therefore appropriate only in diffuse regimes where the radial RFF is already accurate and nonnegative weights matter.

\begin{figure}[t]
\centering
\includegraphics[width=\textwidth]{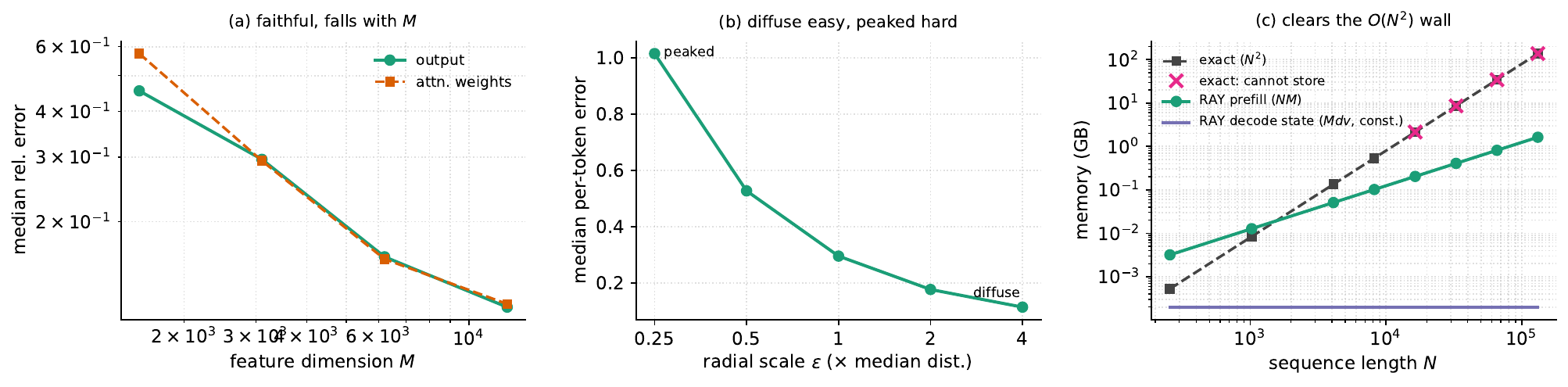}
\caption{RAY as a linear-time, streaming $\yat$-attention primitive (random queries/keys/values, $d=32$). \textbf{(a)} The linear-attention output and the induced attention-weight matrix both match exact $\yat$-attention with a median error that falls with the feature dimension $M$; one fixed map is applied to every token. \textbf{(b)} The one limitation: error scales with attention sharpness: diffuse attention (large radial scale $\varepsilon$) is easy, peaked attention is the hard regime for the radial RFF. \textbf{(c)} The exact $N\times N$ attention matrix becomes impossible to store ($137$\,GB at $N{=}131072$, $\boldsymbol\times$), while RAY's $O(NM)$ features stay linear in $N$ and a single $O(M d_v)$ recurrent state (constant in $N$) decodes causally; that recurrence $S_t=S_{t-1}+\phi(k_t)v_t^\top$ is exact, which a data-dependent landmark scheme cannot offer.}
\label{fig:attention}
\end{figure}

\paragraph{Gram-matrix approximation error.}
\label{sec:exp_gram}

\begin{figure}[h]
\centering
\includegraphics[width=\textwidth]{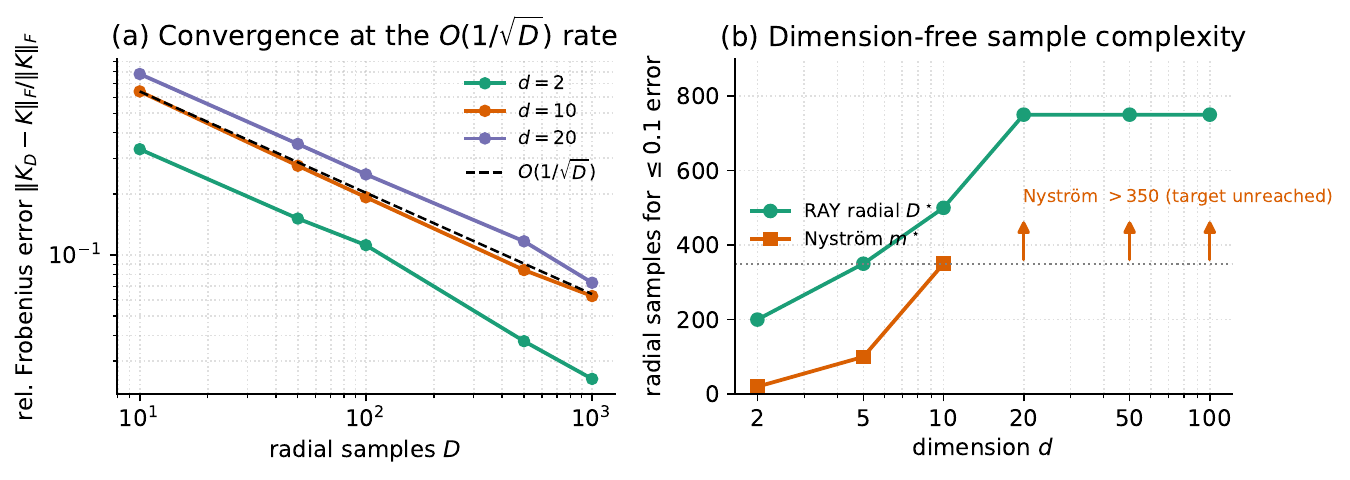}
\caption{Sphere-normalized sanity check (here the kernel coincides with a dot-product kernel, so this isolates the dimension behavior and is not a representation claim). RAY approximates the biased Gram at the Monte-Carlo rate with a radial sample count that grows little with dimension (flat $D'=1$). \textbf{(a)} Relative Frobenius error vs.\ $D$ ($N=1000$, $b=1$, $\varepsilon=1$); all dimensions track the $O(1/\sqrt D)$ guide within a factor of $\sim3$. \textbf{(b)} Radial samples for $\le10\%$ error vs.\ dimension: RAY's $D^\star$ stays bounded and plateaus by $d=20$, while the Nystr\"om landmark count $m^\star$ exceeds the tested range for $d\ge20$.}
\label{fig:overview}
\end{figure}

Random $\yat$-features converge at the Monte-Carlo rate, and their accuracy grows only mildly with dimension. Approximating the exact biased Gram matrix of $N=1000$ sphere-normalized points (normalization fixes $R=1$, isolating the dimension behavior from the $R^2$ growth of raw data; $b=1$, $\varepsilon=1$, mean over $5$ seeds, the recommended flat $D'=1$ so $D$ is the number of cosine draws), the relative Frobenius error falls as $O(1/\sqrt D)$ in the number of radial samples $D$ (from $0.33$ at $D=10$ to $0.024$ at $D=1000$ at $d=2$ (Figure~\ref{fig:overview}a)) and at $D=1000$ stays within $[0.024,0.073]$ across $d\in\{2,\dots,20\}$, a factor of $\sim3$. Uniform-landmark Nystr\"om at a fixed budget behaves oppositely (Table~\ref{tab:gram_error}; we add stronger k-means and adaptive ridge-leverage-score Nystr\"om \citep{musco2017recursive} baselines in Section~\ref{sec:exp_offsphere}): at $d=2$ the error is below the table's displayed precision, likely reflecting unusually rapid spectral decay on this sampled problem rather than literal exact recovery, but by $d=20$ it degrades to $0.39$ ($m{=}50$) and $0.29$ ($m{=}100$). At matched draws versus landmarks, RAY is therefore preferable in moderate-to-high dimension, Nystr\"om only in low. The operator-norm error controlled by Theorem~\ref{thm:bernstein} tracks the Frobenius error: the same $O(1/\sqrt D)$ rate, and smaller in relative terms ($0.43\to0.059$ vs.\ $0.59\to0.071$ as $D$ runs $10\to1000$ at $d=10$). The theorem's data-adaptive constant is borne out directly: across bounded-ball datasets of increasing spectral spread (intrinsic dimension $d_{\mathrm{int}}$ from $1.4$ to $3.5$, $N=400$) the matrix-Bernstein expression stays near $55$--$60$ while the crude $N\max_{ij}$ route grows from $66$ to $147$, the gap widening with the spectrum, and both upper-bound the measured $\|K_D-K\|_{\mathrm{op}}$ ($20$--$24$).

\begin{table}[h]
\centering
\caption{Relative Frobenius error $\|K_{\mathrm{approx}}-K\|_F/\|K\|_F$ of the biased Gram matrix ($b=1$, $\varepsilon=1$, unit sphere, mean over $5$ seeds). RAY uses $D$ radial samples at the recommended flat $D'=1$, so $D$ is the number of cosine draws (total feature dimension $D\,d_b$); Nystr\"om uses $m$ landmarks.}
\label{tab:gram_error}
\begin{tabular}{lcccc}
\toprule
Method & $d=2$ & $d=5$ & $d=10$ & $d=20$ \\
\midrule
RAY, $D{=}100$  & $0.112$ & $0.174$ & $0.192$ & $0.250$ \\
RAY, $D{=}500$  & $0.038$ & $0.076$ & $0.084$ & $0.117$ \\
RAY, $D{=}1000$ & $0.024$ & $0.045$ & $0.063$ & $0.073$ \\
\midrule
Nystr\"om, $m{=}50$  & $0.000$ & $0.188$ & $0.332$ & $0.393$ \\
Nystr\"om, $m{=}100$ & $0.000$ & $0.081$ & $0.216$ & $0.289$ \\
\bottomrule
\end{tabular}
\end{table}

\paragraph{Bias scaling.}
\label{sec:exp_bias}

The bias enters the variance exactly as the fourth power Theorem~\ref{thm:variance} predicts (Figure~\ref{fig:bias}). For two unit-vector pairs in $\mathbb{R}^2$ (aligned ($x^\top w=1$) and $x^\top w=0.5$) we sweep $b\in\{0,\dots,5\}$ at $D=200$, $\varepsilon=1$ and fit the log-log slope of the empirical variance ($2000$ repetitions) against $x^\top w+b$.

\begin{figure}[h]
\centering
\includegraphics[width=0.5\textwidth]{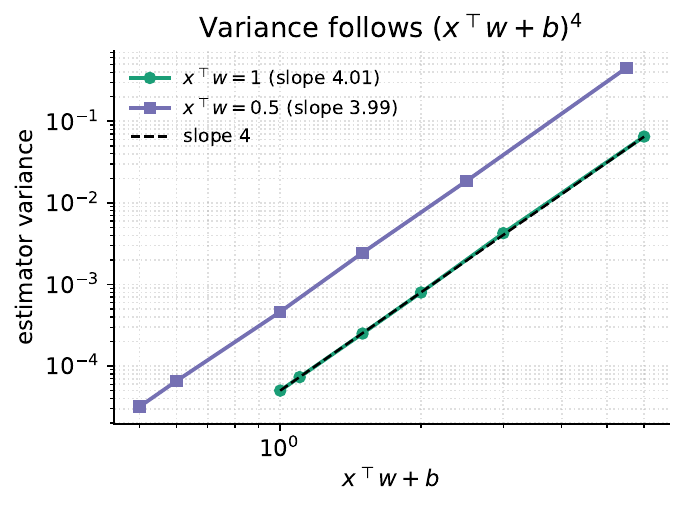}
\caption{Estimator variance vs.\ the bias-shifted alignment $x^\top w+b$ (log-log, $2000$ repetitions). Both pairs follow a fourth power (fitted slopes $4.01$ and $3.99$ against the slope-$4$ guide). For the aligned pair the variance equals the $(R^2+b)^4$ envelope of Theorem~\ref{thm:variance} (the ratio $\Var/(R^2+b)^4$ is constant at $\approx5\times10^{-5}$, so the bound is tight); the $x^\top w=0.5$ pair lies below it, the gap being the Cauchy--Schwarz step in the proof.}
\label{fig:bias}
\end{figure}

The fitted exponents are $4.01$ and $3.99$ (and $2.01$, $1.99$ for the standard deviation), matching the predicted $4$ and $2$. The aligned pair sits on the $(R^2+b)^4$ envelope, where $\Var/(R^2+b)^4$ is constant and the bound is tight; the $x^\top w=0.5$ pair lies strictly below it, the gap being the Cauchy--Schwarz step $x^\top w\le\|x\|\|w\|$ in the proof.

Beyond the envelope, the \emph{exact} variance of Theorem~\ref{thm:exact_variance} is confirmed pointwise: across $b\in\{0,1,2\}$ and $18$ off-sphere pairs spanning $r=\|x-w\|^2\in[0.4,2.2]$, the ratio of empirical (over $2\times10^5$ draws) to predicted variance is $1.001\pm0.002$, and the estimator is unbiased to within $0.5\%$.

\paragraph{The normalized estimator removes the bias/radius variance blow-up.}
\label{sec:exp_normalized}

The normalized variant (Proposition~\ref{prop:normalized}) trades $k_{\yat,b}$ for a bounded-variance rescaling. We confirm the exact relation $\Var(\text{exact modulation})/\Var(\text{normalized})=((\|x\|^2+b)(\|w\|^2+b))^2$ by sweeping $b$ at $R=1$ and the data radius $R$ at $b=1$ ($2\times10^5$ draws per point, Table~\ref{tab:normalized}). The measured ratio matches $(R^2+b)^4$ to three figures, the exact-modulation variance grows from $0.02$ to $449$ across the sweep, and the normalized variant stays bounded in $[0.045,0.53]\le\tfrac32$, as the proposition predicts. The normalized estimator is therefore the stable choice for large-radius or large-bias data, at the cost of estimating the rescaled kernel.

\begin{table}[h]
\centering
\caption{Flat-estimator variance of exact vs.\ normalized modulation ($\varepsilon=1$, $2\times10^5$ draws). The ratio equals $(R^2+b)^4$ exactly; exact modulation blows up while the normalized estimator stays bounded ($\le\tfrac32$).}
\label{tab:normalized}
\begin{tabular}{lcccc@{\hskip 2em}lcccc}
\toprule
\multicolumn{5}{c}{$b$-sweep ($R{=}1$)} & \multicolumn{5}{c}{$R$-sweep ($b{=}1$)} \\
\cmidrule(lr){1-5}\cmidrule(lr){6-10}
$b$ & exact & norm. & ratio & $(1{+}b)^4$ & $R$ & exact & norm. & ratio & $(R^2{+}1)^4$ \\
\midrule
$0$ & $0.02$ & $0.02$ & $1.0$ & $1$    & $0.5$ & $0.93$ & $0.38$ & $2.4$  & $2.4$ \\
$1$ & $3.39$ & $0.21$ & $16$  & $16$   & $1.0$ & $3.38$ & $0.21$ & $16$   & $16$ \\
$2$ & $29.3$ & $0.36$ & $81$  & $81$   & $2.0$ & $45.6$ & $0.07$ & $625$  & $625$ \\
$4$ & $331$  & $0.53$ & $625$ & $625$  & $3.0$ & $449$  & $0.045$& $10^4$ & $10^4$ \\
\bottomrule
\end{tabular}
\end{table}

\paragraph{Fixed-dataset radial sample complexity.}
\label{sec:exp_dimfree}

The radial sample count needed for a target accuracy grows little with dimension (and then plateaus) whereas Nystr\"om's landmark count does not stay bounded at all. With the recommended flat $D'=1$ (so $D$ is the number of cosine draws), for each $d\in\{2,\dots,100\}$ ($N=500$ on the sphere, $b=\varepsilon=1$, $3$ seeds) we find $D^\star$, the smallest $D$ reaching relative Frobenius error $\le0.10$, and the analogous Nystr\"om count $m^\star$ (Table~\ref{tab:dimfree}, Figure~\ref{fig:overview}b).

\begin{table}[h]
\centering
\caption{Radial samples needed for relative Frobenius error $\le0.10$ vs.\ dimension ($b=1$, $\varepsilon=1$, unit sphere, flat $D'=1$). RAY's radial count $D^\star$ stays bounded and plateaus; Nystr\"om's landmark count $m^\star$ exceeds the tested range ($>350$) for $d\ge20$. ``$>350$'' means $350$ landmarks did not reach the target.}
\label{tab:dimfree}
\begin{tabular}{lcccccc}
\toprule
$d$ & $2$ & $5$ & $10$ & $20$ & $50$ & $100$ \\
\midrule
RAY $D^\star$        & $200$ & $350$ & $500$ & $750$ & $750$ & $750$ \\
Nystr\"om $m^\star$  & $20$ & $100$ & $350$ & $>350$ & $>350$ & $>350$ \\
\bottomrule
\end{tabular}
\end{table}

RAY's $D^\star$ rises from $200$ to $750$ and then plateaus (no further growth from $d=20$ to $100$), while Nystr\"om reaches the target only for $d\le10$: from $d=20$ on, $350$ landmarks leave the error above $0.10$ (at $d=100$, $0.20$). The contrast is between a radial count that saturates in $d$ and Nystr\"om's curse of dimensionality $O(m^{-2s/d})$, the empirical content of Table~\ref{tab:rf_comparison}; the total RAY feature dimension is $D^\star d_b$, which does grow with $d$ through the polynomial factor.

\paragraph{Downstream kernel ridge regression.}
\label{sec:exp_krr}

The exact polynomial factor becomes a large downstream advantage at a low number of random draws. On two real datasets (\texttt{digits} (classification, $d=64$) and \texttt{california} (regression, $d=8$), standardized and $\ell_2$-normalized to the sphere, with $\varepsilon$ the median squared distance, $b=1$, and ridge $\lambda=10^{-2}$, all fixed rather than tuned on validation) we fit kernel ridge regression in Gram form (to isolate kernel-approximation quality; Appendix~\ref{sec:exp_scaling} evaluates primal-feature scalability separately) and compare, at a \emph{matched number of random draws} $D$, RAY against IMQ random features (the radial-only estimator of Section~\ref{sec:step-rff}), Gaussian RFF, and Nystr\"om on the exact $\yat$-kernel, with the three exact kernels as references. Matching $D$ matches the Gram-assembly compute up to the one-time $O(N^2d)$ polynomial Gram, a small overhead next to the gap below. Figure~\ref{fig:krr} reports the mean over $3$ splits with $\pm1$ standard-deviation bands.

\begin{figure}[h]
\centering
\includegraphics[width=\textwidth]{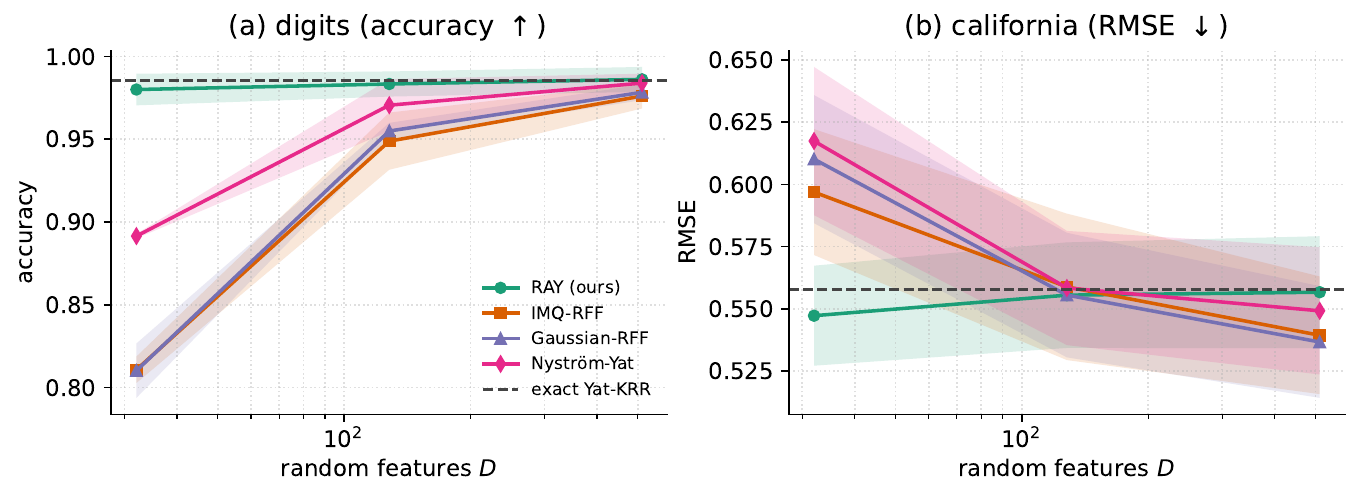}
\caption{Downstream KRR test metric vs.\ the number of random draws $D$ on sphere-normalized real data (mean over $3$ splits, $\pm1$ s.d.\ bands); the dashed line is the exact $\yat$-kernel. \textbf{(a)} digits: RAY-$\yat$ sits at the exact-kernel accuracy already at $D=8$, while Gaussian RFF, IMQ-RFF, and Nystr\"om climb slowly and need $\sim\!512$ features to catch up. \textbf{(b)} california: RAY tracks the exact kernel from the smallest budgets. RAY keeps the alignment numerator exact and spends its $D$ draws only on the radial factor; the others must sample the whole kernel. As \emph{exact} kernels the $\yat$-kernel ties Gaussian on digits ($0.986$) and trails it on california ($0.558$ vs.\ $0.531$ RMSE, IMQ $0.535$).}
\label{fig:krr}
\end{figure}

With only $D=32$ radial samples RAY already reaches $0.980$ accuracy on digits and comes within $0.01$ RMSE of the exact $\yat$-kernel on california, while Gaussian RFF, IMQ-RFF, and Nystr\"om at the same budget trail by a wide margin ($0.81$--$0.89$ on digits): they spend their draws approximating the whole kernel, whereas RAY keeps the alignment numerator exact and spends its $D$ draws only on the radial factor. The IMQ-RFF ablation isolates this: removing the numerator and adding it back lifts digits accuracy from $0.81$ to $0.98$ at $D=32$, so the exact numerator is the source of the low-budget efficiency. As an \emph{exact} kernel the $\yat$-kernel is competitive but not uniformly best (it ties Gaussian on digits and trails it by $0.03$ RMSE on california); the gain here is RAY's draw efficiency. Two caveats define the scope: the comparison is matched in the number of \emph{draws}, not in explicit feature dimension, which for RAY carries the polynomial factor $d_b=O(d^2)$ per draw, so the draw-efficiency advantage does not transfer to a dimension-matched setting in high $d$, where the on-sphere dot-product route (a direction we leave to future work) is preferable; and the $\yat$-kernel needs inputs of bounded norm: on raw standardized features its unbounded numerator destabilizes both the exact kernel and RAY (exact-$\yat$ RMSE $0.84$, RAY diverging), so we keep $\|x\|$ bounded (here by sphere normalization; Section~\ref{sec:exp_offsphere} validates the off-sphere bounded-ball regime).

The relative-spectral KRR theorem (Theorem~\ref{thm:krr_spectral}) holds as stated. On a digits subset ($N=600$, $\lambda=0.1$) the whitened error $\rho=\|A^{-1/2}(K_D-K)A^{-1/2}\|_{\mathrm{op}}$ falls with $D$ (log-log slope $-0.83$), and once $\rho<1$ (at $D=1024$, $\rho=0.69$) the predicted bound $\|\tilde\alpha-\hat\alpha\|_A\le\frac{\rho}{1-\rho}\|\hat\alpha\|_A$ holds with a wide margin ($0.18\le2.18$); at smaller $D$ the relative error decreases monotonically ($1.59\to0.18$ over $D{:}16\to1024$) as the deterministic guarantee predicts once it becomes active.

\paragraph{Real-data check: competitive across regimes.} The coupled-target preference study (Section~\ref{sec:exp_necessity}) is synthetic by construction. The same reading (a single $\yat$-kernel stays competitive whichever geometry carries the labels, rather than failing on one) can be tested on real data. We embed CIFAR-10 and CIFAR-100 with a frozen CLIP image encoder \citep{radford2021clip} and run exact kernel ridge regression with the Gaussian, IMQ, degree-2 polynomial, and $\yat$ kernels under two preprocessings of the embeddings: centered, where class identity is angular and alignment dominates, and bounded-ball, where it is radial and proximity dominates (Table~\ref{tab:cifar}). Across five seeds the four kernels agree to within about one standard deviation: on CIFAR-10 they are statistically tied at $\approx0.944$, and on CIFAR-100 the polynomial (alignment) kernel leads by roughly one standard deviation with the $\yat$-kernel close behind. The reading is modest but consistent with the coupling picture: the $\yat$-kernel is competitive in both regimes, never far from the best, because it carries whichever factor is informative, whereas a single-factor kernel would be exposed on the regime it is blind to; on these embeddings, though, one factor already suffices, so no kernel separates decisively. And this is a statement about the \emph{kernel}: at $d=512$ the random-feature $\yat$ needs more draws than a single-factor RFF to realize it, consistent with the $O(d^2)$ representation floor (Limitation~(v)).

\begin{table}[h]
\centering
\caption{Exact-kernel KRR test accuracy on frozen CLIP embeddings of CIFAR-10/100 ($10$k train / $2$k test, fixed $\lambda$), under centered (alignment-dominated) and bounded-ball (proximity-dominated) preprocessing. Means over $5$ random train/test subsamples; per-cell standard deviation $\le0.006$. On CIFAR-10 all four kernels agree within one standard deviation; on CIFAR-100 the polynomial leads by about one standard deviation. Best mean per row in bold (CIFAR-10 alignment is a four-way tie).}
\label{tab:cifar}
\begin{tabular}{llcccc}
\toprule
dataset & regime & Gaussian & IMQ & polynomial & $\yat$ \\
\midrule
CIFAR-10  & alignment & $0.944$ & $0.944$ & $0.944$ & $0.944$ \\
CIFAR-10  & proximity & $0.942$ & $0.942$ & $\mathbf{0.944}$ & $0.942$ \\
CIFAR-100 & alignment & $0.768$ & $0.766$ & $\mathbf{0.771}$ & $0.766$ \\
CIFAR-100 & proximity & $0.763$ & $0.760$ & $\mathbf{0.768}$ & $0.761$ \\
\bottomrule
\end{tabular}
\end{table}

\paragraph{Real coupled-target test: QM9 atomization energy.}
The coupled-target preference (Section~\ref{sec:exp_necessity}) is synthetic; QM9 \citep{ramakrishnan2014quantum} is the natural real-data test, because atomization energy is the home benchmark of kernel ridge regression in chemistry \citep{rupp2012fast} and is extensive: size and composition live in the \emph{norm} of a Coulomb-matrix descriptor, structure in its \emph{direction}. We run the tuned protocol of Table~\ref{tab:tuned} on all $133{,}885$ molecules (targets $U_0$ atomization in kcal/mol; train $6000$/val $1500$/test $2000$, $3$ resamples; $99$th-percentile norm clip then global scaling, so the data is an off-sphere ball) with two descriptors: the Coulomb-matrix eigenspectrum ($d=29$) and the sorted full Coulomb matrix ($d=435$), and with the two coupling controls of the necessity test: a \emph{direction-only} ablation (inputs projected to the sphere) and a \emph{norm-only} ablation ($1$-D input $\|x\|$), plus the composition-linear ``dressed-atom'' baseline (Table~\ref{tab:qm9}).

\begin{table}[h]
\centering
\caption{QM9 atomization energy, test MAE in kcal/mol (mean $\pm$ std over $3$ resamples; train $6000$ unless stated; per-kernel $\varepsilon$-multiplier/$b$/$\lambda$ tuned on validation). Joint = full descriptor; dir-only = sphere-projected (norms deleted). Norm-only ($1$-D $\|x\|$, Gaussian): $172.9{\pm}14.7$ (eig) / $161.2{\pm}1.6$ (full CM), against $\approx190$ for predicting the mean. Dressed atom (linear in element counts): $19.6$. Full-$N$ sketched-RAY primal trains on all $131{,}885$ rows ($m{=}128$; eig: $D{=}32$, $M{=}5056$, $12$\,s; full CM: $D{=}16$, $M{=}9024$, $35$\,s), where the exact modulation ($d_b\approx95{,}000$ at $d{=}435$) and the $N\times N$ Gram are both impossible.}
\label{tab:qm9}
\begin{tabular}{lcccc}
\toprule
& \multicolumn{2}{c}{eigenspectrum ($d{=}29$)} & \multicolumn{2}{c}{full CM ($d{=}435$)} \\
\cmidrule(lr){2-3}\cmidrule(lr){4-5}
kernel & joint & dir-only & joint & dir-only \\
\midrule
Gaussian   & $\mathbf{31.6}{\scriptstyle\pm.3}$ & $35.0{\scriptstyle\pm1.0}$ & $16.2{\scriptstyle\pm.4}$ & $16.6{\scriptstyle\pm.4}$ \\
IMQ        & $35.6{\scriptstyle\pm.8}$ & $37.3{\scriptstyle\pm.5}$ & $\mathbf{15.9}{\scriptstyle\pm.5}$ & $16.2{\scriptstyle\pm.4}$ \\
Mat\'ern-$\tfrac12$ & $45.7{\scriptstyle\pm1.2}$ & $45.3{\scriptstyle\pm.5}$ & $18.3{\scriptstyle\pm.3}$ & $18.5{\scriptstyle\pm.3}$ \\
polynomial & $32.8{\scriptstyle\pm.3}$ & $35.5{\scriptstyle\pm.6}$ & $17.0{\scriptstyle\pm.4}$ & $17.5{\scriptstyle\pm.5}$ \\
$\yat$ (exact) & $35.7{\scriptstyle\pm.7}$ & $37.3{\scriptstyle\pm.5}$ & $15.9{\scriptstyle\pm.5}$ & $16.3{\scriptstyle\pm.4}$ \\
RAY ($D{=}1000$) & $32.6{\scriptstyle\pm.4}$ & -- & $16.9{\scriptstyle\pm.3}$ & -- \\
\midrule
full-$N$ RAY primal ($131{,}885$ rows) & $26.8$ & -- & $\mathbf{14.5}$ & -- \\
\bottomrule
\end{tabular}
\end{table}

\noindent The verdict is honest in both directions. \emph{The coupling is real but modest:} deleting the norms costs every kernel a consistent $2$--$11\%$ (Gaussian $31.6\to35.0$ on the eigenspectrum, $16.2\to16.6$ on the full CM), while the norm alone is nearly uninformative ($161$--$173$ vs.\ $\approx190$ for the mean), so the target is direction-dominated with a small genuine norm contribution rather than strongly coupled. \emph{The kernel preference does not transfer:} the $\yat$-kernel is statistically tied with the best kernel on the full Coulomb matrix ($15.9{\pm}.5$ vs.\ IMQ $15.9{\pm}.5$, Gaussian $16.2{\pm}.4$) and behind the Gaussian on the eigenspectrum, so the synthetic coupled-target advantage of Table~\ref{tab:necessity} does not reappear here. Two findings are positive and specific to the construction. First, RAY \emph{beats its own exact kernel} on the eigenspectrum ($32.6$ vs.\ $35.7$): the random-feature low-rank truncation acts as a regularizer on a near-noiseless target, an effect the unbiasedness analysis permits but does not predict. Second, and the operative point, the full-dataset sketched primal dominates everything Gram-bound: at $d=435$ the exact modulation feature ($d_b\approx95{,}000$ per draw) cannot be built and the $1.7\times10^{10}$-entry Gram cannot be formed, yet the $M{=}9024$ sketched primal trains on all $131{,}885$ molecules in $35$\,s and reaches $14.5$ kcal/mol, beating the best exact kernel ($15.9$) that is confined to a $6000$-point Gram: more data through compressed features beats a better-resolved kernel on a subsample. Scripts \texttt{qm9\_build\_cache}, \texttt{qm9\_atomization}.

\paragraph{Importance sampling (Section~\ref{sec:variance_reduction}).} The proposal $t\sim\Exp(\varepsilon+\eta)$ is unbiased but has a finite-variance window.

\begin{proposition}[Exponential proposal for the radial factor]\label{prop:importance}
With $T\sim\Exp(\varepsilon+\eta)$ ($\eta\ge0$) and $\widehat h_\eta(r)=\frac{1}{\varepsilon+\eta}e^{-(r-\eta)T}$, we have $\mathbb{E}[\widehat h_\eta(r)]=(r+\varepsilon)^{-1}$, and $\mathbb{E}[\widehat h_\eta(r)^2]=\bigl((\varepsilon+\eta)(\varepsilon+2r-\eta)\bigr)^{-1}$ is finite iff $\eta<\varepsilon+2r$. A proposal safe for all pairs, including $r=0$, requires $\eta<\varepsilon$.
\end{proposition}
\begin{proof}
$\mathbb{E}[\widehat h_\eta(r)]=\int_0^\infty\frac{1}{\varepsilon+\eta}e^{-(r-\eta)t}(\varepsilon+\eta)e^{-(\varepsilon+\eta)t}dt=\int_0^\infty e^{-(r+\varepsilon)t}dt=(r+\varepsilon)^{-1}$. Similarly $\mathbb{E}[\widehat h_\eta(r)^2]=\frac{1}{\varepsilon+\eta}\int_0^\infty e^{-(\varepsilon+2r-\eta)t}dt$, finite iff $\varepsilon+2r-\eta>0$, equal to $((\varepsilon+\eta)(\varepsilon+2r-\eta))^{-1}$.
\end{proof}

\noindent Empirically, on pairs at several squared distances ($x^\top w=0.5$, $b=1$, $\varepsilon=1$, $\eta<\varepsilon$, $D=200$, $3000$ repetitions) the estimator stays unbiased (bias $\le10^{-3}$) and the reduction is real but tuning-dependent: up to a $0.42\times$ variance ratio at $\|x-w\|^2=1$ ($\eta=0.5$) and $0.55\times$ at $\|x-w\|^2=0.25$ ($\eta=0.1$), but \emph{no} improvement for the nearest pair $\|x-w\|^2=0.05$ (where $\eta=0$ is best). Since the optimal $\eta$ is data-dependent, a single fixed proposal is a compromise.

\paragraph{Orthogonal features in $d>1$ (Section~\ref{sec:variance_reduction}).} At a fixed scale $t$ we estimate the Gaussian factor $g_t(x,w)$ with $D'=d$ inner frequencies, drawn either i.i.d.\ $\mathcal{N}(0,2tI_d)$ or as a scaled Haar-orthogonal block \citep{yu2016orthogonal} that preserves the marginal. Both are unbiased (empirical bias $\le6\times10^{-3}$), and the orthogonal block lowers the inner variance by a ratio of $0.72$ at $d=5$ and $0.76$ at $d=20$ ($4000$ repetitions, $x^\top w=0.3$). Unlike the $d=1$ case of Appendix~\ref{sec:exp_var}, where there is no direction to decorrelate, orthogonality helps once $d>1$.

\paragraph{Polynomial sketching (Section~\ref{sec:discussion}).} The TensorSketch that RAY uses (Proposition~\ref{prop:ts_variance}) sketches only the quadratic term $(x^\top w)^2$ and keeps $2b\,x^\top w+b^2$ exact. Here we probe the cruder variant that sketches the whole augmented feature $(x,\sqrt b)$, replacing the exact $d_b$-dimensional polynomial by a $D_{\mathrm{poly}}$-dimensional sketch (this also sketches the bias and linear terms, so we do not use it in the main experiments). Its error decays slowly, as $\sim D_{\mathrm{poly}}^{-1/2}$: at $d=10$ ($d_b=121$) a sketch of $D_{\mathrm{poly}}=128$ gives full-Gram error $0.175$ against the exact feature's $0.046$; at $d=100$ ($d_b=10{,}201$) a $20\times$ smaller sketch ($D_{\mathrm{poly}}=512$) gives $0.172$ against $0.060$. Sketching is thus a controllable memory--accuracy trade, useful only when $d^2$ is prohibitive and some accuracy can be sacrificed; at moderate $d$ the exact polynomial feature is both cheaper and more accurate. The $d^2$ cost is better reduced losslessly by the symmetric $d(d{+}1)/2$ reduction.

\paragraph{A non-$\yat$ Bernstein--Schur instance (Theorem~\ref{thm:bernstein_schur}).} To check that the construction is genuinely about the class and not the $\yat$-kernel alone, we run it unchanged on a different member: $k(x,w)=(x^\top w+b)^3(\|x-w\|^2+\varepsilon)^{-\alpha}$, a degree-3 polynomial modulation times the generalized IMQ ($\alpha=2$), whose Bernstein measure is $\Gamma(\alpha,\varepsilon)$ with shape $\alpha$ and rate $\varepsilon$. Keeping the cubic modulation exact and sampling $T\sim\Gamma(\alpha,\varepsilon)$, $\omega\sim\mathcal{N}(0,2TI_d)$, on off-sphere data ($\|x\|\in[0.3,1.2]$, $d=8$, $N=400$) the estimator is unbiased (mean of $200$ estimates at $D=50$ is within $0.023$ relative Frobenius error of the exact Gram) and converges at the Monte-Carlo rate (relative error $0.71\to0.078$ as $D$ runs $10\to1000$, fitted log-log slope $-0.48$). Only the modulation feature and the mixing law changed from the $\yat$-kernel; the recipe did not.

\paragraph{A kernel-grammar instance on real data (Theorem~\ref{thm:class_bernstein}).} The compositional-search product $\mathrm{LIN}^2\times\mathrm{RQ}$ \citep{duvenaud2013structure} is exercised end to end on California housing ($N=20{,}640$, $d=8$; standardized, $99$th-percentile norm clip so the ball is off-sphere, targets standardized). With $k(x,w)=(x^\top w+b)^2(1+\|x-w\|^2/\varepsilon')^{-\alpha}$, $b=1$, $\alpha=2$, $\varepsilon'=2\alpha\sigma^2$ set to the median squared distance ($0.225$), the class estimator (modulation exact, $T\sim\Gamma(\alpha,\varepsilon')$, trig RFF) gives: \emph{(a)} Gram fidelity at $N=2000$ following the Monte-Carlo rate (relative Frobenius error $0.31\to0.15\to0.089\to0.041$ over $D=50/200/800/3200$, fitted slope $-0.47$; $3$ seeds); \emph{(b)} downstream KRR at $N=3000$ matching the exact composite kernel already at $D=200$ (test RMSE $0.540\pm.007$ vs.\ exact $0.543$, $\lambda=10^{-2}$); \emph{(c)} full-scale primal ridge on all $N=19{,}640$ training rows with explicit features ($D=100$, $d_b=45$, $M=4500$) fitting in $1.2$\,s at test RMSE $0.480$, \emph{below} the exact-kernel reference because the reference is Gram-bound to a $3000$-point subsample while the primal uses all the data; that gap is the scaling argument in one number. Script \texttt{grammar\_kernel}.

\paragraph{Sensitivity to $b$ and $\varepsilon$.} We sweep the bias $b\in\{0,0.1,1,10\}$ and the radial scale $\varepsilon\in\{0.25,1,4\}\times$(median squared distance) and report the relative Frobenius Gram error of the flat estimator ($D=200$, $N=300$, $d=16$, $5$ seeds) for exact, normalized, and sketched modulation (the last is RAY) (Table~\ref{tab:sensitivity}). The exact and normalized estimators stay in a tight $0.06$--$0.25$ band across two orders of magnitude in $b$ and a $16\times$ range in $\varepsilon$: the method is insensitive to either, and $\varepsilon=$ median squared distance is a sound default. RAY adds the expected $O(m^{-1/2})$ sketch error ($m=128$), largest at small $b$ and small $\varepsilon$ where the quadratic term dominates the modulation.

\begin{table}[h]
\centering
\caption{Relative Frobenius Gram error vs.\ the bias $b$ and radial scale $\varepsilon$ (as a multiple of the median squared distance), flat estimator at $D=200$ ($N=300$, $d=16$, mean over $5$ seeds; per-cell standard deviation $\le0.023$ for the exact and normalized estimators, $\le0.046$ for RAY (sketch)). The exact and normalized estimators are stable across the grid; RAY adds an $O(m^{-1/2})$ sketch error, largest where the quadratic modulation dominates (small $b$, small $\varepsilon$).}
\label{tab:sensitivity}
\begin{tabular}{cccccccccc}
\toprule
& \multicolumn{3}{c}{$\varepsilon=0.25\times$} & \multicolumn{3}{c}{$\varepsilon=1\times$} & \multicolumn{3}{c}{$\varepsilon=4\times$} \\
\cmidrule(lr){2-4}\cmidrule(lr){5-7}\cmidrule(lr){8-10}
$b$ & exact & norm & RAY & exact & norm & RAY & exact & norm & RAY \\
\midrule
$0$    & $0.095$ & $0.114$ & $0.256$ & $0.081$ & $0.083$ & $0.384$ & $0.063$ & $0.062$ & $0.547$ \\
$0.1$  & $0.092$ & $0.107$ & $0.231$ & $0.077$ & $0.073$ & $0.337$ & $0.058$ & $0.057$ & $0.481$ \\
$1$    & $0.190$ & $0.165$ & $0.201$ & $0.102$ & $0.086$ & $0.123$ & $0.065$ & $0.061$ & $0.106$ \\
$10$   & $0.251$ & $0.235$ & $0.251$ & $0.111$ & $0.107$ & $0.111$ & $0.069$ & $0.067$ & $0.069$ \\
\bottomrule
\end{tabular}
\end{table}

\subsection{Kernel-value variance and variance reduction}
\label{sec:exp_var}

The estimator's variance falls as $O(1/D)$, and quasi-Monte Carlo over the scale lowers it by a constant factor. We estimate $k_{\yat,b}(x,w)$ for a fixed $d=1$ pair over $1000$ repetitions ($x=0.5$, $w=1$, $b=\varepsilon=1$, $D'=50$; the scalar exact polynomial factor isolates the radial Gaussian approximation), comparing the basic, QMC, orthogonal, and combined estimators in Table~\ref{tab:variance_reduction}.

\begin{table}[h]
\centering
\caption{Empirical variance of the biased-kernel estimator ($\times 10^{-3}$, $b=1$, $\varepsilon=1$, $1000$ repetitions).}
\label{tab:variance_reduction}
\begin{tabular}{lccc}
\toprule
Method & $D=10$ & $D=50$ & $D=100$ \\
\midrule
Basic RAY        & $17.54$ & $3.92$ & $2.00$ \\
QMC RAY          & $\mathbf{7.32}$ & $\mathbf{1.23}$ & $\mathbf{0.62}$ \\
Orthogonal RAY   & $19.69$ & $4.28$ & $2.21$ \\
QMC + orthogonal & $8.31$ & $1.40$ & $0.69$ \\
\bottomrule
\end{tabular}
\end{table}

Basic RAY halves its variance with each doubling of $D$ (a $4.5\times$ drop from $D{=}10$ to $50$), and QMC sampling of the radial scale cuts it a further $2.4$--$3.2\times$ at every $D$. Orthogonalizing the inner frequencies does nothing here (in one dimension there is no direction to decorrelate) but reduces the inner variance by $24$--$28\%$ once $d>1$ (Appendix~\ref{app:further}) and composes with QMC. The Theorem~\ref{thm:variance} bound holds but is loose, its worst-case Cauchy--Schwarz making the empirical constants far smaller (at $D{=}100$, bound $0.16$ vs.\ variance $0.002$).

One caveat on QMC: it buys a constant, not a faster rate. Fitting log-log RMSE against $D\in[10,2000]$ gives a Monte-Carlo slope of $-0.502$ and a QMC slope of $-0.514$, the same $O(1/\sqrt D)$ rate, with QMC a constant $\sim\!1.7\times$ lower. The inner Gaussian RFF noise, which is not quasi-randomized, sets the convergence, so quasi-randomizing only the one-dimensional scale cannot change it; the $O((\log D)^s/D)$ rate of Section~\ref{sec:variance_reduction} is the outer integral's alone.

\subsection{Outer/inner allocation: flat sampling is optimal}
\label{sec:exp_budget}

Flat sampling (one frequency per scale) is the best use of a fixed budget, confirming Proposition~\ref{prop:budget}. The estimator splits its randomness between $D$ outer scales and $D'$ inner frequencies at identical cost $O(N^2 DD')$; holding the budget $B=DD'=1000$ fixed and varying the split ($d=10$, $N=400$, sphere, $\varepsilon=b=1$, $4$ seeds) gives Table~\ref{tab:budget}.

\begin{table}[h]
\centering
\caption{Relative Frobenius error at a fixed budget $B=DD'=1000$ as the inner count $D'$ varies (so $D=1000/D'$). Smaller $D'$ is better; $D'=1$ (flat sampling) wins.}
\label{tab:budget}
\begin{tabular}{lccccc}
\toprule
$D'$ (with $D=1000/D'$) & $1$ & $2$ & $5$ & $10$ & $50$ \\
\midrule
rel.\ Frobenius error & $\mathbf{0.059}$ & $0.063$ & $0.066$ & $0.089$ & $0.175$ \\
\bottomrule
\end{tabular}
\end{table}

Error rises monotonically with $D'$: $D'=1$ is optimal here, $3\times$ better than $D'=50$. This is Proposition~\ref{prop:budget} in numbers: with $DD'$ fixed the inner term $V_{\mathrm{in}}/(DD')$ is constant and the outer term $V_{\mathrm{out}}/D$ falls as $1/D$ when $V_{\mathrm{out}}>0$, so the largest $D$ wins (if $V_{\mathrm{out}}=0$, all allocations tie). The two-level scheme is thus an artifact away from that degenerate case: one should draw $D$ independent $(t_j,\omega_j)$ pairs, the flat estimator~\eqref{eq:flat} that Section~\ref{sec:step-rff} identifies with IMQ random features. (Accordingly, the main Gram, dimension, and ridge-regression experiments all use the recommended flat $D'=1$; only the variance-reduction probe of Appendix~\ref{sec:exp_var}, which isolates the radial Gaussian factor, fixes $D'=50$.)

\subsection{Scalability: clearing the \texorpdfstring{$O(N^2)$}{O(N\^{}2)} wall}
\label{sec:exp_scaling}

Random $\yat$-features clear the $O(N^2)$ bottleneck the construction was built to clear. We time fitting ridge regression on random targets (cost, not accuracy) three ways (exact kernel ridge (form $K$, solve the $N\times N$ system); RAY primal ridge (features $Z\in\mathbb{R}^{N\times M}$ with $M=D\,d_b'$, the symmetric polynomial $d_b'=d(d{+}1)/2+d+1$ and flat $D'=1$, solve the $M\times M$ system); and Nystr\"om ($m$ landmarks)) reporting wall-clock and the dominant array's memory ($d=8$, $D=m=64$, $M=2880$; Figure~\ref{fig:scaling}).

\begin{figure}[h]
\centering
\includegraphics[width=\textwidth]{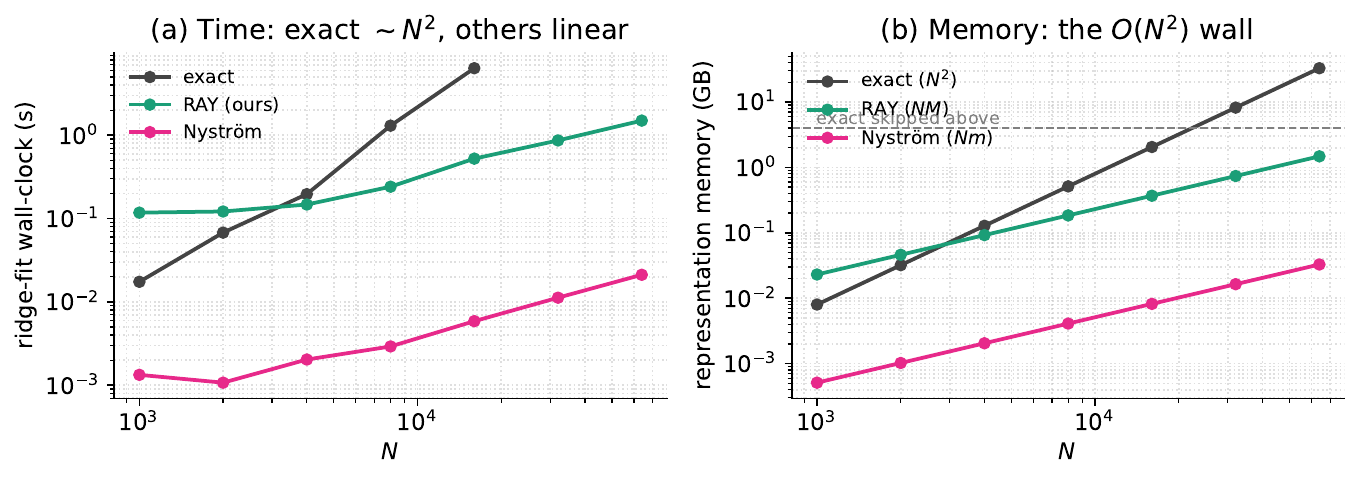}
\caption{Cost of fitting ridge regression vs.\ $N$ ($d=8$, $D=m=64$, log-log). \textbf{(a)} Wall-clock: exact ridge steepens at the predicted $\sim N^2$ rate (fitted exponent $2.1$) and is run only while feasible; RAY and Nystr\"om grow gently. \textbf{(b)} Representation memory: the exact $N\times N$ Gram reaches $33$\,GB by $N=64{,}000$ (above the dashed cap, where it no longer fits), while RAY ($NM$) and Nystr\"om ($Nm$) stay linear in $N$.}
\label{fig:scaling}
\end{figure}

Exact ridge scales as the predicted $O(N^2)$ (fitted time exponent $2.1$) and its Gram matrix reaches $33$\,GB at $N=64{,}000$, where it no longer fits in memory; RAY and Nystr\"om, whose representations are linear in $N$ ($NM$ and $Nm$), run there in $1.5$ and $0.02$ seconds. The honest caveat is the \emph{constant}, not the scaling: RAY's representation $M=D\,d_b'$ carries the degree-2 polynomial dimension, so it is heavier than Nystr\"om's $m$ ($2880$ vs $64$ here); the symmetric reduction (used) and sketching shrink $d_b'$, but a downstream method built on RAY should keep $D$ and the polynomial dimension modest.

\paragraph{Large-scale primal training on HIGGS (full sweep).} The streaming HIGGS demonstration of Section~\ref{sec:exp_higgs} compares, at matched representation dimension $M$, Gaussian RFF ($D{=}M$), exact modulation ($D{=}\lfloor M/d_b\rfloor$), and RAY ($m{=}128$, $D{=}\lfloor M/(m{+}d{+}1)\rfloor$); the off-sphere preprocessing standardizes then rescales by the $99.9$th norm percentile (varying norms in the bounded ball), $\varepsilon$ is the median squared distance, $b=1$, and the polynomial floor is $d_b=435$. Two facts come through (Table~\ref{tab:higgs}). First, the estimator trains at $11$M scale as an ordinary streaming primal model: peak memory is flat at $8.5$\,GB, identical to plain random features, because the Gram is never formed. Second, at matched representation budget compression matches or beats exact modulation everywhere and wins decisively at the top budget ($0.751{\scriptstyle\pm.002}$ vs.\ $0.740{\scriptstyle\pm.003}$ at $M{=}8192$): the exact polynomial starves the radial factor to $D{=}1$ at $M{=}512$, where its AUC swings $0.659$--$0.724$ across seeds ($\pm0.027$), an instability that is itself the argument for compression, while the sketch buys $3$--$52\times$ more draws and stable estimates ($\pm0.001$--$.009$). Gaussian RFF leads on AUC, which is expected, HIGGS is a smooth detector-feature classification with no strong alignment$\times$proximity coupling for the $\yat$-modulation to exploit. The claim is not ``RAY wins HIGGS'' but that compressed Bernstein--Schur features are a practical, memory-flat, million-scale estimator, and that compression, not exactness, is the right setting once $d_b$ is large.

\begin{table}[h]
\centering
\caption{Streaming primal training on HIGGS ($N_{\text{train}}=10.5$M, $N_{\text{test}}=500$k, $d=28$, $d_b=435$) at matched representation dimension $M$. Test AUC ($\uparrow$), mean $\pm$ std over $3$ feature seeds; $D$ = radial draws afforded at that budget; best feature method per $M$ in bold. Peak memory is flat at $8.5$\,GB across all methods (no Gram, no full feature matrix); one sweep runs in $\sim5$ minutes on one M5~Pro GPU. Linear baseline AUC $0.682$. The exact-modulation column at $M{=}512$ is a \emph{single} radial draw, whose $\pm0.027$ seed spread is itself the instability that compression removes.}
\label{tab:higgs}
\begin{tabular}{lcccccc}
\toprule
& \multicolumn{2}{c}{Gaussian RFF} & \multicolumn{2}{c}{exact modulation} & \multicolumn{2}{c}{RAY (sketch)} \\
\cmidrule(lr){2-3}\cmidrule(lr){4-5}\cmidrule(lr){6-7}
$M$ & $D$ & AUC & $D$ & AUC & $D$ & AUC \\
\midrule
$512$  & $512$  & $\mathbf{0.721}{\scriptstyle\pm.002}$ & $1$  & $0.693{\scriptstyle\pm.027}$ & $3$  & $0.697{\scriptstyle\pm.009}$ \\
$1024$ & $1024$ & $\mathbf{0.738}{\scriptstyle\pm.001}$ & $2$  & $0.714{\scriptstyle\pm.012}$ & $6$  & $0.718{\scriptstyle\pm.001}$ \\
$2048$ & $2048$ & $\mathbf{0.749}{\scriptstyle\pm.001}$ & $4$  & $0.730{\scriptstyle\pm.006}$ & $13$ & $0.731{\scriptstyle\pm.002}$ \\
$4096$ & $4096$ & $\mathbf{0.755}{\scriptstyle\pm.001}$ & $9$  & $0.737{\scriptstyle\pm.004}$ & $26$ & $0.742{\scriptstyle\pm.001}$ \\
$8192$ & $8192$ & $\mathbf{0.759}{\scriptstyle\pm.000}$ & $18$ & $0.740{\scriptstyle\pm.003}$ & $52$ & $0.751{\scriptstyle\pm.002}$ \\
\bottomrule
\end{tabular}
\end{table}

\noindent \emph{Protocol.} Streaming Adam (learning rate $0.02$, batch size $8192$, two passes over the $10.5$M training rows), single precision (fp32), three feature seeds (mean $\pm$ std in the table); features are rebuilt per mini-batch on the GPU and discarded. Reported wall-clock includes per-batch feature construction but excludes the one-time gzip load of the data; the $8.6$\,GB peak (measured by \texttt{getrusage}) is the resident process maximum, the same for all methods since none forms a Gram. The whole $\{512,\dots,8192\}\times\{$linear, Gaussian, exact modulation, RAY$\}$ sweep takes $\sim\!260$\,s on one Apple~M5~Pro.

\paragraph{Relative-spectral KRR stability for the sketched estimator.} The relative-spectral guarantee (Theorem~\ref{thm:krr_spectral}) also governs the deployed sketched RAY (Table~\ref{tab:krr_sketched}). The sketch term keeps $\rho$ above the exact-modulation value, so more regularization or budget is needed for the bound to activate, but once it does the coefficient error obeys it.

\begin{table}[h]
\centering
\caption{Relative-spectral KRR stability for the deployed (sketched) RAY (digits subset, $N{=}600$). $\rho=\|A^{-1/2}(K_{D,m}-K)A^{-1/2}\|_{\mathrm{op}}$ with $A=K+\lambda I$; the deterministic bound $\|\tilde\alpha-\hat\alpha\|_A\le\frac{\rho}{1-\rho}\|\hat\alpha\|_A$ activates once $\rho<1$. $\rho$ falls with the radial draws $D$, the sketch size $m$, and the ridge $\lambda$; the measured coefficient error stays within the bound throughout.}
\label{tab:krr_sketched}
\begin{tabular}{lccc}
\toprule
sweep & $\rho$ & rel.\ coef.\ err & bound $\rho/(1{-}\rho)$ \\
\midrule
\multicolumn{4}{l}{$\lambda$ sweep ($D{=}512,m{=}256$)} \\
\midrule
$\lambda{=}0.3$ & $0.757$ & $0.255$ & $3.12$ \\
$\lambda{=}1.0$ & $0.517$ & $0.192$ & $1.07$ \\
$\lambda{=}3.0$ & $0.330$ & $0.138$ & $0.49$ \\
$\lambda{=}10$  & $0.177$ & $0.088$ & $0.22$ \\
\midrule
\multicolumn{4}{l}{$D$ sweep ($m{=}256,\lambda{=}3$)} \\
\midrule
$D{=}64$   & $0.990$ & $0.259$ & $103$  \\
$D{=}256$  & $0.375$ & $0.158$ & $0.60$ \\
$D{=}1024$ & $0.292$ & $0.124$ & $0.41$ \\
\midrule
\multicolumn{4}{l}{$m$ sweep ($D{=}512,\lambda{=}3$)} \\
\midrule
$m{=}64$  & $0.821$ & $0.219$ & $4.58$ \\
$m{=}128$ & $0.518$ & $0.173$ & $1.07$ \\
$m{=}512$ & $0.248$ & $0.108$ & $0.33$ \\
\bottomrule
\end{tabular}
\end{table}

\paragraph{Bounded-input preprocessing.} The unbounded numerator of $k_{\yat,b}$ makes a norm bound mandatory (Limitation~(iii), Table~\ref{tab:prep}). Raw standardization (max norm $9.3$) inflates RMSE to $0.37$ (exact) / $0.96$ (RAY); any bounded scheme recovers $\le0.05$, confirming Limitation~(iii).

\begin{table}[h]
\centering
\caption{Bounded-input preprocessing (off-sphere coupled target, $d{=}16$, KRR test RMSE, mean over $3$ seeds). The unbounded numerator destabilizes both the exact kernel and RAY on raw data; any norm bound restores accuracy; clipping and max-norm scaling are best.}
\label{tab:prep}
\begin{tabular}{lccc}
\toprule
scheme & max $\|x\|$ & exact $\yat$ (RMSE) & RAY (RMSE) \\
\midrule
raw (standardize only) & $9.25$ & $0.373$ & $0.959$ \\
max-norm scaling       & $1.00$ & $0.029$ & $0.033$ \\
percentile clip (99)   & $1.00$ & $\mathbf{0.027}$ & $0.041$ \\
sphere normalization   & $1.00$ & $0.035$ & $0.047$ \\
normalized kernel      & $1.00$ & $0.041$ & -- \\
\bottomrule
\end{tabular}
\end{table}

\paragraph{Sphere-normalized matched-cost checks.} On the unit sphere $k_{\yat,b}$ coincides with a dot-product kernel, where direct zonal/dot-product routes are natural baselines; the sphere-normalized matched-dimension and fair-cost results below are therefore sanity checks, with the off-sphere versions (Tables~\ref{tab:dm}, \ref{tab:offsphere_faircost}) the primary evidence. The ordering is the same: at matched dimension RAY's TensorSketch recovers most of the efficiency the exact feature loses to its $O(d^2)$ size (Table~\ref{tab:ts}), and at matched representation the adaptive Nystr\"om variants are most accurate while RAY is the data-independent streaming option (Table~\ref{tab:faircost}).

\begin{table}[h]
\centering
\caption{Matched-dimension KRR on sphere-normalized \texttt{digits} ($d=64$, $b=1$, accuracy, mean $\pm$ std over $3$ splits). At a fixed explicit feature dimension $M$, exact modulation affords only $M/d_b$ radial draws; RAY ($m=128$) compresses the polynomial factor and affords many more, approaching the optimal rank-$M$ ceiling.}
\label{tab:ts}
\begin{tabular}{lccc}
\toprule
$M$ & exact modulation & RAY ($m{=}128$) & optimal rank-$M$ \\
\midrule
$2145$ & $0.928{\scriptstyle\pm.040}$ & $\mathbf{0.977}{\scriptstyle\pm.006}$ & $0.986{\scriptstyle\pm.007}$ \\
$4290$ & $0.924{\scriptstyle\pm.054}$ & $0.977{\scriptstyle\pm.006}$ & $0.986{\scriptstyle\pm.007}$ \\
$8580$ & $0.973{\scriptstyle\pm.004}$ & $0.979{\scriptstyle\pm.007}$ & $0.986{\scriptstyle\pm.007}$ \\
\bottomrule
\end{tabular}
\end{table}

\begin{table}[h]
\centering
\caption{Fair-cost comparison on sphere-normalized \texttt{digits} ($d=64$, $b=1$, accuracy mean $\pm$ std over $3$ splits) at matched representation dimension. Memory is $N\times\text{dim}$ floats; build is feature/landmark construction wall-clock.}
\label{tab:faircost}
\begin{tabular}{lcccc}
\toprule
Method & dim & accuracy & memory (MB) & build (s) \\
\midrule
exact modulation & $2145$ & $0.928{\scriptstyle\pm.040}$ & $20.6$ & $0.020$ \\
RAY ($m{=}128$) & $2123$ & $0.977{\scriptstyle\pm.006}$ & $20.3$ & $\mathbf{0.009}$ \\
k-means Nystr\"om & $1198$ & $\mathbf{0.986}{\scriptstyle\pm.007}$ & $\mathbf{11.5}$ & $1.18$ \\
rls-leverage Nystr\"om & $1198$ & $\mathbf{0.986}{\scriptstyle\pm.004}$ & $\mathbf{11.5}$ & $0.434$ \\
\midrule
Gaussian RFF (different kernel) & $2145$ & $0.984{\scriptstyle\pm.004}$ & $20.6$ & $0.024$ \\
\bottomrule
\end{tabular}
\end{table}

\section{Experimental Details}
\label{app:exp_details}

This appendix collects the protocol shared across experiments and the per-table configuration, so that every number in the paper can be reproduced. The accompanying code, figures, and a project page are available online.\footnote{\url{https://www.tahabouhsine.com/ray}}

\paragraph{Shared protocol.} Unless a table states otherwise, the radial scale $\varepsilon$ is set to the median squared pairwise distance of the inputs (computed on a subsample), the bias is $b=1$, and the deployed RAY map is the lower-variance quadratic-only TensorSketch (sketch size $m$, $D$ radial draws, flat $D'=1$): the degree-2 term is sketched to dimension $m$ and the linear/constant terms are kept exact, so the feature dimension is $D(m{+}d{+}1)$. Kernel ridge regression is solved exactly (in the primal (features) for RAY and RFF, and in the dual ($N\times N$) for the exact $\yat$-kernel and Nystr\"om) with the ridge $\lambda$ listed per table; ``exact modulation'' is Definition~\ref{def:ryf} (the $m\to\infty$ limit). Off-sphere data is a bounded ball with norms drawn over a stated interval (e.g.\ $\|x\|\in[0.3,1.5]$), the regime where $k_{\yat,b}$ is genuinely non-dot-product; sphere-normalized runs (Appendix tables) fix $R=1$ and are sanity checks only. Reported error bars are mean$\pm$std over the seed count in each caption.

\paragraph{Hardware and precision.} All Gram-error, KRR, operator-norm, and feature-build timing experiments run on CPU in double precision (fp64) and use no GPU. The sole GPU experiment is the million-scale HIGGS streaming run (Table~\ref{tab:higgs}), which runs in single precision (fp32) on one Apple~M5~Pro GPU; its full protocol (streaming Adam, batch $8192$, two passes over the training rows, peak resident memory) is stated inline with the table. Timing tables therefore compare like for like within a row; absolute wall-clock is hardware-specific and reported only to show scaling, not peak throughput.

\begin{table}[h]
\centering
\caption{Per-table configuration. ``Data'' gives the dataset and input dimension $d$; ``train/test'' the KRR split sizes (``Gram'' or ``op-norm'' marks experiments that approximate a matrix rather than fit a predictor, with no split; ``build'' marks feature-construction timing only). Grids list the swept hyperparameters; $\varepsilon$ is the median squared distance throughout. Each table's own caption restates its seed count and error bars.}
\label{tab:exp_details}
\small
\resizebox{\textwidth}{!}{%
\begin{tabular}{lllcl}
\toprule
Table / Fig.\ & Data ($d$) & train/test & seeds & Swept grid \\
\midrule
Fig.~\ref{fig:offsphere}, Tab.~\ref{tab:offsphere} & off-sphere ball ($d{\in}\{2,8,16,32\}$) & Gram, $N{=}1000$ & 5 & $D{\in}\{10,50,100,500,1000\}$, Nystr\"om $m{=}100$ \\
Tab.~\ref{tab:offsphere_faircost} & off-sphere ball ($d{=}64$) & $1500/1000$ & 3 & $M{=}2145$, $m_{\mathrm{TS}}{=}128$, $\lambda{=}10^{-2}$ \\
Tab.~\ref{tab:necessity} & off-sphere ($d{=}16$) & $1333/667$ & 5 & $D{=}4000$, $\lambda{=}10^{-2}$ \\
Tab.~\ref{tab:tuned} & coupled ($d{=}16$), digits ($d{=}64$) & $800/400$ (val $300$) & 3 & $b{\in}\{0,.5,1,2\}$, $\varepsilon_{\mathrm{mult}}{\in}\{.25,1,4\}$, $\lambda{\in}\{10^{-3},10^{-2},10^{-1}\}$ \\
Tab.~\ref{tab:coupled_matched} & off-sphere ($d{=}32$) & $1300/600$ & 3 & $M{=}4096$, $m{=}128$, $\lambda{=}10^{-2}$ \\
Tab.~\ref{tab:dm} & off-sphere ($d{\in}\{16,64,256\}$) & $400/200$ & 3 & $m{\in}\{16,\dots,512\}$, $M{\in}\{4096,8192,16384\}$ \\
Tab.~\ref{tab:runtime_d} & synthetic ($d{\in}\{8,\dots,1024\}$) & build, $N{=}1000$ & 1 & $D{=}8$, $m{=}128$ \\
Tab.~\ref{tab:krr_sketched} & digits ($d{=}64$) & op-norm, $N{=}600$ & 3 & $\lambda{\in}\{.1,.3,1,3,10\}$, $D{\in}\{64,256,512,1024\}$, $m{\in}\{64,128,256,512\}$ \\
Tab.~\ref{tab:prep} & off-sphere ($d{=}16$) & $800/400$ & 3 & 5 norm schemes; $D{=}24,m{=}128$, $\lambda{=}10^{-2}$ \\
Tab.~\ref{tab:higgs} & HIGGS ($d{=}28$) & $10.5\mathrm{M}/500\mathrm{k}$ & 3 & $M{\in}\{512,\dots,8192\}$, $m{=}128$ (MLX, fp32) \\
Tab.~\ref{tab:ts}, Tab.~\ref{tab:faircost} & sphere digits ($d{=}64$) & 3 splits & 3 & draws${\in}\{1,2,4\}$, sketch${\in}\{128,256\}$, $\lambda{=}10^{-2}$ \\
Fig.~\ref{fig:ts_opnorm} & synthetic ($d{=}16$) & op-norm, $N{=}300$ & 5 & $D{\in}\{10,30,100,300,1000\}$, $m{\in}\{64,128,256,512\}$ \\
Thm.~\ref{thm:krr_whitened}, Thm.~\ref{thm:class_bernstein} & synthetic ($d{=}20$) & whitened op-norm, $N{=}300$ & 6 & $\lambda{\in}\{10,1,.1,.01\}$, $D{\in}\{25,\dots,800\}$ \\
Fig.~\ref{fig:attention} & random Q/K/V ($d{=}32$) & op-fidelity, $N{\le}131072$ & 3 & $M$ up to $12416$, $m{=}64$ \\
Thm.~\ref{thm:krr_leverage} & synthetic ($d{=}20$) & whitened op-norm, $N{=}300$ & 4 & $\lambda{\in}\{10,1,.1,.01\}$, $D{\in}\{25,\dots,3200\}$, pool $5{\times}10^4$ \\
Prop.~\ref{prop:pos_dichotomy} & off-sphere ($d{=}16$) & Gram + pointwise, $N{=}400$ & 10 & $\varepsilon_{\mathrm{mult}}{\in}\{.25,1,4\}$, $D{\in}\{128,512,2048\}$, reps $\le10^6$ \\
Rmk.~\ref{rmk:complex} & off-sphere ($d{=}16$) & op-norm, $N{=}300$ & 10 & $m{\in}\{64,128,256,512\}$, real vs.\ complex signs \\
Thm.~\ref{thm:class_bernstein} (grammar) & california ($d{=}8$) & $3000/1000$ + full primal & 3 & $D{\in}\{50,\dots,4000\}$, $\alpha{=}2$, $\lambda{=}10^{-2}$ \\
Tab.~\ref{tab:qm9} & QM9 CM eig ($d{=}29$) / full ($d{=}435$) & $6000/1500/2000$ + full primal & 3 & $\varepsilon_{\mathrm{mult}}{\in}\{.25,1,4\}$, $b{\in}\{0,.25,1,4\}$, $\lambda_{\mathrm{rel}}{\in}\{10^{-9},\dots,10^{-3}\}$ \\
\bottomrule
\end{tabular}}
\end{table}

\end{document}